\documentclass[preprint,10pt,3p]{elsarticle}
\usepackage{geometry}
\usepackage{graphics} 
\usepackage{epsfig} 
\usepackage{times} 
\usepackage{amsmath} 
\usepackage{amssymb}  
\usepackage{epstopdf}
\usepackage[table]{xcolor} 
\usepackage{caption}
\usepackage{verbatim}
\usepackage{cases}
\allowdisplaybreaks[4]
\usepackage{enumerate}
\usepackage{booktabs}
\usepackage{balance}
\usepackage{mathbbol}
\usepackage{multirow}
\usepackage{adjustbox}

\usepackage{algorithm}
\usepackage{algorithmicx}
\usepackage{algpseudocode}
\usepackage{mathrsfs}
\usepackage{threeparttable}
\usepackage[draft]{todonotes}
\usepackage{makecell}

\usepackage{enumitem}

\usepackage{booktabs} 

\usepackage{multirow}

\usepackage[colorlinks=false,      
linkcolor=black,      
citecolor=black,      
filecolor=black,      
urlcolor=blue]{hyperref}

\newtheorem{definition}{Definition}
\newtheorem{proof}{Proof}
\newtheorem{theorem}{Theorem}

\newtheorem{remark}{Remark}
\newtheorem{lemma}{Lemma}
\newtheorem{corollary}{Corollary}

\def\x{{\bf x}}

\def\0{{\bf 0}}
\def\1{{\bf 1}}

\usepackage{wrapfig}

 \biboptions{authoryear}

\begin{document}
\begin{frontmatter}

\title{DriveVLM-RL: Neuroscience-Inspired Reinforcement Learning with Vision-Language Models for Safe and Deployable Autonomous Driving}

\author[a]{Zilin Huang}
\author[a]{Zihao Sheng}
\author[a]{Zhengyang Wan}
\author[b]{Yansong Qu}
\author[a]{Junwei You}
\author[c]{Sicong Jiang}
\author[a]{Sikai Chen\corref{cor1}}
\ead{sikai.chen@wisc.edu}

\cortext[cor1]{Corresponding author: Sikai Chen. }

\address[a]{Department of Civil and Environmental Engineering, University of Wisconsin-Madison, Madison, WI, 53706, USA}
\address[b]{Lyles School of Civil and Construction Engineering, Purdue University, West Lafayette, IN 47907, USA} 
\address[c]{Department of Civil Engineering, McGill University, Montreal, QC, H3A 0C3, Canada}

\begin{abstract}
Ensuring safe decision-making in autonomous 
vehicles remains a fundamental challenge despite 
rapid advances in end-to-end learning approaches. 
Traditional reinforcement learning~(RL) methods 
rely on manually engineered rewards or sparse 
collision signals, which fail to capture the 
rich contextual understanding required for safe 
driving and make unsafe exploration unavoidable 
in real-world settings. Recent vision-language 
models~(VLMs) offer promising semantic 
understanding capabilities; however, their high 
inference latency and susceptibility to 
hallucination hinder direct application to 
real-time vehicle control. To address these 
limitations, this paper proposes 
\textbf{DriveVLM-RL}, a neuroscience-inspired 
framework that integrates VLMs into RL through 
a dual-pathway architecture for safe and 
deployable autonomous driving. Inspired by the 
human brain's habitual and deliberative visual 
processing, DriveVLM-RL decomposes semantic 
reward learning into a \textbf{Static Pathway} 
for continuous spatial safety assessment via 
CLIP-based contrasting language goals, and a 
\textbf{Dynamic Pathway} for attention-gated 
multi-frame semantic risk reasoning via a 
lightweight detection model and large 
VLM~(LVLM). A hierarchical reward synthesis 
mechanism fuses these signals with vehicle state 
information, while an asynchronous training 
pipeline decouples expensive LVLM inference from 
environment interaction. Critically, all VLM 
components operate exclusively during offline 
training and are completely removed at 
deployment, eliminating inference latency at 
test time. Extensive experiments in the CARLA 
simulator demonstrate that DriveVLM-RL
significantly outperforms state-of-the-art
baselines in collision avoidance and task success,
attaining the highest success rate while reducing
collision severity from 10.09 to 1.75~km/h relative
to the strongest VLM-based baseline. Notably, even
under extreme ``no-reward-after-collision'' settings
where explicit collision penalties are removed,
DriveVLM-RL maintains low collision rates
through semantic risk reasoning alone,
incurring the lowest rate of collisions with
vulnerable road users among all compared methods.
Under distribution
shift across unseen towns and traffic densities, its
safety advantage persists most clearly, yielding the
lowest collision severity in every out-of-distribution
town. These results demonstrate 
that DriveVLM-RL provides a practical paradigm 
for integrating foundation models into autonomous 
driving without compromising real-time 
feasibility. The demo video, code, and model
checkpoints are available at:
\href{https://zilin-huang.github.io/DriveVLM-RL-website/}
{\textcolor{magenta}{https://zilin-huang.github.io/DriveVLM-RL-website/}}.
\end{abstract}
		

\end{frontmatter}

\section{Introduction}

The deployment of autonomous vehicles (AVs) in real-world traffic environments has accelerated rapidly in recent years, transitioning from controlled testing scenarios to large-scale urban operations. In late 2025, Tesla released version 14 of its Full Self-Driving (FSD) system, representing a significant advancement in end-to-end neural network-based driving pipelines \citep{tesla2025fsd}. Currently, Waymo has expanded its robotaxi services in multiple U.S. cities \citep{kolodny2025waymo}, while in China, Baidu's Apollo Go scaled its operations to 22 cities and launched international deployments in Dubai and Abu Dhabi \citep{carnewschina2025apollo}.  However, despite these advances, ensuring safe and reliable decision-making in open-world environments remains a fundamental challenge that directly impacts public trust and regulatory acceptance \citep{feng2023dense,jiao2025evadrive,luo2025mtrdrive}. In complex traffic scenarios that involve diverse road users, ambiguous intent, and long-tail events, autonomous driving systems must consider semantic risks beyond purely geometric perception \citep{han2025dme}. 

Two dominant paradigms have emerged for end-to-end learning-based AV decision-making: imitation learning (IL) and reinforcement learning (RL), as shown in  Fig.~\ref{fig1} (a) \citep{huang2024human}. IL methods learn driving policies by mimicking expert demonstrations, offering simplicity and benefiting from abundant naturalistic driving data. Yet, IL suffers from well-known limitations \citep{ross2011reduction,de2019causal}: (1) distribution shift, where the learned policy encounters states not represented in the training data and fails to recover; (2) bounded performance, meaning the policy cannot surpass the capability of the demonstrator; and (3) causal confusion, in which spurious correlations in demonstrations lead to brittle or unsafe behaviors. These limitations are especially problematic for safety-critical scenarios that rarely appear in naturalistic data but are essential for robust autonomous driving. In contrast, RL offers a principled alternative by enabling agents to learn through trial-and-error interaction with the environment, potentially discovering novel strategies that surpass human performance \citep{sutton1998reinforcement}. RL has achieved superhuman capabilities in domains ranging from board games to robotic manipulation \citep{mnih2015human}. For autonomous driving, RL holds the promise of learning adaptive policies that can handle rare but critical scenarios through closed-loop exploration \citep{huang2025pe,he2024trustworthy,aradi2020survey}. However, RL's effectiveness critically depends on the design of reward functions that accurately encode desired driving behaviors.

\begin{figure*}[t]
\centering
  \includegraphics[width=0.99999\textwidth]{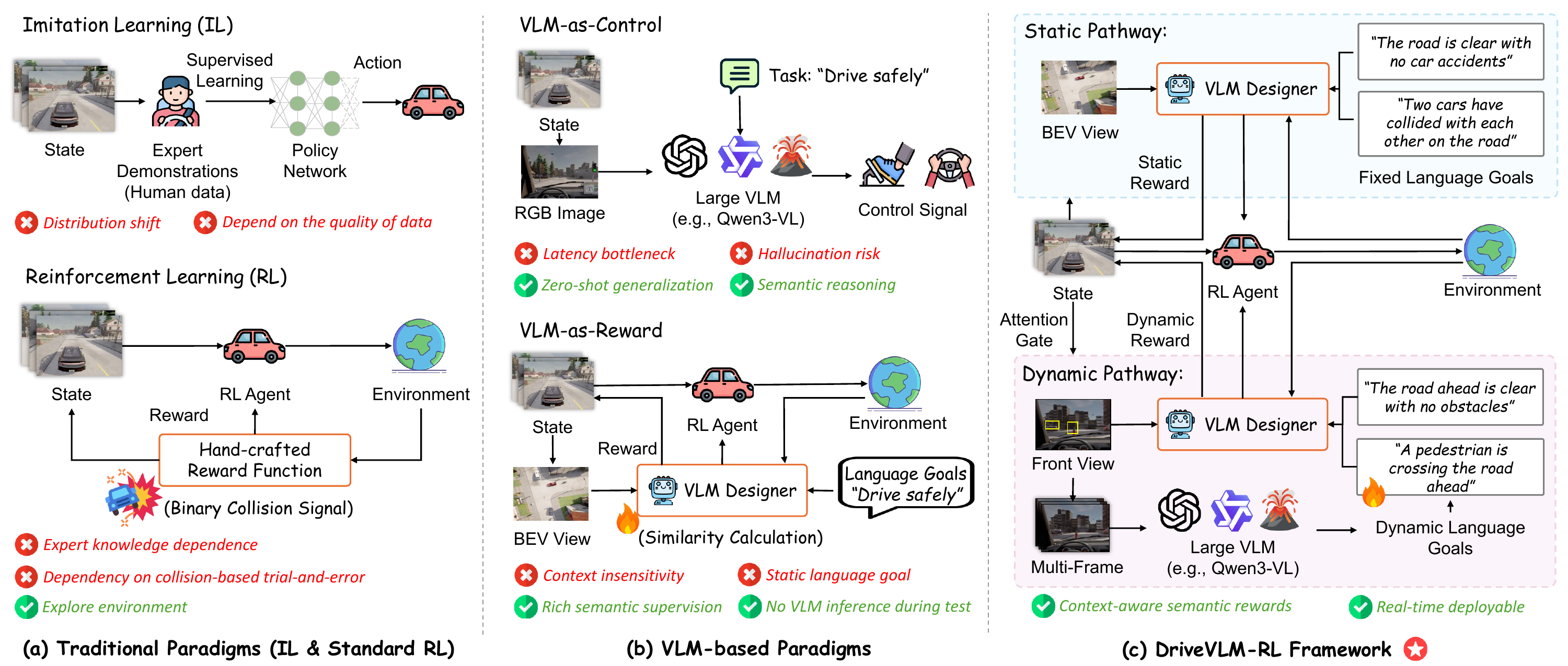}
  \caption{
Comparative learning paradigms for autonomous driving. (a) Traditional policy learning approaches, including IL and RL, which rely on expert demonstrations or hand-crafted rewards. (b) Foundation model-based approaches, including VLM-as-Control and VLM-as-Reward paradigms. (c) The proposed DriveVLM-RL framework, which integrates a dual-pathway architecture to enable dynamic, context-aware semantic rewards while remaining real-time deployable.
}
  \label{fig1}
\end{figure*}

Reward function design plays a pivotal role in RL, as it directly shapes the agent's behavior and determines the quality of the learned driving policy \citep{sutton1998reinforcement,lu2025discovery}. In autonomous driving, reward functions are generally hand-crafted based on expert intuition, typically combining sub-objectives such as speed maintenance, lane tracking, and collision avoidance. Prior studies \citep{knox2023reward,abouelazm2024review} highlight several inherent limitations, including dependence on expert knowledge, conflicting objectives, and poor generalization. More critically, traditional reward functions rely on binary collision signals to learn safety, meaning the agent must physically experience crashes to recognize dangerous situations. This creates a fundamental barrier to real-world deployment: \textit{allowing an AV to learn from actual collisions is unacceptable} \citep{huang2025pe,wu2024recent,garcia2015comprehensive}. Moreover, safe driving involves nuanced judgments that are difficult to encode in hand-crafted reward functions \citep{han2024autoreward,zhou2024human}. For instance, a pedestrian standing on the sidewalk versus stepping onto the road requires fundamentally different responses, yet both scenarios may appear similar in terms of distance-based metrics. Human drivers rely on rich semantic understanding, including intent, context, and social norms, which cannot be represented by simple geometric or physics-based reward terms.

The emergence of foundation models, such as large language models (LLMs) and vision-language models (VLMs), provide a promising alternative for addressing this limitation \citep{cui2024survey,jiang2025survey,hazra2024revolve, sheng2026curricuvlm, pang2026large}. By jointly reasoning over visual observations and natural language, VLMs can infer high-level semantic concepts such as risk, intent, and social context, which are difficult to encode through conventional reward engineering. Several recent studies have explored ``VLM-based control" paradigms (Fig.~\ref{fig1} (b), top), where VLMs directly map visual inputs to driving actions or generate real-time control commands \citep{tian2025drivevlm, qian2025agentthink,zhou2025autovla,you2026v2x}. However, these  paradigms suffer from two fundamental issues: (1) the computational latency of 500–2000 ms per inference far exceeds the 20–100 ms control cycles required for safe vehicle operation \citep{cui2024survey,zhou2024vision}. (2) VLMs are susceptible to hallucination, producing outputs that may be inconsistent with the visual input, which could lead to catastrophic failures when used directly for vehicle control \citep{xie2025vlms,meng2025your}. Recently, researchers have proposed an alternative paradigm: rather than using VLMs for direct control,  VLMs are integrated with RL to provide semantic understanding that shapes reward signals and guides policy learning  (Fig.~\ref{fig1} (b), bottom). This allows policies to leverage VLM semantics while avoiding the latency and reliability constraints of real-time control. 

This ``VLM-as-Reward" paradigm has shown promising results in robotic domain \citep{lee2026roboreward}, such as VLM-SR \citep{baumli2023vision}, RoboCLIP \citep{sontakke2023roboclip}, DriveMind \citep{wasif2025drivemind}, and VLM-RM \citep{rocamonde2023vision}, which leverage contrastive language-image pre-training (CLIP)'s semantic embeddings to measure goal achievement from visual observations. However, unlike robotic domain where goals can be precisely specified, driving objectives expressed in natural language (e.g., ``drive safely") are inherently ambiguous and difficult to translate into dense, informative reward signals. To address this ambiguity, LORD \citep{ye2025lord} proposed using negative language goals to describe dangerous states. Our previous work \citep{huang2025vlm} further proposed contrasting language goals (CLGs) that leverage both positive and negative descriptions. Despite these advances, several challenges remain unresolved: (1) Most existing methods use CLIP to compute similarity between observations and fixed language goals. Such static formulations lack contextual awareness and cannot capture the dynamic, evolving, and temporally dependent nature of traffic risk. (2) A natural solution is to adopt large VLM (LVLM), such as GPT \citep{achiam2023gpt} or Qwen-VL \citep{yang2025qwen3}, to perform multi-frame semantic reasoning. However, invoking LVLM for every frame during RL training is computationally prohibitive, creating severe scalability bottlenecks when millions of environment interactions are required. 

To address these challenges, we draw inspiration from the human brain's visual processing architecture, which has evolved to efficiently balance routine perception and context-dependent reasoning. As shown in Fig. \ref{fig2}, during routine driving, the brain operates in an efficient habitual mode via sensorimotor loops mediated by the parietal cortex \citep{goodale1992separate}. When safety-critical situations occur, such as the sudden appearance of a pedestrian, the brain's selective attention network rapidly engages to redirect cognitive focus \citep{corbetta2002control}, triggering higher-order semantic reasoning in the prefrontal cortex \citep{miller2001integrative}, such as ``A pedestrian is stepping into the roadway; I must slow down and prepare to stop''. The critical insight is that the brain does not perform deep, energy-intensive analysis on every visual frame; instead, it employs an attention gate \citep{desimone1995neural} to determine when to invoke slower but more powerful contextual reasoning processes. Inspired by the brain's dual-pathway cognitive architecture, we propose \textbf{DriveVLM-RL}, a neuroscience-inspired cognitive framework that integrates VLMs into RL for safe and deployable autonomous driving.  As illustrated in  Fig.~\ref{fig1} (c), DriveVLM-RL fundamentally treats VLMs as semantic teachers rather than real-time decision-makers, alleviating the reliance of traditional RL on collision-based learning while avoiding the latency and hallucination issues of ``VLM-as-Control''. 

The main contributions of this work are summarized as follows:

\begin{itemize}
    \item We propose DriveVLM-RL, to the best of our knowledge the first framework to explicitly integrate the human brain's dual-pathway cognitive architecture into the VLM-as-Reward paradigm for autonomous driving. The \textbf{Static Pathway} (using CLIP and fixed language goals) simulates the brain's dorsal stream for continuous spatial awareness, while the \textbf{Dynamic Pathway} (using attention-gated LVLM reasoning) simulates the brain's selective attention-prefrontal cortex circuit for higher-order semantic risk reasoning. This hierarchical design enables the RL agent to learn complex safety maneuvers through semantic understanding, alleviating the fundamental limitations of traditional collision-based reward functions.
    \item We design a novel \textbf{attention gating mechanism} that simulates the brain's selective attention to address the computational complexity and high inference latency inherent in VLM-as-Reward paradigms. A lightweight perception model first analyzes foreground scenes, effectively filtering routine driving frames and triggering computationally expensive LVLM inference only when safety-critical objects are detected. Critically, the VLM is invoked exclusively during the RL training phase; once training is complete, no VLM calls are required during deployment, eliminating the infeasible latency that plagues VLM-as-Control methods.
    \item We introduce a hierarchical reward synthesis mechanism that fuses static visual-language similarity, dynamic multi-frame semantic reasoning, and vehicle state information into dense and proactive reward signals. Furthermore, we design an asynchronous training pipeline that decouples expensive VLM inference from environment interaction, enabling scalable learning despite the high computational cost of large models. Critically, all VLM components operate exclusively during offline training and are completely removed during deployment, allowing the final driving policy to execute with low latency.
    \item Extensive experiments conducted in CARLA simulator \citep{dosovitskiy2017carla} demonstrate that DriveVLM-RL significantly improves driving safety and task success, and that its safety advantage persists under distribution shift to unseen towns and traffic densities. Remarkably, even under extreme ``no-reward-after-collision'' settings where explicit collision penalties are removed, agents trained with DriveVLM-RL still learn to avoid collisions through semantic risk reasoning alone. These results indicate that DriveVLM-RL provides a practical and generalizable paradigm for leveraging foundation models to train autonomous driving policies that are both safe and deployable in real-world systems.
\end{itemize}

The remainder of this paper is organized as follows. Section \ref{Preliminaries} presents the preliminaries on RL formulation and the VLM-as-Reward paradigm. Section \ref{Framework: DriveVLM-RL} details the proposed DriveVLM-RL framework. Section \ref{Experiments and Results} presents the experimental setup and results. Section \ref{Conclusion} concludes the paper and outlines future research directions.

\begin{figure*}[!t]
\centering
  \includegraphics[width=0.99999\textwidth]{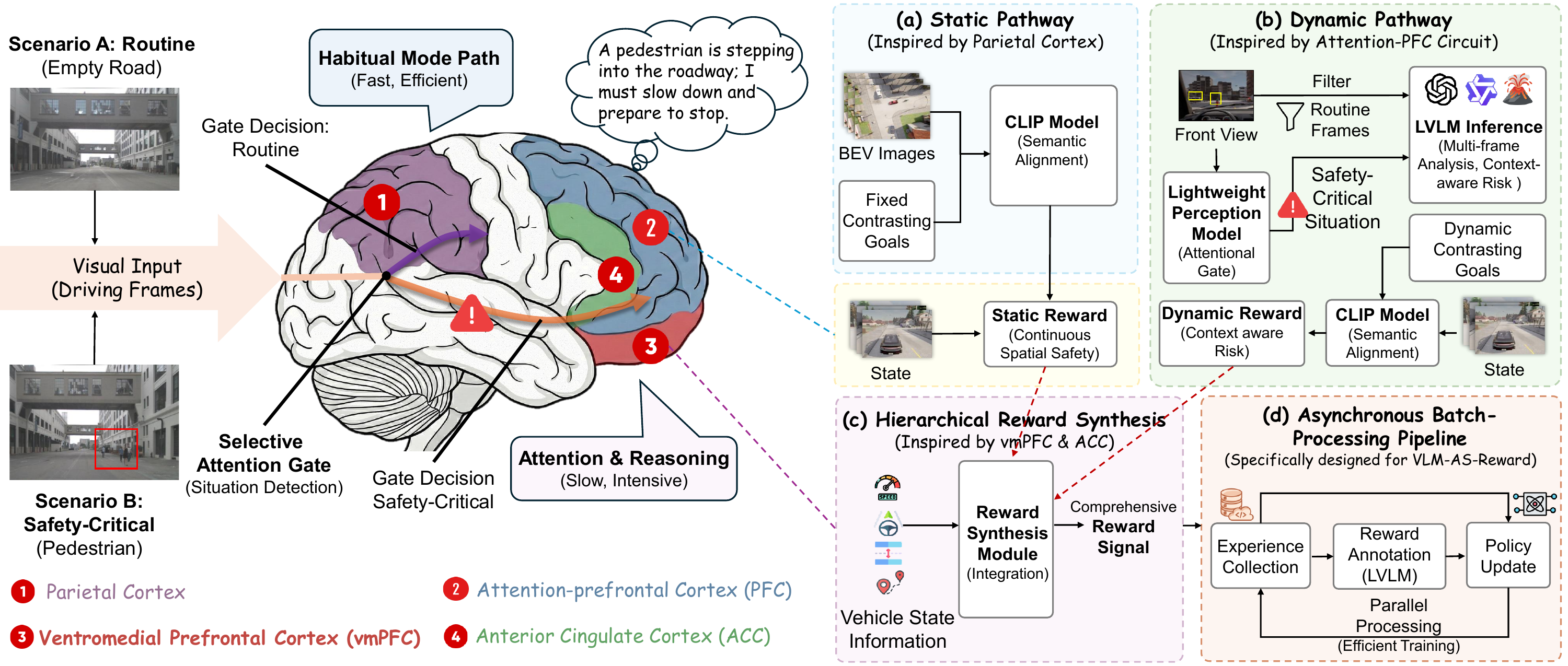}
  \caption{Neuroscience-inspired motivation of DriveVLM-RL. The framework is inspired by the brain's habitual and deliberative visual processing: routine scenes are handled by a fast pathway, while safety-critical situations trigger attention and higher-level semantic reasoning, motivating a dual-pathway reward learning design.}
  \label{fig2} 
\end{figure*}

\section{Preliminaries}
\label{Preliminaries}
\subsection{Markov Decision Process Formulation}
We model the autonomous driving decision-making task as a Partially Observable Markov Decision Process (POMDP) \citep{kaelbling1998planning}, defined by the tuple $(\mathcal{S}, \mathcal{A}, \mathcal{T}, \mathcal{O}, R, \phi, \gamma, d_0)$. Here, $\mathcal{S}$ denotes the state space, $\mathcal{A}$ is the action space, $\mathcal{T}(s' \mid s, a)$ is the state transition function, $R(s,a)$ is the reward function, $\mathcal{O}$ is the observation space, $\phi(o \mid s)$ is the observation emission function, $\gamma \in [0,1)$ is the discount factor, and $d_0(s)$ is the initial state distribution. At each timestep $t$, the agent receives an observation $o_t \in \mathcal{O}$ from the environment and selects an action $a_t \in \mathcal{A}$ according to its policy $\pi(a_t \mid o_t)$. In our framework, the observation $o_t$ comprises bird's-eye-view (BEV) representations and front-view camera images, while the action $
a_t = (a_t^{\text{steer}}, a_t^{\text{throttle}})$ consists of a continuous steering command and a combined throttle/brake command. The environment transitions to a new state $s_{t+1} \sim \mathcal{T}(\cdot \mid s_t, a_t)$, and the agent receives a scalar reward $r_t = R(o_t, a_t)$, where we use observation $o_t$ in place of the latent state $s_t$ due to partial observability. The learning objective is to find an optimal policy $\pi^*$ that maximizes the expected discounted return:
$
\pi^* = \arg\max_{\pi} G(\pi) = \arg\max_{\pi} \mathbb{E}_{\pi}\left[ \sum_{t=0}^{T} \gamma^t r_t \right]$.

\subsection{VLM-as-Reward Paradigm}
VLMs are models capable of jointly processing language inputs $l \in \mathcal{L}^{\leq n}$ and visual inputs $i \in \mathcal{I}^{\leq m}$, where $\mathcal{L}$ denotes a finite vocabulary and $\mathcal{I}$ the space of RGB images. A prominent class of VLMs is based on CLIP~\citep{radford2021learning}. CLIP consists of a language encoder $f_L : \mathcal{L}^{\leq n} \to \mathcal{E}$ and an image encoder $f_I : \mathcal{I} \to \mathcal{E}$, both mapping inputs into a shared embedding space $\mathcal{E} \subseteq \mathbb{R}^d$. These encoders are jointly trained via contrastive learning on large-scale image-caption pairs, minimizing the cosine distance for semantically aligned pairs while maximizing it for mismatched pairs. The alignment capability of CLIP enables the VLM-as-Reward paradigm, where semantic similarity between visual observations and language goals serves as a reward signal for RL training. Given an image encoder $f_I$, a language encoder $f_L$, a visual observation $o_t$, and a language goal $l$, the VLM-based reward is usually defined as $
r_t^{\text{VLM}} = \mathrm{sim}(f_I(o_t), f_L(l))$, where $\mathrm{sim}(\cdot, \cdot)$ denotes cosine similarity. This formulation provides the technical foundation for the VLM-as-Reward paradigm.

\subsection{Problem Statement}
A key challenge is designing an effective reward function $R(o_t, a_t)$ that guides the agent toward safe and efficient behaviors. Traditional reward engineering requires manual specification and extensive tuning of multiple sub-objectives, which is labor-intensive, error-prone, and difficult to generalize across diverse driving scenarios. The VLM-as-Reward paradigm offers a promising alternative by leveraging VLMs to provide semantically grounded reward signals. Ideally, we seek a reward function of the form:
\begin{equation}
R_{\text{VLM}}(o_t) = \Phi(l, o_t, c;\, \theta_{\text{VLM}})
\label{eq1}
\end{equation}
where $l$ is a linguistic goal specification, $o_t$ is the current observation, $c \in \mathcal{C}$ is optional contextual information (e.g., multi-frame history or scene description), and $\theta_{\text{VLM}}$ denotes the frozen VLM parameters. When $c = \emptyset$, the formulation reduces to the standard CLIP-based reward $r_t^{\text{VLM}} = \mathrm{sim}(f_I(o_t), f_L(l))$.

Most existing methods employ CLIP with fixed language goals (negative or positive) to compute reward signals based on text--image similarity. However, traffic scenes involve rich temporal and contextual variations, and this setting cannot capture factors such as pedestrian trajectory, motion intent, or evolving environmental conditions. A natural extension is to use LVLM to generate dynamic language goals conditioned on the current scene. Given an image $o_t$ and a text prompt $p$, an LVLM generates a response $y = \text{LVLM}(o_t, p)$. While LVLM can analyze driving scenarios and provide nuanced semantic assessments, they incur substantially higher inference latency compared to CLIP. This latency far exceeds the requirements for real-time RL training, where millions of environment steps must be evaluated, making per-frame LVLM inference computationally infeasible. This motivates the attention-gated dual-pathway 
design of DriveVLM-RL, detailed in Section~\ref{Framework: DriveVLM-RL}.

\section{Framework: DriveVLM-RL}
\label{Framework: DriveVLM-RL}

\begin{figure*}[t]
\centering
  \includegraphics[width=1\textwidth]{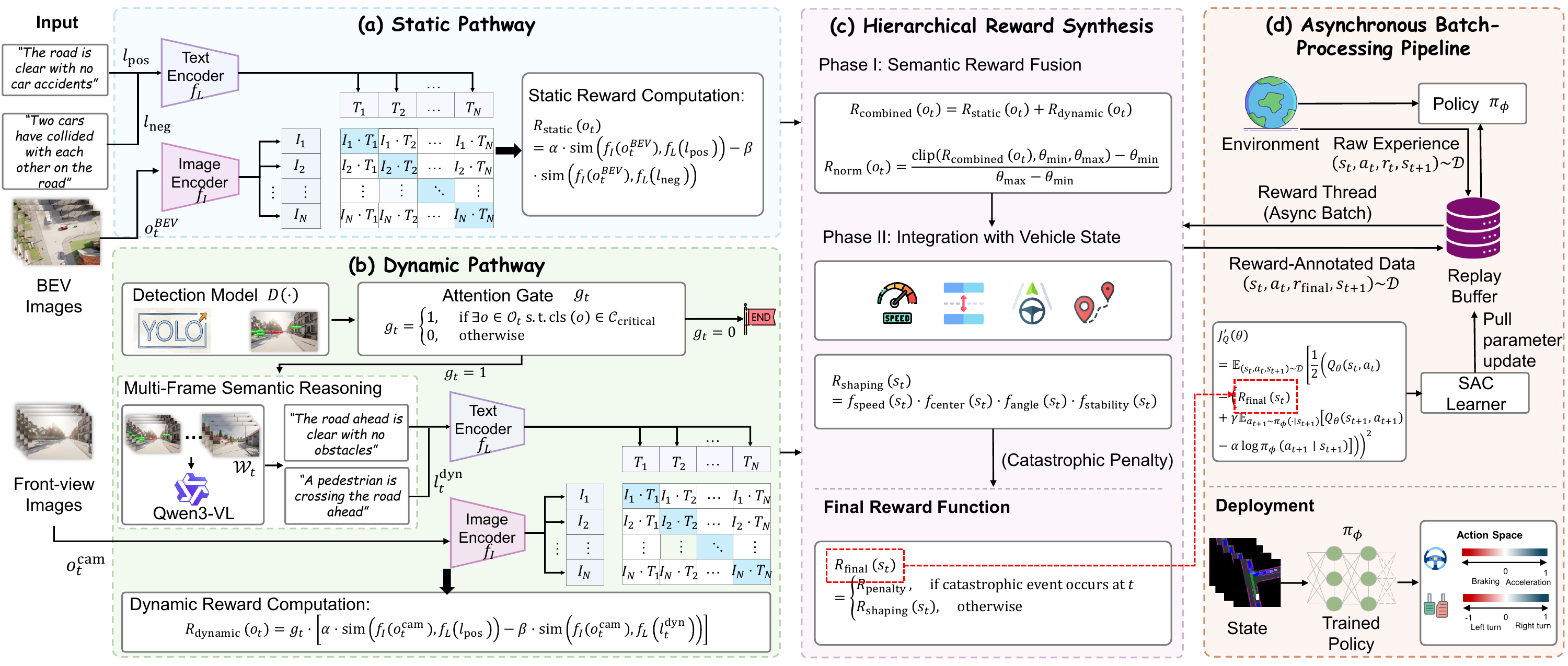}
  \caption{Overview of DriveVLM-RL. (a) Static Pathway: CLIP-based semantic alignment with contrasting language goals to provide continuous spatial safety assessment. (b) Dynamic Pathway: an attention-gated mechanism triggers multi-frame LVLM reasoning only in safety-critical situations. (c) Hierarchical reward synthesis: static and dynamic semantic signals are fused and integrated with vehicle-state factors to produce the final shaping reward. (d) Asynchronous training pipeline: reward computation is decoupled from environment interaction and policy learning.}
  \label{fig3} 
\end{figure*}

\subsection{Overview}
The DriveVLM-RL framework is a neuro-inspired cognitive architecture designed to address the reward design challenge in RL-based autonomous driving while overcoming the semantic and contextual limitations of existing VLM-as-Reward methods. As illustrated in Fig.~\ref{fig2}, the framework comprises four main components: \textbf{(1) Static Pathway.} Simulating the brain's dorsal visual stream (parietal cortex), this pathway utilizes a pre-trained CLIP model to compute semantic alignment between BEV images and fixed CLGs, providing continuous spatial safety assessment. \textbf{(2) Dynamic Pathway.} Simulating the brain's attention-prefrontal cortex (PFC) circuit, this pathway employs a lightweight perception model as an attentional gate to filter routine frames, triggering computationally expensive LVLM inference only when safety-critical situations are detected. The LVLM analyzes multi-frame sequences to generate dynamic, context-aware risk descriptions. \textbf{(3) Hierarchical Reward Synthesis.} Simulating the ventromedial prefrontal cortex (vmPFC) and anterior cingulate cortex (ACC), this module integrates static and dynamic rewards with vehicle state information to produce comprehensive reward signals. \textbf{(4) Asynchronous Batch-Processing Pipeline.} To enable efficient training despite LVLM latency, this pipeline decouples reward computation from environment interaction, allowing parallel processing of experience collection, reward annotation, and policy updates. We describe each component in detail in the following subsections, as shown in Fig.~\ref{fig3}.

\subsection{Static Pathway}
The Static Pathway simulates the parietal cortex (dorsal visual stream), which mediates habitual sensorimotor processing for spatial tasks such as lane keeping and distance maintenance. This pathway generates continuous reward signals based on foundational spatial safety assessment.

\subsubsection{Static Reward Computation}
The input for this pathway is the agent's BEV image $o_t^{\text{BEV}}$, as this top-down representation provides unambiguous spatial relationships without the occlusion inherent in first-person views, analogous to the parietal cortex's integrated spatial map. However, a single abstract goal (e.g., ``drive safely'') is semantically ambiguous and provides weak reward signals~\citep{ye2025lord}. We therefore introduce a CLG formulation~\citep{huang2025vlm} that compares desired and undesired outcomes to produce more discriminative rewards.

\begin{definition}[Static Contrasting Language Goal]
Given the driving task, the Static CLG is a fixed pair $(l_{\text{pos}}, l_{\text{neg}}) \in \mathcal{L}^{\leq n} \times \mathcal{L}^{\leq n}$, where $l_{\text{pos}}$ describes the desired baseline state and $l_{\text{neg}}$ describes the fundamental undesired state. For this pathway, we define:
\begin{itemize}
    \item Positive Goal ($l_{\text{pos}}$): ``The road is clear with no car accidents.''
    \item Negative Goal ($l_{\text{neg}}$): ``Two cars have collided with each other on the road.''
\end{itemize}
\label{definition1}
\end{definition}

Based on this definition, we can formalize the static reward function. Specifically, we employ the CLIP model \citep{radford2021learning} as the foundation for semantic reward computation.

\begin{definition}[Static Reward]
Given the pre-trained CLIP model with image encoder $f_I: \mathcal{I} \rightarrow \mathcal{E}$ and language encoder $f_L: \mathcal{L}^{\leq n} \rightarrow \mathcal{E}$ mapping into a shared latent space $\mathcal{E} \subseteq \mathbb{R}^d$, the static CLG pair $(l_{\text{pos}}, l_{\text{neg}})$, the static reward for observation $o_t^{\text{BEV}}$ is:
\begin{equation}
R_{\text{static}}(o_t) 
= \alpha \cdot \mathrm{sim}(f_I(o_t^{\text{BEV}}), f_L(l_{\text{pos}}))
- \beta \cdot \mathrm{sim}(f_I(o_t^{\text{BEV}}), f_L(l_{\text{neg}}))
\label{eq2}
\end{equation}
where $\mathrm{sim}(\cdot, \cdot)$ denotes cosine similarity between embeddings:
\begin{equation}
\mathrm{sim}(v_1, v_2) = \frac{v_1^\top v_2}{\|v_1\| \, \|v_2\|}
\label{eq3}
\end{equation}
and $\alpha, \beta > 0$ are weighting factors with $\alpha + \beta = 1$. For simplicity, we set $\alpha = \beta = 0.5$ in this work. This formulation yields a continuous reward $R_{\text{static}} \in [-1, 1]$, encouraging states semantically similar to $l_{\text{pos}}$ while penalizing those similar to $l_{\text{neg}}$.
\label{definition2}
\end{definition}

\subsubsection{Theoretical Properties}

We then establish theoretical guarantees for this static reward formulation.

\begin{lemma}[Boundedness]
\label{lemma:Boundedness}
For any observation $o_t$ and CLG pair $(l_{\text{pos}}, l_{\text{neg}})$, the static reward is bounded: $
R_{\text{static}}(o_t) \in [-1, 1]$.
\end{lemma}

\begin{lemma}[Discriminability]
\label{lemma:Discriminability}
The CLG formulation provides strictly greater reward discrimination than single-goal formulations. Specifically, for observations $o_1, o_2$ where $
\mathrm{sim}(f_I(o_1), f_L(l_{\text{pos}})) = \mathrm{sim}(f_I(o_2), f_L(l_{\text{pos}}))$
but $\mathrm{sim}(f_I(o_1), f_L(l_{\text{neg}})) \neq \mathrm{sim}(f_I(o_2), f_L(l_{\text{neg}}))$, we have $R_{\text{static}}(o_1) \neq R_{\text{static}}(o_2)$, even when single-goal similarity fails to distinguish the two states.
\end{lemma}

Building on these properties, we establish that the CLG formulation induces a well-defined preference ordering over states.

\begin{theorem}[Reward-Induced State Ordering]
\label{theorem:Reward-Induced State Ordering}
Let $\mathcal{S}$ be the state space and define the binary relation $\succeq$ on $\mathcal{S}$ such that $s_1 \succeq s_2$ if and only if $R_{\text{static}}(s_1) \geq R_{\text{static}}(s_2)$. Then $\succeq$ is a total preorder (reflexive, transitive, and total), inducing a consistent preference ranking over states aligned with the semantic safety specification $(l_{\text{pos}}, l_{\text{neg}})$.
\end{theorem}
The proofs of Lemmas \ref{lemma:Boundedness}--\ref{lemma:Discriminability} follow directly from the cosine similarity bounds established in our previous work~\citep{huang2025vlm}. The proof of Theorem \ref{theorem:Reward-Induced State Ordering} is provided in ~\ref{appendices1}.

\subsection{Dynamic Pathway}
The Static Pathway provides spatial safety assessment but cannot handle complex events requiring semantic understanding. The Dynamic Pathway addresses this limitation by simulating the brain's attention-PFC circuit, which operates on a ``when-needed'' basis: a fast attentional mechanism identifies salient stimuli and gates the activation of slower, high-level reasoning. As illustrated in Fig.~\ref{fig4}, this attention-gated VLM reasoning mechanism selectively triggers expensive semantic analysis only when safety-critical situations are detected, achieving computational efficiency while preserving information fidelity for critical scenarios.

\subsubsection{Attentional Gate}
The human brain does not expend cognitive resources processing all visual input through the PFC; subcortical structures filter stimuli and forward only salient information. We implement this mechanism using a lightweight object detection model.

\begin{definition}[Attentional Gate]
Let $o_t^{\text{cam}}$ be the front-view camera image at time $t$. A detection model $D(\cdot)$ produces detected objects $\mathcal{O}_t = D(o_t^{\text{cam}})$. Given a predefined set of safety-critical classes, the binary gate $g_t$ is defined as:
\begin{equation}
g_t = 
\begin{cases}
1, & \text{if } \exists o \in \mathcal{O}_t \text{ s.t. } \text{cls}(o) \in \mathcal{C}_{\text{critical}} \\
0, & \text{otherwise}
\end{cases}
\label{eq4}
\end{equation}
where $\text{cls}(o)$ returns the class label of object $o$. 
\label{definition3}
\end{definition}

We employ YOLOv8 \citep{yolov8} as the detection model $D(\cdot)$. We set $\mathcal{C}_{\text{critical}}$ to include 11 safety-critical object categories: $\{\text{person, bicycle, motorcycle, dog, horse, sheep, cow, elephant, bear, zebra, giraffe}\}$. This set covers a wide variety of road users  whose unpredictable behavior most benefits from semantic reasoning beyond spatial proximity. While the core safety-critical categories are person, bicycle, and motorcycle, we include 
additional animal categories to improve robustness against rare but high-risk long-tail scenarios. Vehicle-to-vehicle conflicts are primarily captured by the Static Pathway's BEV-based spatial assessment, which provides sufficient signal for structured traffic scenarios.

\subsubsection{Multi-Frame Semantic Reasoning}
When $g_t = 1$, the framework simulates the PFC for semantic reasoning and the hippocampus for temporal context integration. We construct a causal temporal window of the most recent $K+1$ frames: $\mathcal{W}_t = \{o_{t-K}^{\text{cam}}, \ldots, o_t^{\text{cam}}\}$, where $K$ is a hyperparameter controlling the temporal context window size. In our implementation, we set $K=3$, providing a 3-frame temporal window. This window enables the VLM to understand motion dynamics and intentions rather than static snapshots.

\begin{definition}[Dynamic Language Goal]
The dynamic language goal $l_t^{\text{dyn}}$ is generated by a  LVLM $F_{\text{LVLM}}(\cdot)$ conditioned on the temporal window and detected objects:
\begin{equation}
l_t^{\text{dyn}} = F_{\text{LVLM}}(\mathcal{W}_t, \mathcal{O}_t)
\label{eq5}
\end{equation}
This output serves as a semantic hypothesis about the current risk (e.g., ``A pedestrian is crossing the road ahead''). 
\label{definition4}
\end{definition}

We employ Qwen3-VL \citep{yang2025qwen3} as $F_{\text{LVLM}}(\cdot)$ in our implementation. Note that $F_{\text{LVLM}}(\cdot)$ is a LVLM that produces natural language descriptions, distinct from the CLIP encoders used for reward computation.

\subsubsection{Dynamic Reward Computation}
The dynamic reward converts the VLM's semantic understanding into a numerical signal. Unlike the Static Pathway's fixed goal, this pathway uses the dynamically generated $l_t^{\text{dyn}}$ as the context-specific risk description. The reward is computed using the same CLIP encoders as the Static Pathway, ensuring consistent semantic alignment.

\begin{definition}[Dynamic Reward]
Given the attentional gate $g_t$, dynamic language goal $l_t^{\text{dyn}}$, and static positive goal $l_{\text{pos}}$, the dynamic reward is:
\begin{equation}
R_{\text{dynamic}}(o_t) 
= g_t \cdot \Big[ 
\alpha \cdot \mathrm{sim}(f_I(o_t^{\text{cam}}), f_L(l_{\text{pos}})) 
- \beta \cdot \mathrm{sim}(f_I(o_t^{\text{cam}}), f_L(l_t^{\text{dyn}})) 
\Big]
\label{eq6}
\end{equation}

This formulation ensures $R_{\text{dynamic}} = 0$ for non-critical frames ($g_t = 0$). When triggered, it provides sparse but semantically rich penalty signals that capture complex risks beyond spatial proximity.
\label{definition5}
\end{definition}

\begin{figure*}[t]
\centering
  \includegraphics[width=1\textwidth]{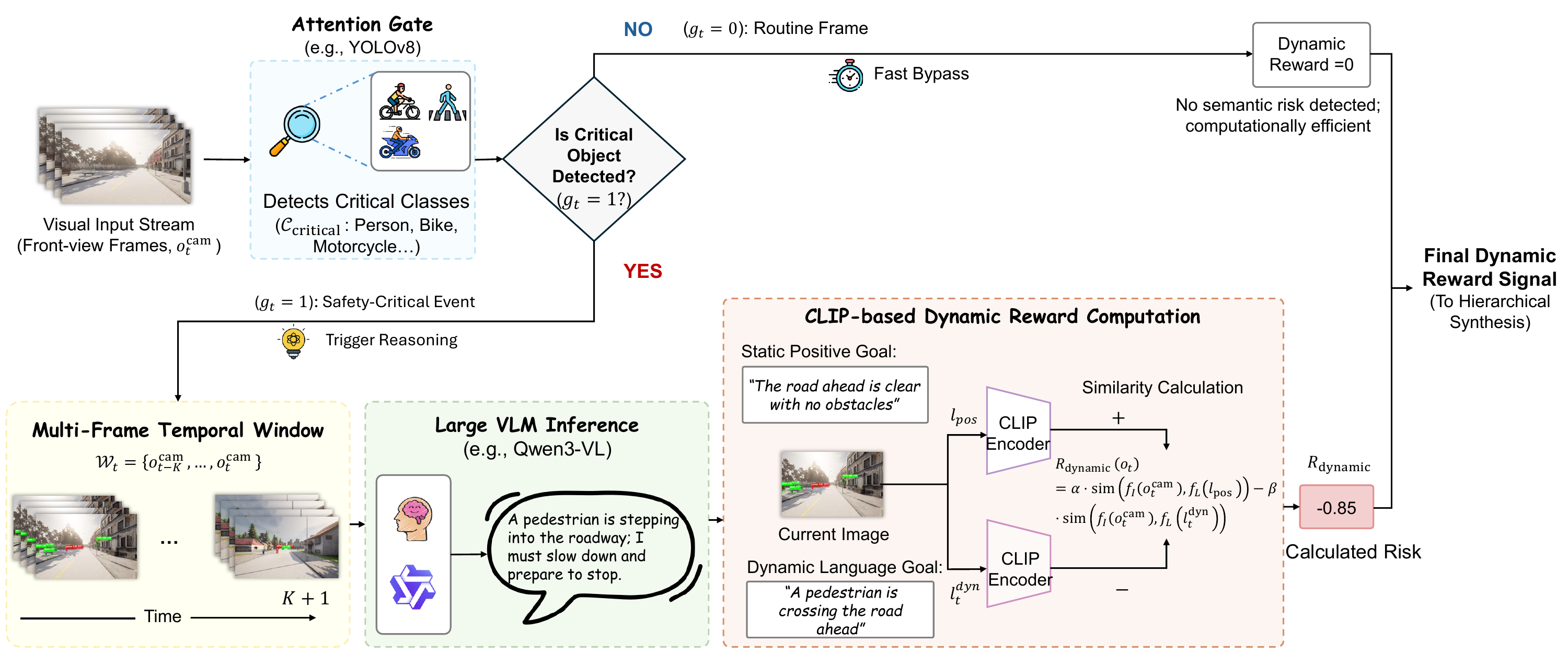}
  \caption{Attention-gated dynamic reward generation in DriveVLM-RL. Routine frames bypass semantic reasoning, while safety-critical frames trigger multi-frame LVLM inference to produce a risk description, which is converted into a dynamic reward via CLIP-based semantic similarity.}
  \label{fig4} 
\end{figure*}

\subsubsection{Theoretical Properties}

We analyze the computational efficiency and information-theoretic properties of the attentional gating mechanism.

\begin{lemma}[Computational Efficiency]
Let $p = P(g_t = 1)$ be the probability of gate activation, and let $T_{\text{LVLM}}$, $T_{\text{det}}$ denote the inference time of the LVLM and detection model respectively. The expected per-frame computation time of the Dynamic Pathway is $T_{\text{det}} + p \cdot T_{\text{LVLM}}$, compared to $T_{\text{LVLM}}$ for ungated approaches. When $p \ll 1$ and $T_{\text{det}} \ll T_{\text{LVLM}}$, this yields relative computational savings of approximately $(1-p)\times 100\%$ compared to ungated LVLM inference.
\end{lemma}

\begin{theorem}[Information Preservation under Gating]
Let $\mathcal{S}_{\text{critical}} \subseteq \mathcal{S}$ denote the set of safety-critical states, and let $\mu$ be a distribution over $\mathcal{S}_{\text{critical}}$. Assume the detection model $D(\cdot)$ achieves recall $\rho$ on $\mathcal{S}_{\text{critical}}$, assume $R_{\text{LVLM}}(s) \geq 0$ for all $s \in \mathcal{S}_{\text{critical}}$, and  $\mathbb{E}_\mu[R_{\text{LVLM}} \mid g=1] 
\geq \mathbb{E}_\mu[R_{\text{LVLM}}]$. Let $g(s)\in\{0,1\}$ denote the gating indicator for state $s$. Then:
\begin{equation}
\mathbb{E}_{s \sim \mu}[ g(s)\cdot R_{\text{LVLM}}(s) ] 
\ge \rho \cdot 
\mathbb{E}_{s \sim \mu}[ R_{\text{LVLM}}(s) ]
\label{eq7}
\end{equation}
\end{theorem}

Combining Lemma 3 and Theorem 2, the efficiency of the Dynamic Pathway is governed by the gate activation probability $p$, and its fidelity by the detector recall $\rho$. Empirically, $p$ is strongly scene-dependent: as reported in Table~\ref{tab:yolo_comparison}, it ranges from about $11\%$ in sparse scenes to $66\%$ in dense, safety-critical traffic. Because the lightweight detector is orders of magnitude cheaper than the LVLM ($T_{\text{det}} \ll T_{\text{LVLM}}$), gating skips LVLM inference on every routine frame and thus substantially reduces the expected reward-annotation cost relative to ungated per-frame inference across this entire range, while the detector's high recall on the safety-critical classes preserves semantic information for safety-critical states. A case study can be found in Section~\ref{sec:reward-viz}.

\begin{remark} The gating mechanism may fail to trigger VLM reasoning when safety-critical objects fall outside the predefined class set $\mathcal{C}_{\text{critical}}$, or when detection recall degrades due to occlusion, adverse weather, or domain shift. In such cases, during training-time reward synthesis, we fall back to the Static Pathway, which provides baseline spatial safety assessment. This graceful degradation ensures the framework remains functional, albeit with reduced semantic understanding, rather than failing catastrophically.\end{remark}

The proofs are provided in~\ref{appendices3}.

\subsection{Hierarchical Reward Synthesis}

The Static and Dynamic Pathways provide parallel assessments of spatial and semantic risk. As illustrated in Fig. \ref{fig2}, in the brain, such information is integrated by the vmPFC and ACC~\citep{rangel2008framework}, which synthesize diverse value signals into unified judgments guiding behavior. Our Hierarchical Reward Synthesis module performs this integration through a two-phase process.

\subsubsection{Phase I: Semantic Reward Fusion}
The first phase combines pathway outputs into a unified semantic score:
\begin{equation}
R_{\text{combined}}(o_t) = R_{\text{static}}(o_t) + R_{\text{dynamic}}(o_t)
\label{eq8}
\end{equation}

This additive formulation naturally handles attentional gating: when $g_t = 0$, $R_{\text{dynamic}} = 0$ and the score defaults to spatial assessment alone.

The combined score is then normalized to $[0,1]$:
\begin{equation}
R_{\text{norm}}(o_t) = 
\frac{\mathrm{clip}(R_{\text{combined}}(o_t), \theta_{\min}, \theta_{\max}) - \theta_{\min}}
{\theta_{\max} - \theta_{\min}}
\label{eq9}
\end{equation}
where $\mathrm{clip}(x,a,b)=\min(\max(x,a),b)$, and $\theta_{\min}, \theta_{\max}$ are empirically determined hyperparameters (e.g., $\theta_{\min}=-0.1$, $\theta_{\max}=0.2$). This normalized score $R_{\text{norm}}$ represents a unified semantic safety assessment.

\subsubsection{Phase II: Integration with Vehicle State}
The second phase uses the normalized safety score to modulate low-level control objectives, enabling hierarchical behavior guidance.

\begin{definition}[Shaping Reward]
The shaping reward integrates semantic safety with vehicle state factors:
\begin{equation}
R_{\text{shaping}}(o_t)
= f_{\text{speed}}(o_t)\cdot f_{\text{center}}(o_t)\cdot f_{\text{angle}}(o_t)\cdot f_{\text{stability}}(o_t)
\label{eq10}
\end{equation}
where:
\begin{itemize}
    \item $f_{\text{speed}}(o_t) = \max\!\left(0,\ 1 - \frac{|v_{\text{actual}} - v_{\text{desired}}|}{v_{\max}}\right)$ measures speed tracking, with $v_{\text{desired}} = R_{\text{norm}}(o_t)\cdot v_{\max}$. Here, $v_{\text{actual}} \in [0, v_{\max}]$ is the current vehicle speed, ensuring $f_{\text{speed}}(o_t) \in [0,1]$ by construction. This design encodes a safety-speed trade-off: higher semantic safety scores permit higher desired speeds, while perceived risk naturally induces conservative speed targets, down to $v_{\text{desired}} = 0$ in critical scenarios. Note that $R_{\text{norm}}(o_t)$ is computed prior to $R_{\text{shaping}}(o_t)$, using only the semantic reward components $R_{\text{static}}$ and $R_{\text{dynamic}}$ from Phase I;
    \item $f_{\text{center}}(o_t)$ evaluates lateral deviation from lane center;
    \item $f_{\text{angle}}(o_t)$ measures heading alignment with road direction;
    \item $f_{\text{stability}}(o_t)$ penalizes lateral oscillation.
\end{itemize}
Each factor is bounded in $[0,1]$.
\label{definition6}
\end{definition}

\begin{remark}
We adopt multiplicative composition rather than weighted summation for combining reward factors. This design choice ensures joint constraint satisfaction: if any factor approaches zero (e.g., severe lane deviation), the entire reward diminishes regardless of other factors. In contrast, additive formulations require careful calibration of relative weights and may allow agents to exploit high rewards in some dimensions to compensate for unsafe behaviors in others, leading to reward hacking.\end{remark} 

\subsubsection{Final Reward Function}
The final reward combines the dense shaping signal with a sparse penalty for catastrophic events:
\begin{equation}
R_{\text{final}}(o_t) =
\begin{cases}
R_{\text{penalty}}, & \text{if catastrophic event occurs at } t \\
R_{\text{shaping}}(o_t), & \text{otherwise}
\end{cases}
\label{eq11}
\end{equation}
where $R_{\text{penalty}}\ll 0$ is a large negative constant applied upon collision with vehicles, pedestrians, or obstacles.

Importantly, our framework does not solely rely on this explicit punishment. As demonstrated in Section~\ref{Experiments and Results}, the agent learns safe policies even when $R_{\text{penalty}}=0$ under ``no-reward-after-collision'' settings, validating that $R_{\text{shaping}}$ provides sufficient proactive guidance to anticipate and avoid risks.

\subsubsection{Theoretical Properties}

We establish that the hierarchical reward synthesis preserves the convergence properties of the underlying RL algorithm. We first show that the final reward is bounded, a prerequisite for the soft actor--critic policy improvement guarantee.

\begin{corollary}[Bounded Final Reward]
\label{cor:bounded_final}
The normalized reward satisfies $R_{\text{norm}}(o_t) \in [0,1]$ by construction of the clipping operation in Eq.~(\ref{eq9}). Since each factor $f_{\text{speed}}, f_{\text{center}}, f_{\text{angle}}, f_{\text{stability}} \in [0,1]$ (the first by the $\max(0,\cdot)$ operator together with $v_{\text{desired}} \leq v_{\max}$; the remaining three by definition in Eq.~(\ref{eq10})), their product satisfies $R_{\text{shaping}}(o_t) \in [0,1]$. The final reward $R_{\text{final}}$ in Eq.~(\ref{eq11}) is therefore bounded: $R_{\text{final}}(o_t) \in [R_{\text{penalty}},\, 1]$.
\end{corollary}

\begin{theorem}[Policy Improvement Guarantee]
Let $\pi_k$ denote the policy at iteration $k$, and $\pi_{k+1}$ the updated policy obtained under the hierarchical reward $R_{\text{final}}$. Under standard assumptions of soft actor--critic learning, including bounded rewards, sufficient exploration, and stable function approximation, the policy update satisfies
\begin{equation}
J(\pi_{k+1}) \ge J(\pi_k) - \epsilon_k
\label{eq12}
\end{equation}
where $J(\pi) = \mathbb{E}_{\pi}\!\left[\sum_{t=0}^{T} \gamma^t R_{\text{final}}(o_t)\right]$, and $\epsilon_k$ denotes a bounded approximation error that diminishes as training progresses.
\end{theorem}

This guarantee follows from the SAC policy improvement theorem~\citep{haarnoja2018soft}, combined with the bounded reward property (Corollary~\ref{cor:bounded_final}) ensuring stable Q-value estimation. The hierarchical structure of $R_{\text{final}}$ does not interfere with convergence because all component rewards are bounded and the shaping reward $R_{\text{shaping}}$ is state-dependent only. While these assumptions are standard in theoretical RL analysis~\citep{haarnoja2018soft}, we empirically verify convergence behavior in Section~\ref{Experiments and Results} through training curves reported across three independent random seeds. The proofs are provided in~\ref{appendices4}.

\subsection{Asynchronous Batch-Processing Pipeline}

The LVLM inference required for $R_{\text{dynamic}}$ is computationally expensive and unsuitable for tight closed-loop interaction, making synchronous per-step reward computation impractical within the environment loop. We address this challenge through an asynchronous batch-processing pipeline that decouples reward calculation from experience collection.

\subsubsection{RL Algorithm}
We employ Soft Actor-Critic (SAC)~\citep{haarnoja2018soft} as the backbone RL algorithm due to its sample efficiency and stability in continuous control. Importantly, the DriveVLM-RL framework is algorithm-agnostic by design: since our hierarchical reward synthesis operates independently of the policy optimization procedure, it can in principle be combined with standard RL algorithms. This compatibility allows practitioners to leverage advances in RL algorithms while benefiting from our semantic reward design.

SAC maximizes the entropy-regularized objective:
\begin{equation}
J(\pi_\phi) = \mathbb{E}_{\pi_\phi} \left[ \sum_{t=0}^T \gamma^t \left( R(o_t, a_t) + \lambda \mathcal{H}(\pi_\phi(\cdot \mid o_t)) \right) \right]
\label{eq13}
\end{equation}

Its $Q$-function parameters $\theta$ are updated by minimizing the standard soft Bellman residual:
\begin{equation}
J_Q(\theta) = \mathbb{E}_{(o_t,a_t,r_t,o_{t+1}) \sim \mathcal{D}} 
\left[ \frac{1}{2} \left( Q_\theta(o_t,a_t) - 
\left( r_t + \gamma \mathbb{E}_{a_{t+1} \sim \pi_\phi(\cdot \mid o_{t+1})}
\left[ Q_\theta(o_{t+1}, a_{t+1}) - \lambda \log \pi_\phi(a_{t+1} \mid o_{t+1}) \right] \right) \right)^2 \right]
\label{eq14}
\end{equation}

Our core modification is to replace the standard immediate reward $r_t$ with our asynchronously computed hierarchical reward $R_{\text{final}}(o_t)$ from Eq.~(\ref{eq11}), yielding the modified Bellman residual:
\begin{equation}
J'_Q(\theta) = \mathbb{E}_{(o_t,a_t,o_{t+1}) \sim \mathcal{D}} 
\left[ \frac{1}{2} \left( Q_\theta(o_t,a_t) - 
\left( R_{\text{final}}(o_t) + \gamma \mathbb{E}_{a_{t+1} \sim \pi_\phi(\cdot \mid o_{t+1})}
\left[ Q_\theta(o_{t+1}, a_{t+1}) - \lambda \log \pi_\phi(a_{t+1} \mid o_{t+1}) \right] \right) \right)^2 \right]
\label{eq15}
\end{equation}

\subsubsection{Pipeline Architecture}

To populate the replay buffer $\mathcal{D}$ with these $R_{\text{final}}(o_t)$ values, the pipeline operates in three parallel processes:

\begin{enumerate}
\item \textbf{Interaction Thread.} The agent interacts with the environment, storing transitions $(o_t, a_t, \text{images}_t)$ in $\mathcal{D}$ with placeholder rewards $(r_t \leftarrow \texttt{NaN},\ \texttt{ready} = 0)$. This thread runs at the maximum possible environmental speed to rapidly collect raw experience.

\item \textbf{Reward Thread.} This thread runs in parallel with a separate worker. It continuously samples mini-batches $\{(o_i, a_i, o_{i+1}, \text{images}_i)\}_{i=1}^B$ from $\mathcal{D}$. For each transition in the batch, it executes the full hierarchical reward computation (Eqs.~\ref{eq2},\ref{eq6},\ref{eq8}--\ref{eq11}) and updates placeholder rewards with computed $R_{\text{final}}$ values.

\item \textbf{Learner Thread.} The SAC learner preferentially samples reward-annotated transitions (i.e., those with $\texttt{ready} = 1$) from $\mathcal{D}$ and performs policy and $Q$-function updates using Eq.~(\ref{eq15}). To mitigate reward staleness, the learner only begins policy updates once at least $N_{\text{warmup}}$ transitions have been reward-annotated, ensuring Q-value estimates are predominantly trained on accurate reward signals.
\end{enumerate}

This design enables experience collection to proceed without waiting for LVLM inference, maintaining high training throughput while preserving reward quality.

\subsubsection{Inference Acceleration}

To further reduce LVLM inference latency within the Reward Thread, we serve Qwen3-VL as an OpenAI-compatible API endpoint using vLLM~\citep{kwon2023efficient}, an optimized LLM serving framework. The model is deployed in \texttt{bfloat16} precision with chunked prefill enabled and a maximum batch size of 24 concurrent sequences, decoupled from the training process on a dedicated GPU. This setup achieves approximately 1~Hz annotation throughput, sufficient to keep the reward annotation backlog bounded relative to the policy learning rate. The CLIP and YOLO components are executed in FP16 precision directly within the Reward Thread without additional serving overhead.

\subsubsection{Deployment}

Once training is complete, the entire VLM-based reward apparatus is discarded. During deployment, DriveVLM-RL executes only the learned policy network $\pi_\phi$. The detector $D$, CLIP encoders $(f_I,f_L)$, and the LVLM $F_{\text{LVLM}}$ are used \emph{only} for offline training-time reward synthesis and are not executed at test time, achieving the goal of leveraging foundation model reasoning without incurring any deployment latency. The complete training procedure is outlined in~\ref{appendices5}.

\section{Experiments and Results}
\label{Experiments and Results}
The experiments are structured to address the following research questions: \textbf{Q1:} How does DriveVLM-RL compare with state-of-the-art reward design methods in terms of safety, efficiency, and task completion? \textbf{Q2:} Can DriveVLM-RL learn safe driving behaviors without explicit collision penalties through semantic understanding alone? \textbf{Q3:} Can the learned policy generalize to unseen environments and traffic conditions?

\subsection{Experimental Setup}
\subsubsection{Simulation Environment}
Consistent with prior VLM-as-Reward works~\citep{wasif2025drivemind, huang2025vlm}, we adopt CARLA~\citep{dosovitskiy2017carla} as our primary simulation platform, which provides photorealistic rendering, accurate vehicle dynamics, and diverse urban environments essential for evaluating end-to-end driving policies.

\textit{1) Training Environment.} 
All models in our comparative study are trained exclusively in CARLA Town 2 to ensure fair comparison and isolate the effects of different reward design approaches. Town 2 is a compact European-style urban layout featuring residential districts, commercial zones, single-lane roads, and signalized intersections, as shown in Fig.~\ref{towns}. To further evaluate generalization, we additionally test all methods on Towns~1, 3, 4, and 5, which represent fully out-of-distribution environments with distinct road layouts and topologies.

\textit{2) Traffic Configuration.} 
Different from VLM-RL~\citep{huang2025vlm}, whose experimental setting contains only vehicle interactions without traffic lights or other types of road users, we construct a more complex and heterogeneous traffic environment, as shown in Fig.~\ref{bev}. The simulation includes 20 vehicles for natural traffic flow, 20 pedestrians (walking speeds 0.8--1.5 m/s) near crosswalks, 20 motorcycles with short following distances and frequent cut-in behaviors, and 20 bicycles traveling at low speeds requiring safe overtaking.

\textit{3) Navigation Routes.} 
We employ dynamic route assignment during both training and evaluation. At each episode reset, we randomly select two distinct spawn points from the 101 predefined locations in Town 2 and compute the shortest path using the A* algorithm. Episodes continue until the cumulative driving distance reaches 3000~m, allowing comprehensive evaluation across diverse navigation scenarios within a single episode.

\textit{4) Episode Termination.} 
Each training and evaluation episode continues until one of the following three termination conditions is satisfied: (i) a collision with static infrastructure, other vehicles, pedestrians, cyclists, or motorcyclists is detected; (ii) the agent becomes stuck, defined as maintaining a speed below 1~km/h for more than 90 consecutive seconds; or (iii) the lateral deviation from the lane center exceeds 3~meters, indicating loss of lane-keeping control or off-road driving.

\begin{figure}[t]
  \centerline{\includegraphics[width=0.993\textwidth]{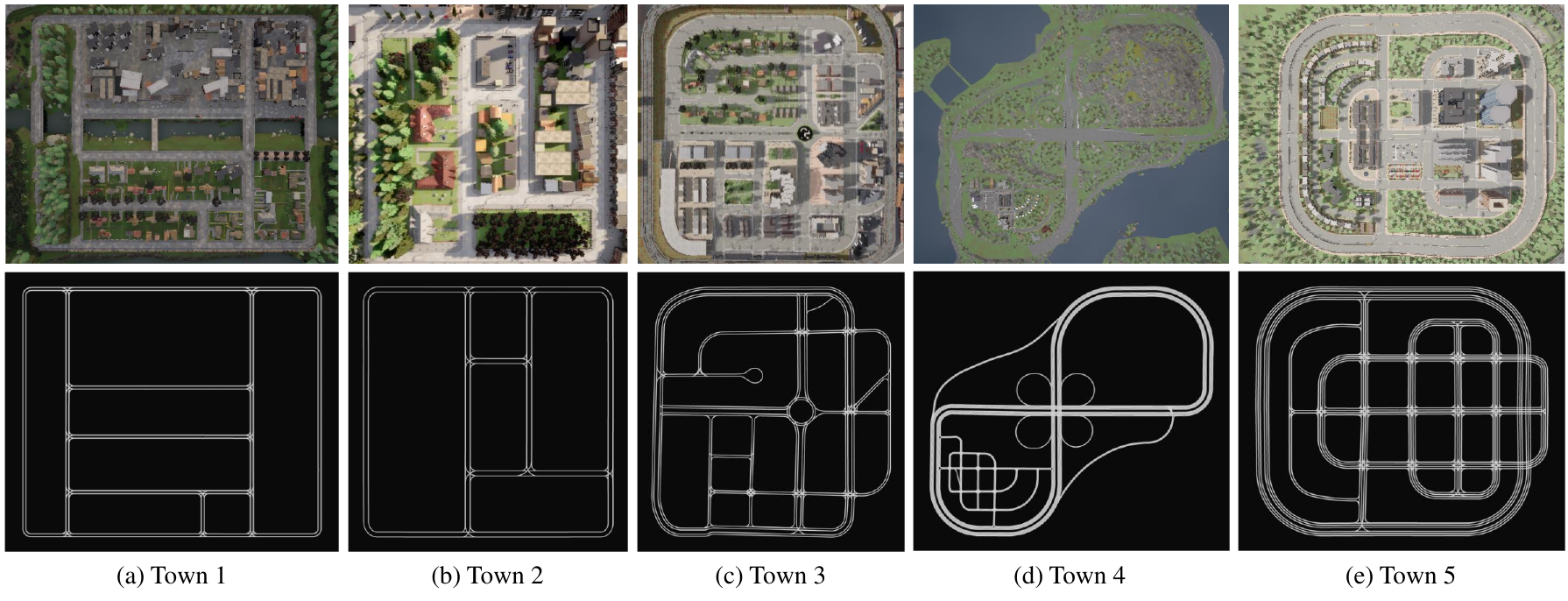}}
  \caption{CARLA towns used for training and evaluation, covering diverse urban layouts and road topologies. The top row shows aerial views of the environments, and the bottom row presents the corresponding lane network structures.}
  \label{towns}
\end{figure}

\begin{figure}[t]
  \centerline{\includegraphics[width=0.993\textwidth]{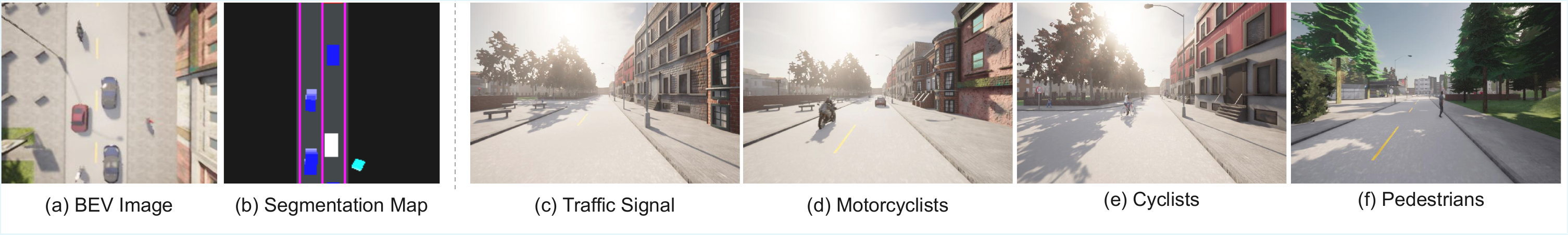}}
  \caption{Multi-modal observations of the ego vehicle in urban traffic, comprising BEV representation, semantic segmentation, and camera views with diverse traffic participants (signals, motorcyclists, cyclists, and pedestrians).}
  \label{bev}
\end{figure}

\subsubsection{Observation and Action Spaces}

\textit{1) Observation Space.} 
The RL agent receives: (i) a BEV semantic segmentation image rendered at a resolution of $224\times224$ pixels, generated by projecting CARLA ground-truth semantic labels onto a local coordinate frame centered on the ego vehicle; (ii) an ego-state vector capturing the vehicle's dynamic state, consisting of the current steering angle (normalized to $[-1,1]$), the throttle/brake command (normalized to $[-1,1]$), and the instantaneous vehicle speed in km/h; and (iii) a navigation context representation composed of the next 15 future waypoints along the planned route, expressed as $(x,y)$ coordinates in the ego-centric reference frame with the $x$-axis aligned with the vehicle's heading. Waypoints are sampled at 2~m intervals, providing approximately 30~m of route preview.

\textit{2) Action Space.} 
We employ a continuous two-dimensional action space $\mathcal{A}=[-1,1]^2$. The first dimension controls steering angle ($-1$: maximum left, $+1$: maximum right), while the second dimension combines throttle and brake control (positive values map to throttle intensity, negative values to brake intensity). This end-to-end perception-to-control formulation enables direct policy deployment without any auxiliary reasoning modules at inference time.

\subsubsection{Language Goal Configuration}
\label{sec:language_goal}
Following the CLG paradigm~\citep{huang2025vlm}, we adopt a simple yet effective contrastive prompt design for the static reward component. Specifically, we use ``the road is clear with no car accidents'' as the positive goal and ``two cars have collided with each other on the road'' as the negative goal. These prompts remain fixed across all experiments, highlighting the zero-shot nature of DriveVLM-RL framework and eliminating the need for task-specific prompt engineering. For the dynamic reward component, Qwen3-VL generates a context-specific risk description $l_t^{\text{dyn}}$ conditioned on the temporal window $\mathcal{W}_t$ and detected objects $\mathcal{O}_t$, following Definition~\ref{definition4}. To constrain the output space and ensure consistent CLIP-compatible semantic embeddings, we provide the LVLM with a reference vocabulary of 10 canonical scene descriptions covering common driving risk scenarios (e.g., ``a pedestrian is crossing the road ahead'', ``a cyclist is directly ahead on the road''). The LVLM may draw from or adapt these descriptions when generating $l_t^{\text{dyn}}$, while retaining the flexibility to produce novel descriptions for out-of-distribution scenarios. The complete reference vocabulary is provided in \ref{appendices6}.

\subsubsection{VLM Configuration}

We employ OpenCLIP's ViT-bigG-14 model \citep{ilharco2021openclip} pre-trained on the LAION-2B dataset with 2.32 billion English image--text pairs. The model uses a patch size of $14\times14$ pixels and accepts $224\times224$ pixel images. All CLIP components remain frozen during training to ensure stable semantic reward generation. We use YOLOv8-small as the lightweight detection model for the attention gate. While YOLO can detect all 80 COCO object classes, we selectively use 11 safety-critical classes $\{\text{person, bicycle, motorcycle, dog, horse, sheep, cow, elephant, bear, zebra, giraffe}\}$ to trigger VLM inference. For semantic reasoning, we employ Qwen3-VL-4B-Instruct~\citep{yang2025qwen3} as the LVLM, configured with a temporal window of $K=3$ frames. During training, the Reward Worker Thread processes stored transitions in mini-batches every $\Delta = 10$ control steps, invoking Qwen3-VL for reward annotation at an effective rate of approximately 1~Hz. YOLO pre-filtering via the attentional gate further reduces unnecessary Qwen3-VL calls by skipping transitions where no safety-critical objects are detected ($g_t = 0$).

\subsection{Evaluation Metrics}
We employ a set of quantitative metrics to evaluate both driving efficiency and safety performance:

\begin{itemize}
\item \textbf{Driving Efficiency Metrics.}
\emph{Average Speed} (\textbf{AS}) measures the mean vehicle speed over an episode. 
\emph{Total Distance} (\textbf{TD}) records the cumulative distance traveled by the ego vehicle.

\item \textbf{Safety Metrics.}
The \emph{Collision Rate} (\textbf{CR}) measures the percentage of episodes in which a collision with other vehicles or obstacles occurs, including rear-end and side collisions. 
To further characterize collision frequency, we report \emph{Time-based Collision Frequency} (\textbf{TCF}), defined as the number of collisions per 1000 time steps, and \emph{Distance-based Collision Frequency} (\textbf{DCF}), defined as the number of collisions per kilometer traveled. 
To quantify collision severity, we record the \emph{Collision Speed} (\textbf{CS}) at the moment of impact. 
We additionally compute the \emph{Inter-Collision Time} (\textbf{ICT}), defined as the average number of time steps between consecutive collisions, which reflects the temporal distribution of safety-critical events.

\item \textbf{Task Success.}
\emph{Route Completion} (\textbf{RC}) is defined as the number of successfully completed navigation routes within a single episode.
During the test phase, we additionally report the \emph{Success Rate} (\textbf{SR}), defined as the fraction of trials in which the agent successfully reaches the destination across 10 predefined evaluation routes.  
\emph{Average Collision} (\textbf{AC}) represents the average number of collisions per episode.
\end{itemize}

Detailed metric definitions follow~\citep{huang2025vlm}.

\subsection{Baseline Methods}
We compare against 11 representative methods spanning three reward design paradigms. To enable a comprehensive comparison across learning paradigms, we evaluate both SAC-based and PPO-based variants of the baselines. As demonstrated in our prior work VLM-RL~\citep{huang2025vlm}, the VLM-as-Reward paradigm transfers readily to PPO-based algorithms, supporting a broader evaluation without algorithmic re-design.
All baselines use identical network architectures, observation/action spaces, and training hyperparameters to ensure fair comparison.

\textit{1) Expert-Designed Reward Methods.}
We implement the following baselines with manually crafted reward functions:
\begin{itemize}
\item \textbf{TIRL-SAC}~\citep{cao2022trustworthy}: Binary reward with $-1$ for collision and $0$ otherwise, representing minimal reward informativeness.

\item \textbf{Chen-SAC}~\citep{chen2021interpretable}: Hand-tuned weighted reward balancing collision penalty, speed incentive, lane centering, and steering smoothness.

\item \textbf{ASAP-PPO}~\citep{wang2023efficient}: Skill-based reward providing positive incentives for route progress, destination arrival, and overtaking, with penalties for collisions and boundary violations.

\item \textbf{ChatScene-PPO}~\citep{zhang2024chatscene}: Smoothness-focused reward penalizing longitudinal acceleration, lateral acceleration, and abrupt steering changes, with a constant baseline signal to stabilize learning.
\end{itemize}

\textit{2) LLM-Designed Reward Methods.}
We compare against recent approaches that leverage LLMs for automated reward generation:
\begin{itemize}
\item \textbf{Revolve / Revolve-Auto}~\citep{hazra2024revolve}: An evolutionary framework using LLMs to generate reward function code guided by human feedback. We adopt their best-performing reward function for comparison.
\end{itemize}

\textit{3) VLM-Designed Reward Methods.}
We compare against five VLM-based reward shaping approaches:
\begin{itemize}
\item \textbf{VLM-SR}~\citep{baumli2023vision}: Binary reward using CLIP similarity thresholding to determine goal achievement.

\item \textbf{RoboCLIP}~\citep{sontakke2023roboclip}: Episodic reward computing average CLIP similarity between trajectory frames and a task descriptor.

\item \textbf{VLM-RM}~\citep{rocamonde2023vision}: Generates continuous reward by projecting the current state embedding onto the direction vector between baseline and target state descriptions.

\item \textbf{LORD}~\citep{ye2025lord}: Penalizes similarity to dangerous states using negative language goals.

\item \textbf{VLM-RL}~\citep{huang2025vlm}: A contrasting language goal formulation encouraging similarity to safe states while penalizing similarity to dangerous states.
\end{itemize}

\subsection{Implementation Details}
Our implementation is built upon the Stable-Baselines3 library \citep{raffin2021stable}, which provides reliable implementations of modern RL algorithms. The standard implementations of SAC and PPO are extended to incorporate our dual-pathway reward computation architecture during the training process. The policy network accommodates heterogeneous input modalities: a 6-layer convolutional neural network (CNN) extracts visual features from the BEV semantic segmentation images, while separate multi-layer perceptrons (MLPs) process the ego-state variables and future navigation waypoints. The resulting feature embeddings are concatenated and passed to a shared policy head to generate the final control actions. All experiments are conducted on a workstation equipped with three NVIDIA RTX A6000 GPUs (each with 48~GB memory, 10,752 CUDA cores), an AMD Ryzen Threadripper Pro 7985WX processor (64 cores, 128 threads), and 512~GB system memory. 

\begin{table*}[!t]
\begin{small}
\caption{Performance comparison with baselines during
training. Mean and standard deviation over 3 seeds.
The best results are marked in \textbf{bold}.
Note that the training-time CS is logged at every
control step (and is zero on non-collision steps) and
averaged over the full logging window, whereas the
testing CS in Table~\ref{tab:carla_test} is averaged
only over actual collision events; the two CS columns
are thus computed over different populations and are
not directly comparable.}
\label{tab:carla_train}
\centering
\renewcommand{\arraystretch}{1.4}
\setlength{\tabcolsep}{4pt}
\begin{adjustbox}{center}
\begin{tabular}{@{}lccccccccc@{}}
\toprule
Model & Reference & AS~$\uparrow$ & RC~$\uparrow$ & 
TD~$\uparrow$ & CS~$\downarrow$ & CR~$\downarrow$ & 
ICT~$\uparrow$ & DCF~$\downarrow$ & TCF~$\downarrow$ \\
\midrule

\rowcolor{gray!15}
\multicolumn{10}{l}{\textit{\textbf{Expert-designed 
Reward Methods (Binary Rewards)}}} \\
TIRL-SAC & TR-C'22 
  & 15.20 {\tiny $\pm$ 5.11} 
  & 0.26 {\tiny $\pm$ 0.25} 
  & 14.18 {\tiny $\pm$ 6.51} 
  & 4.02 {\tiny $\pm$ 5.48} 
  & 0.083 {\tiny $\pm$ 0.07} 
  & 36935 {\tiny $\pm$ 58721} 
  & 130.77 {\tiny $\pm$ 28.73} 
  & 3.028 {\tiny $\pm$ 2.15} \\
\midrule

\rowcolor{gray!15}
\multicolumn{10}{l}{\textit{\textbf{Expert-designed 
Reward Methods (Summation Rewards)}}} \\
Chen-SAC & T-ITS'22 
  & \textbf{25.06} {\tiny $\pm$ 0.20} 
  & 2.28 {\tiny $\pm$ 0.76} 
  & 560.92 {\tiny $\pm$ 265.32} 
  & 3.02 {\tiny $\pm$ 1.25} 
  & 0.293 {\tiny $\pm$ 0.05} 
  & 2071 {\tiny $\pm$ 593} 
  & 2.71 {\tiny $\pm$ 0.75} 
  & 1.833 {\tiny $\pm$ 0.75} \\
ASAP-PPO & RSS'23 
  & 18.66 {\tiny $\pm$ 3.37} 
  & 0.75 {\tiny $\pm$ 0.39} 
  & 67.17 {\tiny $\pm$ 46.87} 
  & 0.04 {\tiny $\pm$ 0.05} 
  & 0.403 {\tiny $\pm$ 0.16} 
  & 21398 {\tiny $\pm$ 14202} 
  & 34.71 {\tiny $\pm$ 25.75} 
  & 0.426 {\tiny $\pm$ 0.25} \\
ChatScene-PPO & CVPR'24 
  & 15.36 {\tiny $\pm$ 0.28} 
  & 2.42 {\tiny $\pm$ 0.29} 
  & 672.3 {\tiny $\pm$ 249.09} 
  & 0.05 {\tiny $\pm$ 0.04} 
  & 0.817 {\tiny $\pm$ 0.09} 
  & 2186 {\tiny $\pm$ 808} 
  & 1.66 {\tiny $\pm$ 0.72} 
  & 0.597 {\tiny $\pm$ 0.22} \\
\midrule

\rowcolor{gray!15}
\multicolumn{10}{l}{\textit{\textbf{LLM-based 
Reward Methods}}} \\
Revolve & ICLR'25 
  & 18.46 {\tiny $\pm$ 0.71} 
  & 2.78 {\tiny $\pm$ 0.63} 
  & 910.12 {\tiny $\pm$ 283.56} 
  & 0.53 {\tiny $\pm$ 0.56} 
  & 0.767 {\tiny $\pm$ 0.13} 
  & 2493 {\tiny $\pm$ 730} 
  & 1.18 {\tiny $\pm$ 0.39} 
  & 0.493 {\tiny $\pm$ 0.10} \\
Revolve-auto & ICLR'25 
  & 17.92 {\tiny $\pm$ 2.06} 
  & 3.65 {\tiny $\pm$ 0.44} 
  & 1283.8 {\tiny $\pm$ 313.86} 
  & 0.50 {\tiny $\pm$ 0.25} 
  & 0.930 {\tiny $\pm$ 0.03} 
  & 4702 {\tiny $\pm$ 2254} 
  & \textbf{0.81} {\tiny $\pm$ 0.19} 
  & 0.281 {\tiny $\pm$ 0.17} \\
\midrule

\rowcolor{gray!15}
\multicolumn{10}{l}{\textit{\textbf{VLM-based Reward 
Methods (Robotic)}}} \\
VLM-SR & NeurIPS'23 
  & 1.49 {\tiny $\pm$ 1.54} 
  & 0.51 {\tiny $\pm$ 0.36} 
  & 53.44 {\tiny $\pm$ 61.79} 
  & 0.02 {\tiny $\pm$ 0.02} 
  & \textbf{0.024} {\tiny $\pm$ 0.04} 
  & \textbf{869244} {\tiny $\pm$ 35550} 
  & 42.73 {\tiny $\pm$ 30.12} 
  & \textbf{0.192} {\tiny $\pm$ 0.14} \\
RoboCLIP & NeurIPS'23 
  & 11.05 {\tiny $\pm$ 5.74} 
  & 0.73 {\tiny $\pm$ 0.51} 
  & 130.17 {\tiny $\pm$ 67.78} 
  & 0.008 {\tiny $\pm$ 0.01} 
  & 0.097 {\tiny $\pm$ 0.10} 
  & 290455 {\tiny $\pm$ 438767} 
  & 47.17 {\tiny $\pm$ 46.46} 
  & 0.391 {\tiny $\pm$ 0.21} \\
VLM-RM & ICLR'24 
  & 10.86 {\tiny $\pm$ 4.57} 
  & 0.66 {\tiny $\pm$ 0.24} 
  & 61.81 {\tiny $\pm$ 48.32} 
  & 0.14 {\tiny $\pm$ 0.19} 
  & 0.067 {\tiny $\pm$ 0.06} 
  & 101193 {\tiny $\pm$ 62897} 
  & 35.26 {\tiny $\pm$ 24.13} 
  & 0.211 {\tiny $\pm$ 0.05} \\
\midrule

\rowcolor{gray!15}
\multicolumn{10}{l}{\textit{\textbf{VLM-based Reward 
Methods (Autonomous Driving)}}} \\
LORD & WACV'25 
  & 15.77 {\tiny $\pm$ 8.13} 
  & 0.72 {\tiny $\pm$ 0.32} 
  & 111.10 {\tiny $\pm$ 130.47} 
  & 0.56 {\tiny $\pm$ 0.91} 
  & 0.063 {\tiny $\pm$ 0.03} 
  & 45880 {\tiny $\pm$ 29377} 
  & 60.92 {\tiny $\pm$ 74.85} 
  & 0.398 {\tiny $\pm$ 0.21} \\
VLM-RL & TR-C'25 
  & 22.53 {\tiny $\pm$ 0.57} 
  & 2.77 {\tiny $\pm$ 0.33} 
  & 806.97 {\tiny $\pm$ 190.56} 
  & 0.008 {\tiny $\pm$ 0.01} 
  & 0.407 {\tiny $\pm$ 0.01} 
  & 5017 {\tiny $\pm$ 1297} 
  & 1.50 {\tiny $\pm$ 0.48} 
  & 0.404 {\tiny $\pm$ 0.05} \\

\rowcolor{green!10}
DriveVLM-RL & Ours 
  & 23.08 {\tiny $\pm$ 0.98} 
  & \textbf{3.97} {\tiny $\pm$ 1.41} 
  & \textbf{1347.0} {\tiny $\pm$ 465.31} 
  & \textbf{0.004} {\tiny $\pm$ 0.01} 
  & 0.190 {\tiny $\pm$ 0.03} 
  & 19043 {\tiny $\pm$ 12628} 
  & 1.61 {\tiny $\pm$ 0.34} 
  & 0.357 {\tiny $\pm$ 0.03} \\
\bottomrule
\end{tabular}
\end{adjustbox}
\end{small}
\end{table*}

\begin{figure}[!t]
  \centerline{\includegraphics[width=0.993\textwidth]
  {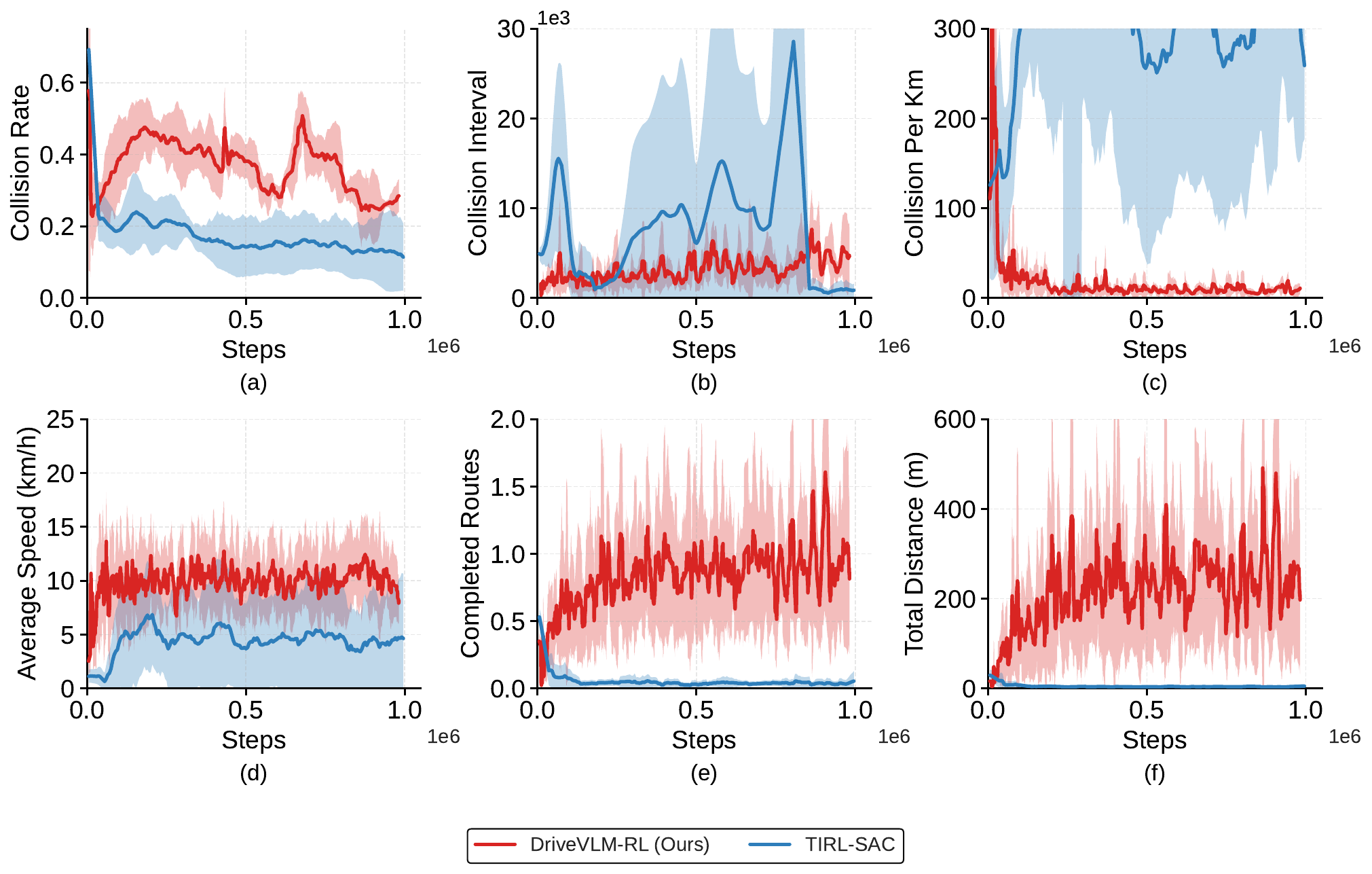}}
  \caption{Training curves for expert-designed 
binary reward baselines. 
(a)~Collision rate; (b)~Collision interval; 
(c)~Collision per km; (d)~Average speed; 
(e)~Completed routes; (f)~Total distance. 
DriveVLM-RL (red) progressively improves 
navigation capability with route completion 
reaching 1.5--2.0 and total distance exceeding 
400~m, while TIRL-SAC (blue) converges to a 
near-stationary policy with negligible route 
completion and extremely high collision-per-km 
despite a similar collision rate, confirming 
the sparse penalty failure mode.}
  \label{[main]-expert-binary}
\end{figure}

\subsection{Main Results}
Table~\ref{tab:carla_train} reports the mean and 
standard deviation of key metrics at the final 
training checkpoint, averaged over three independent 
runs with different random seeds to ensure robustness 
and reliability. Unlike VLM-RL~\citep{huang2025vlm}, 
which evaluates only vehicle interactions, we construct 
a more challenging and realistic environment that 
additionally includes pedestrians, motorcycles, and 
bicycles. To ensure fair comparison under this setting, 
all baseline models are retrained from scratch in our 
environment. The results are organized by reward 
design category to facilitate analysis of different 
paradigm strengths and limitations. We also report 
both training dynamics 
(Figs.~\ref{[main]-expert-binary}--\ref{[main]-vlm-driving}) 
and testing performance 
(Table~\ref{tab:carla_test}) to provide a complete 
picture of learned policy quality.

\subsubsection{Training Performance Analysis}

\textit{1) Comparison with expert-designed binary rewards.}
As shown in Fig.~\ref{[main]-expert-binary} and Table~\ref{tab:carla_train}, TIRL-SAC achieves a relatively low collision rate (CR~$= 0.083$) during training, which may appear favorable at first glance. However, the corresponding efficiency metrics reveal that the agent barely moves: route completion is only 0.26 and total distance is merely 14.18~m per episode. This reflects a well-known failure mode of sparse binary penalties, where the agent learns to minimize collision risk by suppressing forward movement entirely rather than developing genuine driving competence. In contrast, DriveVLM-RL progressively increases both average speed and route completion throughout training while maintaining a substantially lower collision rate than navigation-capable baselines (Fig.~\ref{[main]-expert-binary}(d)--(f)), demonstrating that the learned policy achieves purposeful navigation rather than passive risk avoidance.

\begin{figure}[!t]
  \centerline{\includegraphics[width=0.993\textwidth]
  {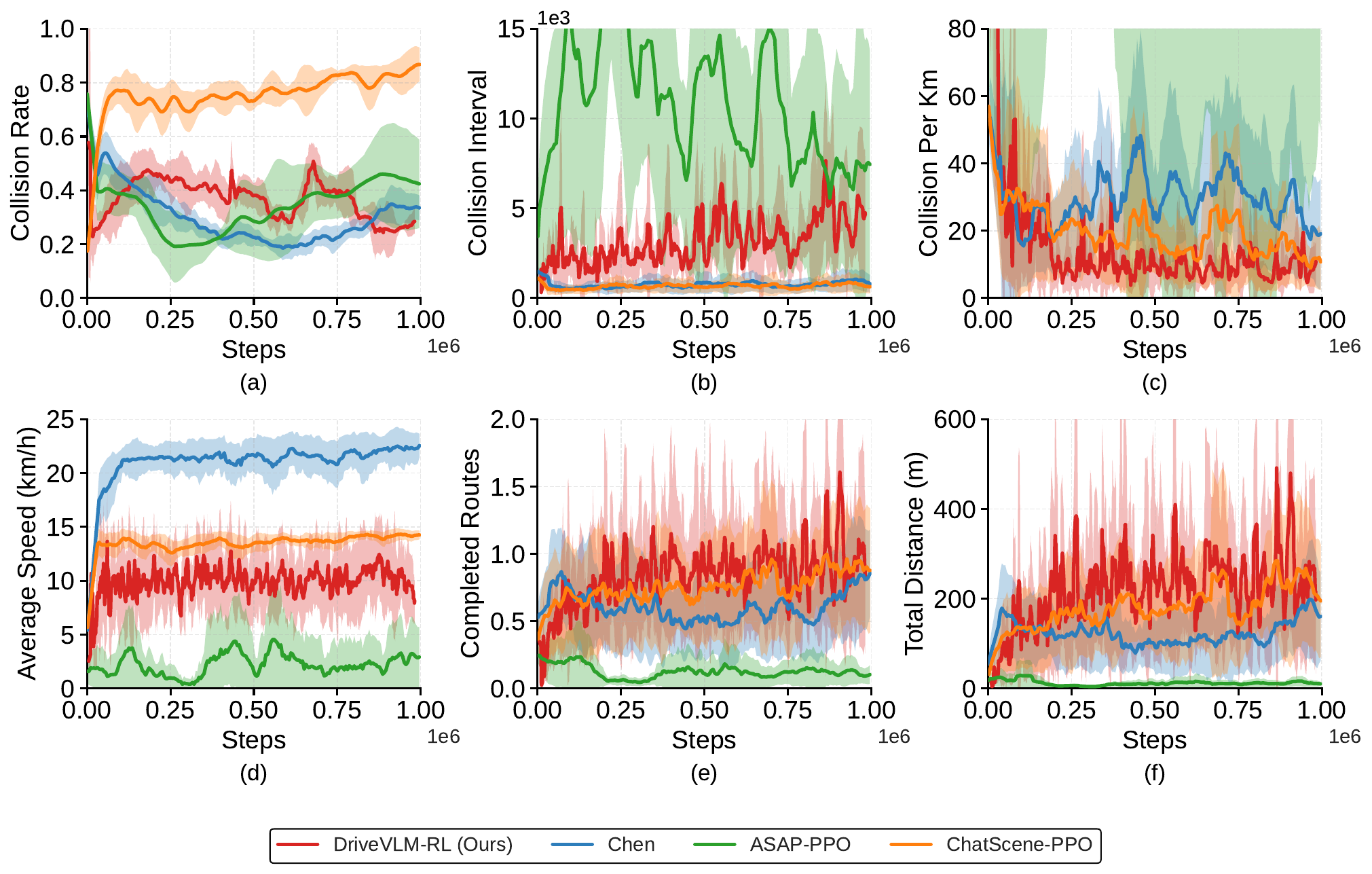}}
  \caption{Training curves for expert-designed 
summation reward baselines. 
(a)~Collision rate; (b)~Collision interval; 
(c)~Collision per km; (d)~Average speed; 
(e)~Completed routes; (f)~Total distance. 
DriveVLM-RL (red) progressively reduces its 
collision rate to 0.15--0.20 while maintaining 
competitive route completion, whereas 
ChatScene-PPO (orange) and Chen (blue) sustain 
high collision rates throughout training despite 
achieving high speeds, and ASAP-PPO (green) nearly 
stalls with negligible forward progress.}
  \label{[main]-expert}
\end{figure}

\textit{2) Comparison with summation-reward methods.}
As illustrated in Fig.~\ref{[main]-expert} and Table~\ref{tab:carla_train}, Chen-SAC employs a weighted reward combining speed incentives, collision penalties, lane centering, and steering smoothness. This design achieves the highest average speed among all baselines (AS~$= 25.06$~km/h), yet at a significant safety cost: a collision rate of 0.293 with a DCF of 2.71 collisions per kilometer. As shown in Figs.~\ref{[main]-expert}(a)--(c), Chen-SAC maintains consistently high collision rates throughout the entire training process, indicating that its reward function prioritizes driving efficiency over safety.

ASAP-PPO adopts a more conservative design with explicit penalties for boundary violations. It achieves an extremely low collision speed (CS~$= 0.04$~km/h), yet its collision rate remains high at 0.403, indicating that the agent frequently contacts surrounding objects at very low speed. This creeping behavior severely limits navigation performance, yielding only 0.75 route completions and 67.17~m traveled per episode.

ChatScene-PPO penalizes abrupt control actions and achieves competitive navigation (RC~$= 2.42$, TD~$= 672.3$~m). However, its collision rate of 0.817, DCF of 1.66, and TCF of 0.597 reveal a fundamentally unsafe driving profile. Fig.~\ref{[main]-expert} shows that ChatScene-PPO rapidly learns to navigate routes but sustains high collision rates throughout training, without developing any capacity for proactive risk anticipation.

DriveVLM-RL achieves a substantially more balanced profile. With 3.97 route completions and 1347.0~m traveled per episode, it leads all summation-reward baselines in navigation performance while simultaneously reducing collision rate by 77\% relative to ChatScene-PPO (from 0.817 to 0.190) and by 35\% relative to Chen-SAC (from 0.293 to 0.190). The collision speed of 0.004~km/h and ICT of 19043 steps further confirm that DriveVLM-RL encounters far fewer and less severe safety-critical events throughout training.

\textit{3) Comparison with LLM-designed methods.}
As shown in Fig.~\ref{[main]-llm} and Table~\ref{tab:carla_train}, both Revolve and Revolve-auto leverage LLMs to generate reward function code through evolutionary search with human feedback. These methods demonstrate strong navigation capability, with Revolve completing 2.78 routes at 18.46~km/h and Revolve-auto completing 3.65 routes at 17.92~km/h. However, both methods exhibit persistently high collision rates throughout training (0.767 and 0.930 respectively), with collision speeds of 0.53 and 0.50~km/h, indicating that LLM-generated reward code cannot reliably encode safe driving behavior in heterogeneous traffic. As visible in Fig.~\ref{[main]-llm}(a), collision rates for both methods plateau near 1.0 and show no meaningful reduction across the full training horizon.

DriveVLM-RL achieves a 75\% reduction in collision rate compared to Revolve (from 0.767 to 0.190) and a 80\% reduction compared to Revolve-auto (from 0.930 to 0.190), while maintaining competitive or superior route completion. This advantage stems from a fundamental difference in reward design: LLM-based approaches translate high-level safety specifications into static code functions that cannot respond to evolving visual context, whereas DriveVLM-RL directly grounds safety assessment in multi-frame visual observations through LVLM reasoning.

\begin{figure}[!t]
  \centerline{\includegraphics[width=0.993\textwidth]
  {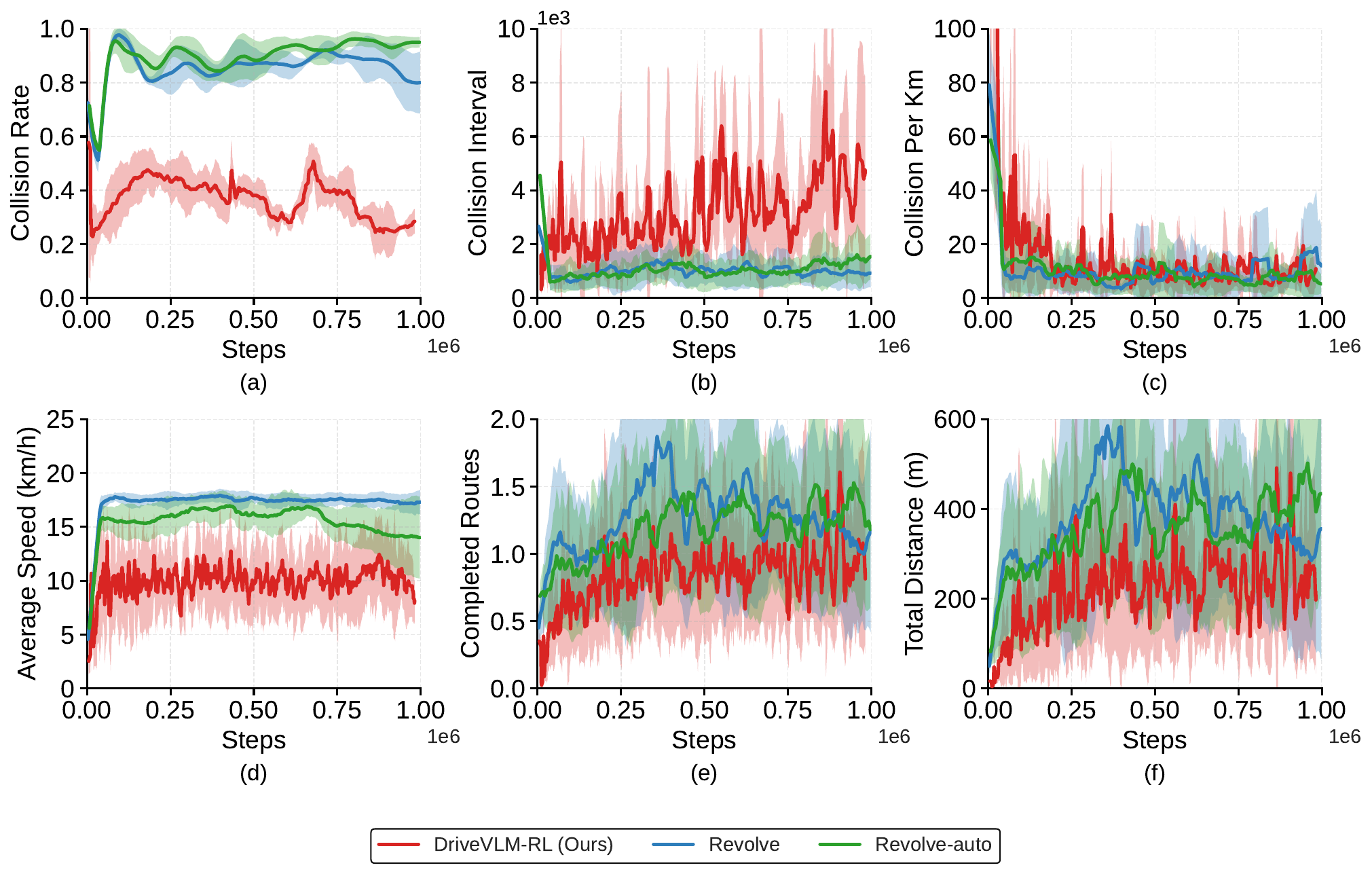}}
  \caption{Training curves for LLM-designed reward 
baselines. 
(a)~Collision rate; (b)~Collision interval; 
(c)~Collision per km; (d)~Average speed; 
(e)~Completed routes; (f)~Total distance. 
DriveVLM-RL (red) maintains a consistently lower 
collision rate (0.15--0.25) than Revolve (blue) 
and Revolve-auto (green), which both plateau near 
1.0 throughout training, while achieving 
comparable route completion and distance despite 
the substantially safer driving profile.}
  \label{[main]-llm}
\end{figure}

\textit{4) Comparison with VLM-Based Robotic methods.}
As shown in Fig.~\ref{[main]-vlm-rob} and Table~\ref{tab:carla_train}, VLM-SR, RoboCLIP, and VLM-RM all exhibit severe performance degradation when transferred to the autonomous driving setting. VLM-SR achieves an average speed of only 1.49~km/h, completing 0.51 routes and covering 53.44~m per episode. RoboCLIP and VLM-RM show similarly poor navigation with speeds of 11.05 and 10.86~km/h and route completions below 1.0. As illustrated in Figs.~\ref{[main]-vlm-rob}(d)--(f), all three methods fail to develop meaningful forward progress throughout training, with speed and distance metrics remaining near zero across most of the training horizon.

This failure reflects a fundamental domain gap between robotic manipulation and autonomous driving. In manipulation tasks, visual goals such as ``grasp the red cube'' can be precisely specified and measured through single-frame image similarity. In contrast, driving objectives are inherently ambiguous and temporally dependent: the risk posed by a pedestrian near a crosswalk cannot be captured by a fixed language goal or a static image comparison. Autonomous driving requires reward signals that encode contextual evolution, motion intent, and multi-agent dynamics, which single-frame CLIP similarity scores are unable to provide.

\begin{figure}[!t]
  \centerline{\includegraphics[width=0.993\textwidth]
  {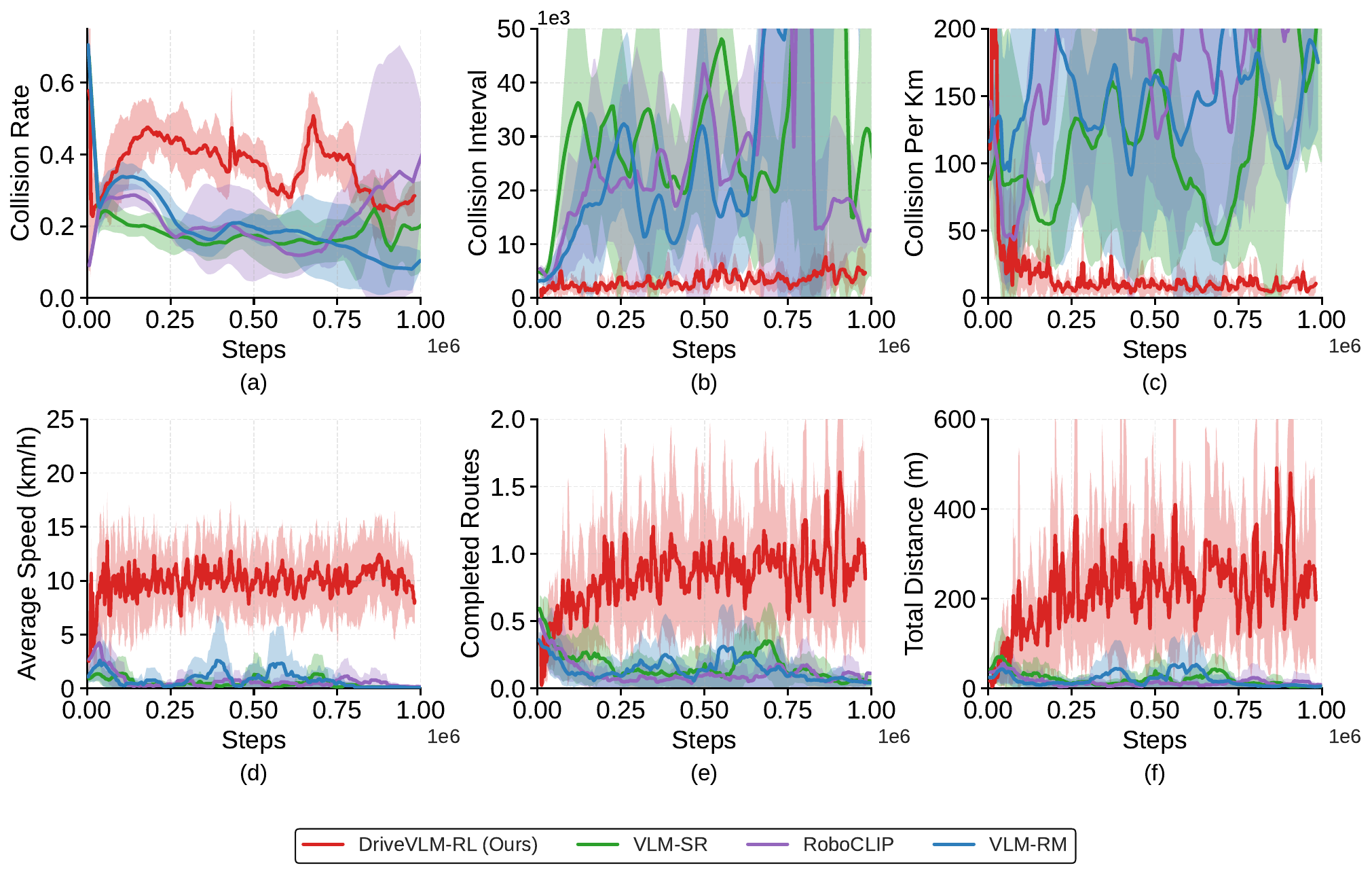}}
  \caption{Training curves for VLM-based robotic 
reward baselines. 
(a)~Collision rate; (b)~Collision interval; 
(c)~Collision per km; (d)~Average speed; 
(e)~Completed routes; (f)~Total distance. 
DriveVLM-RL (red) demonstrates substantially 
higher navigation capability across all metrics, 
while VLM-SR (green), RoboCLIP (purple), and 
VLM-RM (blue) fail to develop meaningful driving 
behaviors, remaining near-stationary throughout 
training despite low collision rates.}
  \label{[main]-vlm-rob}
\end{figure}

\begin{figure}[!t]
  \centerline{\includegraphics[width=0.993\textwidth]
  {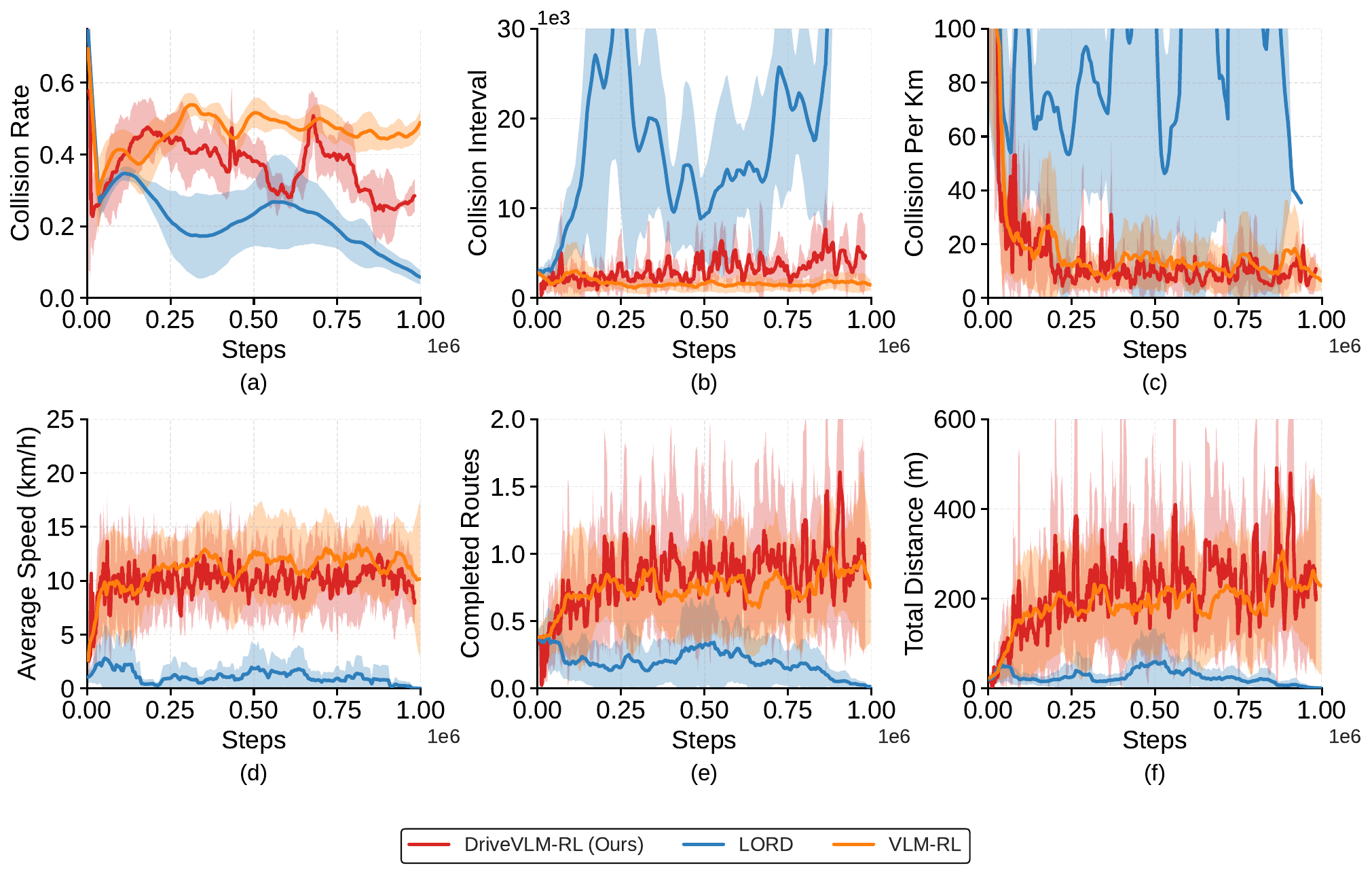}}
  \caption{Training curves for VLM-based autonomous 
driving reward baselines. 
(a)~Collision rate; (b)~Collision interval; 
(c)~Collision per km; (d)~Average speed; 
(e)~Completed routes; (f)~Total distance. 
DriveVLM-RL (red) achieves the lowest collision 
rate and highest navigation performance, while 
LORD (blue) nearly stalls with minimal forward 
progress despite low collision frequency.}
  \label{[main]-vlm-driving}
\end{figure}

\textit{5) Comparison with VLM-Based Driving Methods.}
As shown in Fig.~\ref{[main]-vlm-driving} and Table~\ref{tab:carla_train}, LORD and VLM-RL represent the current state of the art in VLM-as-Reward approaches for autonomous driving.

LORD uses negative language goals to describe dangerous states. Despite maintaining a low collision rate (CR~$= 0.063$), it achieves only 0.72 route completions and 111.10~m traveled per episode, indicating that the agent learns to avoid collisions by suppressing movement. This behavior mirrors the sparse-penalty failure mode observed in TIRL-SAC, and suggests that negative-only language goals are insufficient to provide the positive driving incentive needed for active navigation.

VLM-RL, our previous work using contrasting language goals with CLIP, achieves substantially better navigation performance with 2.77 route completions at 22.53~km/h and 806.97~m traveled. However, the collision rate of 0.407 reveals a key limitation of static CLG formulations: fixed language goals and single-frame BEV observations cannot capture the evolving, context-dependent safety cues present in heterogeneous traffic. Situations such as a pedestrian stepping off a curb or a motorcycle preparing to cut in require temporal reasoning that static CLIP similarity cannot provide.

DriveVLM-RL substantially improves upon both baselines. Compared to VLM-RL, our dual-pathway architecture achieves a 53\% reduction in collision rate (from 0.407 to 0.190), a 43\% increase in route completions (from 2.77 to 3.97), and a 67\% increase in distance traveled (from 806.97 to 1347.0~m). Compared to LORD, DriveVLM-RL achieves dramatically better navigation (3.97 versus 0.72 routes) while still maintaining a competitive safety profile, confirming that the Dynamic Pathway successfully decouples safety from mobility by grounding risk reasoning in multi-frame visual context rather than relying on conservative movement suppression.

\subsubsection{Performance Evaluation in Testing}

To assess the generalization of learned policies, we evaluate all methods on 10 predefined routes in Town~2 that were not encountered during training, as shown in Table~\ref{tab:carla_test}. Results are averaged over three random seeds, with SR and AC serving as the primary indicators of deployment readiness.

\textbf{Expert-designed methods.} TIRL-SAC, which already exhibited movement suppression during training, collapses entirely at test time: average speed drops to 0.45~km/h, route completion to 0.01, and SR to 0\%. Chen-SAC preserves its speed advantage (AS~$= 24.32$~km/h) and achieves a 50\% SR, but its collision speed of 16.04~km/h and AC of 0.50 indicate severe collision risk in unseen scenarios. ASAP-PPO achieves 0\% SR with the highest AC among all evaluated methods (0.67), revealing that its conservative training behavior does not transfer to meaningful driving performance at test time. ChatScene-PPO reaches a 40\% SR with TD of 127.85~m but maintains a high collision speed of 6.05~km/h and AC of 0.60, reflecting that smoothness-based reward shaping does not generalize to safety-critical generalization scenarios.

\textbf{LLM-based methods.} Both Revolve and Revolve-auto achieve a 40\% SR, demonstrating partial generalization of their LLM-generated reward functions. However, both methods exhibit very high collision speeds of 10.33 and 7.80~km/h respectively, confirming that reward functions generated from natural language code specifications do not provide sufficient visual grounding to produce safe behavior in previously unseen traffic conditions.

\textbf{VLM-based robotic methods.} VLM-SR, RoboCLIP, and VLM-RM all achieve 0\% SR with near-zero navigation capability, confirming that the domain gap identified during training analysis persists at test time. The extremely low movement of these methods, particularly VLM-RM (AS~$= 0.08$~km/h, AC~$= 0.00$), reflects passive safety achieved through complete immobility rather than genuine risk management.

\textbf{VLM-based driving methods.} LORD, despite its low training collision rate, collapses to 0\% SR at test time with only 4.10~m traveled per episode, indicating that its negative-goal formulation overfits to the training environment and fails to generalize to new routes. VLM-RL achieves a 40\% SR and 138.08~m total distance, but its collision speed of 10.09~km/h highlights the limitation of static CLIP-based rewards when confronting diverse unseen traffic interactions.

DriveVLM-RL achieves the highest SR of 57\% and the greatest total distance of 186.59~m among all evaluated methods. Compared to VLM-RL, this represents a 43\% improvement in SR (from 0.40 to 0.57), an 83\% reduction in collision speed (from 10.09 to 1.75~km/h), and a 35\% increase in total distance traveled (from 138.08 to 186.59~m). Compared to the expert-designed ChatScene-PPO baseline, DriveVLM-RL improves SR by 43\% (from 0.40 to 0.57) while reducing collision speed by 71\% (from 6.05 to 1.75~km/h). These results demonstrate that the dual-pathway architecture, by grounding safety assessment in temporally-aware visual semantics rather than fixed reward heuristics, generalizes more effectively to unseen driving conditions.

\begin{table}[!t]
\begin{small}
\caption{Performance comparison with baselines during testing. Mean and standard deviation over 3 seeds. The best results are marked in \textbf{bold}.}
\label{tab:carla_test}
\centering
\renewcommand{\arraystretch}{1.4}
\setlength{\tabcolsep}{4pt}
\begin{tabular*}{\columnwidth}{@{\extracolsep{\fill}}lccccccc@{}}
\toprule
Model & Reference & AS~$\uparrow$ & RC~$\uparrow$ & TD~$\uparrow$ & CS~$\downarrow$ & SR~$\uparrow$ & AC~$\downarrow$ \\
\midrule

\rowcolor{gray!15}
\multicolumn{8}{l}{\textit{\textbf{Expert-designed Reward Methods}}} \\
TIRL-SAC & TR-C'22 & 0.45 {\tiny $\pm$ 0.77} & 0.01 {\tiny $\pm$ 0.01} & 1.49 {\tiny $\pm$ 2.32} & 0.29 {\tiny $\pm$ 0.50} & 0.00 {\tiny $\pm$ 0.00} & 0.07 {\tiny $\pm$ 0.12} \\
Chen-SAC & T-ITS'22 & \textbf{24.32} {\tiny $\pm$ 0.46} & 0.49 {\tiny $\pm$ 0.08} & 162.01 {\tiny $\pm$ 17.67} & 16.04 {\tiny $\pm$ 2.51} & 0.50 {\tiny $\pm$ 0.10} & 0.50 {\tiny $\pm$ 0.10} \\
ASAP-PPO & RSS'23 & 11.53 {\tiny $\pm$ 10.22} & 0.12 {\tiny $\pm$ 0.11} & 25.00 {\tiny $\pm$ 24.92} & 7.07 {\tiny $\pm$ 5.96} & 0.00 {\tiny $\pm$ 0.00} & 0.67 {\tiny $\pm$ 0.32} \\
ChatScene-PPO & CVPR'24 & 14.78 {\tiny $\pm$ 0.30} & 0.44 {\tiny $\pm$ 0.14} & 127.85 {\tiny $\pm$ 10.39} & 6.05 {\tiny $\pm$ 1.28} & 0.40 {\tiny $\pm$ 0.10} & 0.60 {\tiny $\pm$ 0.10} \\
\midrule

\rowcolor{gray!15}
\multicolumn{8}{l}{\textit{\textbf{LLM-based Reward Methods}}} \\
Revolve & ICLR'25 & 17.42 {\tiny $\pm$ 0.80} & 0.40 {\tiny $\pm$ 0.12} & 134.37 {\tiny $\pm$ 15.26} & 10.33 {\tiny $\pm$ 2.25} & 0.40 {\tiny $\pm$ 0.20} & 0.60 {\tiny $\pm$ 0.20} \\
Revolve-auto & ICLR'25 & 14.12 {\tiny $\pm$ 3.07} & 0.33 {\tiny $\pm$ 0.12} & 129.14 {\tiny $\pm$ 33.22} & 7.80 {\tiny $\pm$ 1.06} & 0.40 {\tiny $\pm$ 0.20} & 0.60 {\tiny $\pm$ 0.20} \\
\midrule

\rowcolor{gray!15}
\multicolumn{8}{l}{\textit{\textbf{VLM-based Reward Methods (Robotic)}}} \\
VLM-SR & NeurIPS'23 & 0.06 {\tiny $\pm$ 0.05} & 0.01 {\tiny $\pm$ 0.00} & 2.26 {\tiny $\pm$ 1.26} & 0.66 {\tiny $\pm$ 1.14} & 0.00 {\tiny $\pm$ 0.00} & 0.07 {\tiny $\pm$ 0.12} \\
RoboCLIP & NeurIPS'23 & 0.13 {\tiny $\pm$ 0.09} & 0.02 {\tiny $\pm$ 0.01} & 3.46 {\tiny $\pm$ 2.32} & 0.01 {\tiny $\pm$ 0.02} & 0.00 {\tiny $\pm$ 0.00} & 0.03 {\tiny $\pm$ 0.06} \\
VLM-RM & ICLR'24 & 0.08 {\tiny $\pm$ 0.01} & 0.02 {\tiny $\pm$ 0.00} & 3.60 {\tiny $\pm$ 0.38} & \textbf{0.00} {\tiny $\pm$ 0.00} & 0.00 {\tiny $\pm$ 0.00} & \textbf{0.00} {\tiny $\pm$ 0.00} \\
\midrule

\rowcolor{gray!15}
\multicolumn{8}{l}{\textit{\textbf{VLM-based Reward Methods (Autonomous Driving)}}} \\
LORD & WACV'25 & 0.36 {\tiny $\pm$ 0.59} & 0.03 {\tiny $\pm$ 0.03} & 4.10 {\tiny $\pm$ 6.11} & 1.52 {\tiny $\pm$ 2.63} & 0.00 {\tiny $\pm$ 0.00} & 0.07 {\tiny $\pm$ 0.12} \\
VLM-RL & TR-C'25 & 14.38 {\tiny $\pm$ 1.53} & 0.51 {\tiny $\pm$ 0.08} & 138.08 {\tiny $\pm$ 16.68} & 10.09 {\tiny $\pm$ 5.93} & 0.40 {\tiny $\pm$ 0.00} & 0.10 {\tiny $\pm$ 0.10} \\

\rowcolor{green!10}
DriveVLM-RL & Ours & 14.54 {\tiny $\pm$ 1.81} & \textbf{0.57} {\tiny $\pm$ 0.03} & \textbf{186.59} {\tiny $\pm$ 14.00} & 1.75 {\tiny $\pm$ 3.02} & \textbf{0.57} {\tiny $\pm$ 0.15} & 0.20 {\tiny $\pm$ 0.26} \\
\bottomrule
\end{tabular*}
\end{small}
\end{table}

\subsection{No-Reward-After-Collision Experiment}

Traditional RL approaches depend on collision-based 
trial-and-error, creating an insurmountable barrier 
to real-world deployment where physical crashes are 
unacceptable. To investigate whether DriveVLM-RL 
can learn safe behaviors through semantic 
understanding alone, we conduct an extreme ablation 
study by completely removing the collision penalty 
term, modifying Eq.~(\ref{eq11}) to:
\begin{equation}
R_{\text{final}}(o_t) = R_{\text{shaping}}(o_t),
\quad R_{\text{penalty}} = 0
\label{eq:extreme}
\end{equation}

\subsubsection{Training Performance}

All three methods are retrained under the extreme 
behavior setting, where collisions with vulnerable 
road users (pedestrians, cyclists, and motorcycles) 
incur no penalty and do not terminate the episode, 
whereas vehicle collisions still terminate the 
episode but carry no penalty reward. To quantify 
behavior under this setting, we additionally define 
two extreme-collision metrics that are specific to 
this experiment. The \emph{Extreme Collision 
Number} (ExtCol) is the total number of 
collisions with vulnerable road users accumulated 
over the course of training, and the \emph{Extreme 
Collision Rate} (ExtCR) is the fraction of the most 
recent 100 episodes that contain at least one such 
collision, computed analogously to CR. As a result, 
CR reflects vehicle-collision 
avoidance, whereas ExtCol and ExtCR isolate the 
policy's ability to avoid vulnerable road users 
through learned semantics alone, since such 
collisions carry no penalty signal. We report the 
training dynamics in Fig.~\ref{[main]-extreme} and 
the final-checkpoint metrics in 
Table~\ref{tab:extreme_train}.

\begin{table*}[!t]
\begin{small}
\caption{Training performance under the 
no-reward-after-collision setting 
($R_{\text{penalty}} = 0$). Mean and standard 
deviation over three seeds. The best result in 
each column is in \textbf{bold}.}
\label{tab:extreme_train}
\centering
\renewcommand{\arraystretch}{1.4}
\setlength{\tabcolsep}{6pt}
\begin{adjustbox}{center}
\begin{tabular}{@{}lcccccccc@{}}
\toprule
Model & Reference & AS~$\uparrow$ & RC~$\uparrow$ & 
TD~$\uparrow$ & CR~$\downarrow$ & ICT~$\uparrow$ & 
ExtCR~$\downarrow$ & ExtCol~$\downarrow$ \\
\midrule
ChatScene-PPO & CVPR'24
  & 15.47 {\tiny $\pm$ 0.07}
  & \textbf{3.29} {\tiny $\pm$ 0.27}
  & \textbf{1390.1} {\tiny $\pm$ 429.3}
  & 0.813 {\tiny $\pm$ 0.100}
  & 4013 {\tiny $\pm$ 451}
  & 0.327 {\tiny $\pm$ 0.093}
  & 556.0 {\tiny $\pm$ 11.1} \\
VLM-RL & TR-C'25
  & \textbf{24.16} {\tiny $\pm$ 0.65}
  & 2.79 {\tiny $\pm$ 0.36}
  & 1024.4 {\tiny $\pm$ 218.4}
  & \textbf{0.170} {\tiny $\pm$ 0.056}
  & \textbf{12602} {\tiny $\pm$ 3196}
  & 0.337 {\tiny $\pm$ 0.047}
  & 607.7 {\tiny $\pm$ 34.5} \\
\rowcolor{green!10}
DriveVLM-RL & Ours
  & 22.89 {\tiny $\pm$ 1.69}
  & 2.72 {\tiny $\pm$ 0.17}
  & 865.0 {\tiny $\pm$ 162.3}
  & 0.230 {\tiny $\pm$ 0.053}
  & 9116 {\tiny $\pm$ 3812}
  & \textbf{0.273} {\tiny $\pm$ 0.032}
  & \textbf{476.3} {\tiny $\pm$ 175.5} \\
\bottomrule
\end{tabular}
\end{adjustbox}
\end{small}
\end{table*}

\begin{figure}[!t]
  \centerline{\includegraphics[width=0.993\textwidth]{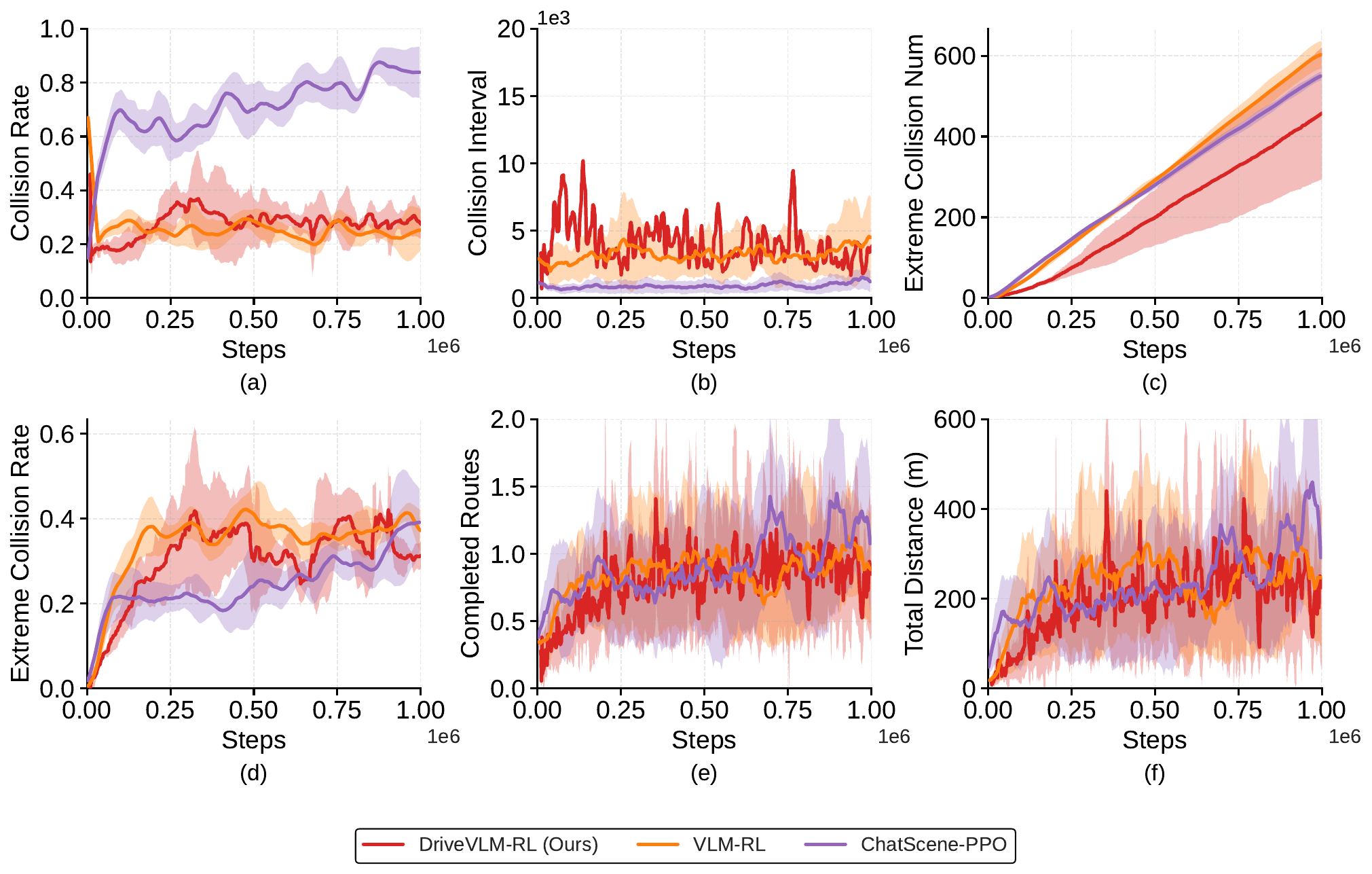}}
  \caption{Training curves under the 
  no-reward-after-collision setting 
  ($R_{\text{penalty}} = 0$). 
  (a)~Collision rate; 
  (b)~Collision interval; 
  (c)~Extreme collision number; 
  (d)~Extreme collision rate; 
  (e)~Completed routes; 
  (f)~Total distance. 
  Although VLM-RL and ChatScene-PPO achieve 
  comparable or higher navigation, DriveVLM-RL 
  (red) attains the lowest cumulative extreme 
  collision count (c) and the lowest extreme 
  collision rate (d), indicating the strongest 
  protection of vulnerable road users when no 
  explicit collision penalty is provided.}
  \label{[main]-extreme}
\end{figure}

\textit{ChatScene-PPO.} Its expert smoothness 
reward without any penalty produces aggressive 
driving. As shown in Fig.~\ref{[main]-extreme}(a), 
the vehicle collision rate rises above 0.8 within 
250k steps and remains persistently high 
(CR~$= 0.813$). This behavior yields the highest 
route completion (RC~$= 3.29$) and total distance 
(TD~$= 1390.1$~m), but at a catastrophic safety 
cost: it still incurs a large number of extreme 
collisions (ExtCol~$= 556.0$, 
Fig.~\ref{[main]-extreme}(c)), confirming that the 
policy maximizes forward progress while largely 
disregarding both vehicles and vulnerable road 
users once the penalty is removed.

\textit{VLM-RL.} The CLIP contrastive reward yields 
the strongest vehicle-collision avoidance 
(CR~$= 0.170$) and the longest collision interval 
(ICT~$= 12602$ steps), together with high average 
speed (AS~$= 24.16$~km/h). However, its fixed 
language goal is defined around car-accident 
semantics and does not generalize to vulnerable 
road users: VLM-RL accumulates the most extreme 
collisions (ExtCol~$= 607.7$, 
Fig.~\ref{[main]-extreme}(c)) and a higher extreme 
collision rate (ExtCR~$= 0.337$). This exposes the 
limitation of static single-frame language goals, 
which cannot capture context-dependent risks such 
as pedestrians stepping onto the road or 
motorcycles initiating cut-ins.

\textit{DriveVLM-RL.} DriveVLM-RL attains the 
lowest extreme collision rate (ExtCR~$= 0.273$) and 
the lowest extreme collision number 
(ExtCol~$= 476.3$), which is clearly the lowest 
trajectory in Fig.~\ref{[main]-extreme}(c). In 
other words, it provides the strongest protection 
of vulnerable road users when no explicit penalty 
is available, while remaining competitive in 
navigation (AS~$= 22.89$~km/h, RC~$= 2.72$) and in 
vehicle-collision safety (CR~$= 0.230$). Because the 
Dynamic Pathway evaluates temporal, 
context-dependent risk during training, the policy 
proactively anticipates approaching pedestrians and 
cut-in motorcycles even in the complete absence of 
penalty signals.

Because the extreme-collision metrics isolate 
penalty-free semantic avoidance, they are precisely 
where DriveVLM-RL holds a decisive advantage. This 
validates that $R_{\text{shaping}}$ alone provides 
sufficient proactive guidance for the agent to 
internalize a predictive safety criterion rather 
than merely reacting to collision outcomes.

\subsubsection{Testing Performance}

We further evaluate the policies trained under the 
no-reward-after-collision setting on the same test 
routes used in the main experiments. Following the 
main testing protocol, we compare against 
ChatScene-PPO and VLM-RL. Table~\ref{tab:test} 
reports the results.

\begin{table}[!t]
\caption{No-reward-after-collision testing 
performance. Mean and standard deviation over 
three random seeds. The best results in each 
column are highlighted in \textbf{bold}.}
\label{tab:test}
\centering
\renewcommand{\arraystretch}{1.3}
\setlength{\tabcolsep}{4pt}

\begin{tabular*}{\columnwidth}{@{\extracolsep{\fill}}lccccccc@{}}
\toprule
Model & Reference & AS$\uparrow$ & RC$\uparrow$ & TD$\uparrow$ & CS$\downarrow$ & SR$\uparrow$ & AC$\downarrow$ \\
\midrule

ChatScene-PPO & CVPR'24
& \textbf{15.20} {\tiny $\pm$ 0.39}
& 0.38 {\tiny $\pm$ 0.05}
& 137.51 {\tiny $\pm$ 12.42}
& 3.91 {\tiny $\pm$ 1.75}
& 0.47 {\tiny $\pm$ 0.12}
& 0.53 {\tiny $\pm$ 0.12} \\

VLM-RL & TR-C'25
& 14.52 {\tiny $\pm$ 0.56}
& \textbf{0.49} {\tiny $\pm$ 0.09}
& 136.33 {\tiny $\pm$ 31.86}
& 4.59 {\tiny $\pm$ 4.09}
& 0.40 {\tiny $\pm$ 0.10}
& \textbf{0.13} {\tiny $\pm$ 0.06} \\

\rowcolor{green!10}
DriveVLM-RL & Ours
& 15.17 {\tiny $\pm$ 1.85}
& 0.44 {\tiny $\pm$ 0.03}
& \textbf{149.69} {\tiny $\pm$ 34.85}
& \textbf{0.69} {\tiny $\pm$ 1.09}
& \textbf{0.50} {\tiny $\pm$ 0.10}
& 0.20 {\tiny $\pm$ 0.20} \\

\bottomrule
\end{tabular*}
\end{table}

At test time, DriveVLM-RL achieves the highest 
success rate (SR~$= 0.50$), the longest travel 
distance (TD~$= 149.69$~m), and by far the lowest 
collision severity (CS~$= 0.69$~km/h). The two 
baselines exhibit complementary failure modes. 
VLM-RL records the fewest collisions per episode 
(AC~$= 0.13$), yet when contact does occur it 
happens at nearly seven times higher speed 
(CS~$= 4.59$~km/h), so its collisions are rare but 
violent. ChatScene-PPO instead incurs both the most 
collisions per episode (AC~$= 0.53$) and high-speed 
impacts (CS~$= 3.91$~km/h). DriveVLM-RL collides 
slightly more often than VLM-RL (AC~$= 0.20$), but 
its impacts occur at near-stationary speed, 
indicating that the policy has already decelerated 
before contact. This proactive braking, rather than 
mere collision avoidance, is what yields the highest 
success rate among all methods, confirming that the 
semantic safety learned by DriveVLM-RL during 
training transfers to unseen test routes.

\subsection{Generalization to Unseen Environments}
\subsubsection{Cross-Town Generalization}
To evaluate generalization beyond the training 
environment, we test all methods across five 
distinct CARLA urban layouts (Towns~1--5), where 
Town~2 is the training environment and 
Towns~1, 3--5 are fully out-of-distribution. 
Results are summarized in 
Table~\ref{tab:crosstown}.
Under this severe distribution shift, the task
success rate of all methods drops substantially,
and no method dominates on SR: averaged over the
four out-of-distribution towns, ChatScene-PPO,
VLM-RL, and DriveVLM-RL reach mean SR of $14.3\%$,
$6.8\%$, and $5.8\%$, respectively. However,
DriveVLM-RL's distinctive safety advantage persists
out of distribution: it attains the \emph{lowest
collision speed} in every unseen town (CS~$=1.59$,
$10.97$, $3.57$, and $3.57$~km/h for Towns~1, 3, 4,
and 5), yielding a mean out-of-distribution CS of
$4.93$~km/h, roughly $55\%$ lower than VLM-RL
($11.04$~km/h) and ChatScene-PPO ($10.84$~km/h). In
other words, when the policy does fail in an
unfamiliar environment, DriveVLM-RL's collisions are
markedly less severe, consistent with the semantic
risk reasoning learned during training. DriveVLM-RL
also retains competitive task progress, achieving
the highest route completion in Towns~3 and~5
(RC~$=0.38$ and $0.30$) and the longest travel
distance in Town~5 ($76.19$~m). The large-scale
highway layout of Town~4 and the multi-level
structure of Town~5 remain the most challenging for
all methods, indicating that highway-specific and
multi-level driving are important directions for
future work.

\begin{table*}[!t]
\caption{Cross-town generalization performance.
Models are trained exclusively on Town~2 and
evaluated on Towns~1--5. Town~2 (marked
$^{\dagger}$) is the in-distribution training
environment; Towns~1, 3, 4, and 5 are
out-of-distribution. Mean and standard
deviation over 3 seeds. The best result in each
column within a town is in \textbf{bold}.}
\label{tab:crosstown}
\centering
\renewcommand{\arraystretch}{1.25}
\setlength{\tabcolsep}{5pt}
\begin{tabular*}{\textwidth}{@{\extracolsep{\fill}}cccccccc@{}}
\toprule
Town & Model & AS~$\uparrow$ & RC~$\uparrow$ 
& TD~$\uparrow$ & CS~$\downarrow$ 
& SR~$\uparrow$ & AC~$\downarrow$ \\
\midrule

\multirow{3}{*}{Town 1}
& ChatScene-PPO
  & 15.57 {\tiny $\pm$ 0.07}
  & \textbf{0.33} {\tiny $\pm$ 0.04}
  & \textbf{275.59} {\tiny $\pm$ 53.61}
  & 4.87 {\tiny $\pm$ 0.48}
  & \textbf{0.30} {\tiny $\pm$ 0.10}
  & 0.70 {\tiny $\pm$ 0.10} \\
& VLM-RL
  & \textbf{16.25} {\tiny $\pm$ 1.47}
  & 0.27 {\tiny $\pm$ 0.08}
  & 138.66 {\tiny $\pm$ 8.40}
  & 10.64 {\tiny $\pm$ 3.24}
  & 0.03 {\tiny $\pm$ 0.06}
  & 0.40 {\tiny $\pm$ 0.10} \\
\rowcolor{green!10}
\cellcolor{white}& DriveVLM-RL
  & 13.98 {\tiny $\pm$ 2.08}
  & 0.21 {\tiny $\pm$ 0.03}
  & 125.39 {\tiny $\pm$ 33.13}
  & \textbf{1.59} {\tiny $\pm$ 0.80}
  & 0.03 {\tiny $\pm$ 0.06}
  & \textbf{0.27} {\tiny $\pm$ 0.06} \\
\midrule

\multirow{3}{*}{Town 2$^{\dagger}$}
& ChatScene-PPO
  & \textbf{14.78} {\tiny $\pm$ 0.30}
  & 0.44 {\tiny $\pm$ 0.14}
  & 127.85 {\tiny $\pm$ 10.39}
  & 6.05 {\tiny $\pm$ 1.28}
  & 0.40 {\tiny $\pm$ 0.10}
  & 0.60 {\tiny $\pm$ 0.10} \\
& VLM-RL
  & 14.38 {\tiny $\pm$ 1.53}
  & 0.51 {\tiny $\pm$ 0.08}
  & 138.08 {\tiny $\pm$ 16.68}
  & 10.09 {\tiny $\pm$ 5.93}
  & 0.40 {\tiny $\pm$ 0.00}
  & \textbf{0.10} {\tiny $\pm$ 0.10} \\
\rowcolor{green!10}
\cellcolor{white}& DriveVLM-RL
  & 14.54 {\tiny $\pm$ 1.81}
  & \textbf{0.57} {\tiny $\pm$ 0.03}
  & \textbf{186.59} {\tiny $\pm$ 14.00}
  & \textbf{1.75} {\tiny $\pm$ 3.02}
  & \textbf{0.57} {\tiny $\pm$ 0.15}
  & 0.20 {\tiny $\pm$ 0.26} \\
\midrule

\multirow{3}{*}{Town 3}
& ChatScene-PPO
  & \textbf{16.83} {\tiny $\pm$ 0.68}
  & 0.34 {\tiny $\pm$ 0.06}
  & \textbf{111.06} {\tiny $\pm$ 53.46}
  & 15.27 {\tiny $\pm$ 4.49}
  & \textbf{0.10} {\tiny $\pm$ 0.00}
  & 0.30 {\tiny $\pm$ 0.00} \\
& VLM-RL
  & 15.42 {\tiny $\pm$ 0.22}
  & 0.28 {\tiny $\pm$ 0.04}
  & 69.80 {\tiny $\pm$ 27.60}
  & 18.20 {\tiny $\pm$ 8.46}
  & 0.07 {\tiny $\pm$ 0.06}
  & \textbf{0.23} {\tiny $\pm$ 0.12} \\
\rowcolor{green!10}
\cellcolor{white}& DriveVLM-RL
  & 13.35 {\tiny $\pm$ 1.19}
  & \textbf{0.38} {\tiny $\pm$ 0.04}
  & 96.17 {\tiny $\pm$ 15.80}
  & \textbf{10.97} {\tiny $\pm$ 3.05}
  & \textbf{0.10} {\tiny $\pm$ 0.00}
  & 0.43 {\tiny $\pm$ 0.06} \\
\midrule

\multirow{3}{*}{Town 4}
& ChatScene-PPO
  & 19.88 {\tiny $\pm$ 0.99}
  & \textbf{0.27} {\tiny $\pm$ 0.02}
  & \textbf{494.91} {\tiny $\pm$ 53.16}
  & 15.20 {\tiny $\pm$ 3.16}
  & 0.10 {\tiny $\pm$ 0.00}
  & \textbf{0.23} {\tiny $\pm$ 0.15} \\
& VLM-RL
  & \textbf{20.29} {\tiny $\pm$ 0.95}
  & 0.18 {\tiny $\pm$ 0.02}
  & 355.34 {\tiny $\pm$ 41.04}
  & 8.54 {\tiny $\pm$ 4.16}
  & \textbf{0.17} {\tiny $\pm$ 0.06}
  & \textbf{0.23} {\tiny $\pm$ 0.06} \\
\rowcolor{green!10}
\cellcolor{white}& DriveVLM-RL
  & 18.99 {\tiny $\pm$ 1.15}
  & 0.18 {\tiny $\pm$ 0.07}
  & 409.86 {\tiny $\pm$ 128.30}
  & \textbf{3.57} {\tiny $\pm$ 2.45}
  & 0.07 {\tiny $\pm$ 0.12}
  & 0.27 {\tiny $\pm$ 0.06} \\
\midrule

\multirow{3}{*}{Town 5}
& ChatScene-PPO
  & 16.49 {\tiny $\pm$ 0.20}
  & 0.29 {\tiny $\pm$ 0.09}
  & 66.25 {\tiny $\pm$ 25.38}
  & 8.02 {\tiny $\pm$ 7.06}
  & \textbf{0.07} {\tiny $\pm$ 0.06}
  & 0.10 {\tiny $\pm$ 0.10} \\
& VLM-RL
  & \textbf{16.91} {\tiny $\pm$ 3.16}
  & 0.22 {\tiny $\pm$ 0.06}
  & 53.85 {\tiny $\pm$ 19.98}
  & 6.77 {\tiny $\pm$ 11.73}
  & 0.00 {\tiny $\pm$ 0.00}
  & \textbf{0.03} {\tiny $\pm$ 0.06} \\
\rowcolor{green!10}
\cellcolor{white}& DriveVLM-RL
  & 16.10 {\tiny $\pm$ 0.20}
  & \textbf{0.30} {\tiny $\pm$ 0.02}
  & \textbf{76.19} {\tiny $\pm$ 18.54}
  & \textbf{3.57} {\tiny $\pm$ 1.50}
  & 0.03 {\tiny $\pm$ 0.06}
  & 0.17 {\tiny $\pm$ 0.06} \\
\bottomrule
\end{tabular*}
\begin{minipage}{\linewidth}
\vspace{2pt}
\footnotesize
$^{\dagger}$Town~2 is the training environment; 
all other towns are out-of-distribution.
\end{minipage}
\end{table*}

\subsubsection{Traffic Density Robustness}
We evaluate all methods under three 
distinct traffic densities: Empty (0 vehicles), 
Regular (20 vehicles, matching training), and 
Dense (40 vehicles). Results are summarized in 
Table~\ref{tab:density}.
Under Empty conditions, where no interactive
hazards are present, all methods follow routes
reliably and the semantic pathway offers little
additional benefit: VLM-RL and DriveVLM-RL incur no
collisions, and ChatScene-PPO attains the highest SR
($0.90$) while DriveVLM-RL remains slightly more
conservative (SR $0.70$, yet with the highest route
completion, RC~$=0.57$). The advantage of DriveVLM-RL
emerges as interaction complexity increases. Under
Regular density (the training condition), DriveVLM-RL
achieves the best safety profile (CS~$=1.75$~km/h
vs.\ $6.05$--$10.09$~km/h for baselines) together
with the highest SR ($0.57$) and route completion
($0.57$). Under Dense conditions, all methods
degrade, but DriveVLM-RL degrades most gracefully and
leads five of six metrics: it retains the highest SR
($0.33$ vs.\ $0.20$--$0.27$), the longest distance
($127.47$~m), the lowest collision count
(AC~$=0.30$), and the lowest collision severity
(CS~$=2.28$~km/h vs.\ $4.77$~km/h for ChatScene-PPO
and $6.93$~km/h for VLM-RL). These results show that
the Dynamic Pathway's attention-gated semantic
reasoning scales effectively to high-density
interaction scenarios, with its safety advantage
widening precisely as the driving environment
becomes more hazardous.

\begin{table}[!t]
\caption{Performance under different traffic 
densities. Mean and standard deviation over 
3 seeds. The best result in each column within 
a density is in \textbf{bold}.}
\label{tab:density}
\centering
\renewcommand{\arraystretch}{1.25}
\setlength{\tabcolsep}{4.5pt}
\begin{tabular*}{\columnwidth}{@{\extracolsep{\fill}}c c c c c c c c@{}}
\toprule
Traffic Density & Model & AS~$\uparrow$ 
& RC~$\uparrow$ & TD~$\uparrow$ & CS~$\downarrow$ 
& SR~$\uparrow$ & AC~$\downarrow$ \\
\midrule
\multirow{3}{*}{Empty}
& ChatScene-PPO
  & 15.72 {\tiny $\pm$ 0.05}
  & \textbf{0.57} {\tiny $\pm$ 0.00}
  & \textbf{195.00} {\tiny $\pm$ 0.29}
  & 8.92 {\tiny $\pm$ 4.04}
  & \textbf{0.90} {\tiny $\pm$ 0.00}
  & 0.10 {\tiny $\pm$ 0.00} \\
& VLM-RL
  & \textbf{17.88} {\tiny $\pm$ 1.39}
  & 0.53 {\tiny $\pm$ 0.12}
  & 190.63 {\tiny $\pm$ 21.56}
  & \textbf{0.00} {\tiny $\pm$ 0.00}
  & 0.77 {\tiny $\pm$ 0.32}
  & \textbf{0.00} {\tiny $\pm$ 0.00} \\
\rowcolor{green!10}
\cellcolor{white}& DriveVLM-RL
  & 16.38 {\tiny $\pm$ 1.70}
  & \textbf{0.57} {\tiny $\pm$ 0.05}
  & 181.67 {\tiny $\pm$ 10.20}
  & \textbf{0.00} {\tiny $\pm$ 0.00}
  & 0.70 {\tiny $\pm$ 0.10}
  & \textbf{0.00} {\tiny $\pm$ 0.00} \\
\midrule
\multirow{3}{*}{Regular$^{\dagger}$}
& ChatScene-PPO
  & \textbf{14.78} {\tiny $\pm$ 0.30}
  & 0.44 {\tiny $\pm$ 0.14}
  & 127.85 {\tiny $\pm$ 10.39}
  & 6.05 {\tiny $\pm$ 1.28}
  & 0.40 {\tiny $\pm$ 0.10}
  & 0.60 {\tiny $\pm$ 0.10} \\
& VLM-RL
  & 14.38 {\tiny $\pm$ 1.53}
  & 0.51 {\tiny $\pm$ 0.08}
  & 138.08 {\tiny $\pm$ 16.68}
  & 10.09 {\tiny $\pm$ 5.93}
  & 0.40 {\tiny $\pm$ 0.00}
  & \textbf{0.10} {\tiny $\pm$ 0.10} \\
\rowcolor{green!10}
\cellcolor{white}& DriveVLM-RL
  & 14.54 {\tiny $\pm$ 1.81}
  & \textbf{0.57} {\tiny $\pm$ 0.03}
  & \textbf{186.59} {\tiny $\pm$ 14.00}
  & \textbf{1.75} {\tiny $\pm$ 3.02}
  & \textbf{0.57} {\tiny $\pm$ 0.15}
  & 0.20 {\tiny $\pm$ 0.26} \\
\midrule
\multirow{3}{*}{Dense}
& ChatScene-PPO
  & \textbf{14.70} {\tiny $\pm$ 0.28}
  & 0.41 {\tiny $\pm$ 0.09}
  & 90.71 {\tiny $\pm$ 12.73}
  & 4.77 {\tiny $\pm$ 1.10}
  & 0.20 {\tiny $\pm$ 0.17}
  & 0.80 {\tiny $\pm$ 0.17} \\
& VLM-RL
  & 11.52 {\tiny $\pm$ 2.73}
  & 0.37 {\tiny $\pm$ 0.07}
  & 99.11 {\tiny $\pm$ 12.74}
  & 6.93 {\tiny $\pm$ 1.63}
  & 0.27 {\tiny $\pm$ 0.06}
  & 0.50 {\tiny $\pm$ 0.00} \\
\rowcolor{green!10}
\cellcolor{white}& DriveVLM-RL
  & 10.69 {\tiny $\pm$ 0.11}
  & \textbf{0.46} {\tiny $\pm$ 0.08}
  & \textbf{127.47} {\tiny $\pm$ 20.02}
  & \textbf{2.28} {\tiny $\pm$ 1.83}
  & \textbf{0.33} {\tiny $\pm$ 0.15}
  & \textbf{0.30} {\tiny $\pm$ 0.17} \\
\bottomrule
\end{tabular*}
\begin{minipage}{\linewidth}
\vspace{2pt}
\footnotesize
$^{\dagger}$Regular is the training condition;
all other traffic densities are out-of-distribution.
\end{minipage}
\end{table}

\subsection{Ablation Study}

To assess the contribution of each core component,
we ablate the front-view camera input, the
attentional gating mechanism, and the hierarchical
reward synthesis, evaluating each variant on the
test routes. The results are summarized in
Table~\ref{tab:ablation}.

\begin{table}[!t]
\caption{Ablation study on framework components
during testing. All values are reported as means;
the full DriveVLM-RL model is averaged over 3 seeds
(consistent with Table~\ref{tab:carla_test}),
whereas each ablation variant is evaluated with a
single training run due to computational resource
constraints. The best results are highlighted in
\textbf{bold}.}
\label{tab:ablation}
\centering
\renewcommand{\arraystretch}{1.3}
\setlength{\tabcolsep}{5pt}
\begin{tabular*}{\columnwidth}{@{\extracolsep{\fill}}
lcccccc@{}}
\toprule
Model & AS~$\uparrow$ & RC~$\uparrow$ & TD~$\uparrow$ 
& CS~$\downarrow$ & SR~$\uparrow$ & AC~$\downarrow$ \\
\midrule
w/o First-View
  & 13.16
  & 0.42
  & 115.36
  & 3.87
  & 0.43
  & 0.31 \\
w/o Attentional Gating
  & 10.31
  & 0.27
  & 69.74
  & 7.38
  & 0.30
  & 0.50 \\
w/o Reward Synthesis
  & 14.02
  & 0.48
  & 132.17
  & 9.46
  & 0.40
  & \textbf{0.10} \\
\rowcolor{green!10}
DriveVLM-RL (Full)
  & \textbf{14.54}
  & \textbf{0.57}
  & \textbf{186.59}
  & \textbf{1.75}
  & \textbf{0.57}
  & 0.20 \\
\bottomrule
\end{tabular*}
\end{table}

\subsubsection{The Impact of First-View Camera Input}
In this variant, the Dynamic Pathway computes
dynamic rewards from BEV observations rather than
front-view camera images (Eq.~(\ref{eq6})).
Performance degrades moderately, with RC decreasing
from $0.57$ to $0.42$, SR declining from $57\%$ to
$43\%$, and AC rising from $0.20$ to $0.31$.
Although BEV representations capture spatial
relationships effectively, they lack the depth
cues and perspective context required to interpret
pedestrian intent and vehicle motion dynamics,
which leads the LVLM to select less informative
items from the reference vocabulary (\ref{appendices6})
(e.g., the static ``A pedestrian is directly ahead
on the road'' rather than the motion-aware ``A
pedestrian is crossing the road ahead''). Notably, this variant still
outperforms the w/o Reward Synthesis baseline
(SR: $43\%$ vs.\ $40\%$), indicating that the
remaining components preserve robustness even when
a single modality is degraded.

\subsubsection{The Impact of Attentional Gating}
Removing the YOLO-based attentional gate and
invoking LVLM inference on every frame severely
degrades the learned policy. This variant reaches
the lowest performance among all ablation variants,
with SR dropping from $57\%$ to $30\%$, RC falling
from $0.57$ to $0.27$, and AC rising from $0.20$ to
$0.50$.
This counterintuitive result, in which more
frequent semantic evaluation yields worse
performance, can be attributed to two factors.
First, invoking LVLM on every frame floods the 
reward signal with low-quality annotations from 
routine, hazard-free transitions, where LVLM 
outputs are inherently noisy and unstable. These 
spurious rewards (e.g., ``A vehicle might suddenly 
appear'' on a clear road) corrupt the policy 
gradient and obscure the informative signal from 
genuine safety-critical events. Second, the 
increased annotation workload strains the 
asynchronous pipeline, causing reward timestamps 
to lag behind policy updates and introducing 
temporal inconsistency in training. This validates 
our design principle: semantic reasoning should 
activate selectively on safety-critical stimuli, 
ensuring that LVLM rewards are both informative 
and temporally aligned.

\subsubsection{The Impact of Hierarchical Reward Synthesis}
This ablation removes the multiplicative 
integration with vehicle state functions 
(Eqs.~(\ref{eq10})--(\ref{eq11})), retaining 
only the combined semantic reward 
$R_{\text{combined}}$ from the Static and Dynamic
Pathways as the sole reward signal. The resulting
policy achieves the \textit{lowest} AC ($= 0.10$)
yet attains only a low SR ($= 40\%$), together with
a high CS when impacts do occur ($= 9.46$~km/h),
revealing a ``defensive stagnation'' failure mode. Without
multiplicative coupling to speed tracking, lane
centering, and heading alignment, the agent can
maximize semantic safety rewards by remaining
nearly stationary, a behavior that satisfies the
semantic notion of ``the road is clear'' while
failing to make navigational progress. When the
agent does attempt forward motion, the absence of
dynamic state guidance produces poorly timed
maneuvers and high-speed impacts, as reflected in
the elevated CS of $9.46$~km/h. The
full framework's multiplicative composition ensures
that semantic safety and vehicle dynamics objectives
are satisfied simultaneously, grounding abstract
risk reasoning in concrete control requirements.

\subsection{Reward Visualization}
\label{sec:reward-viz}

\begin{figure*}[!t]
\centering
\includegraphics[width=0.99999\textwidth]{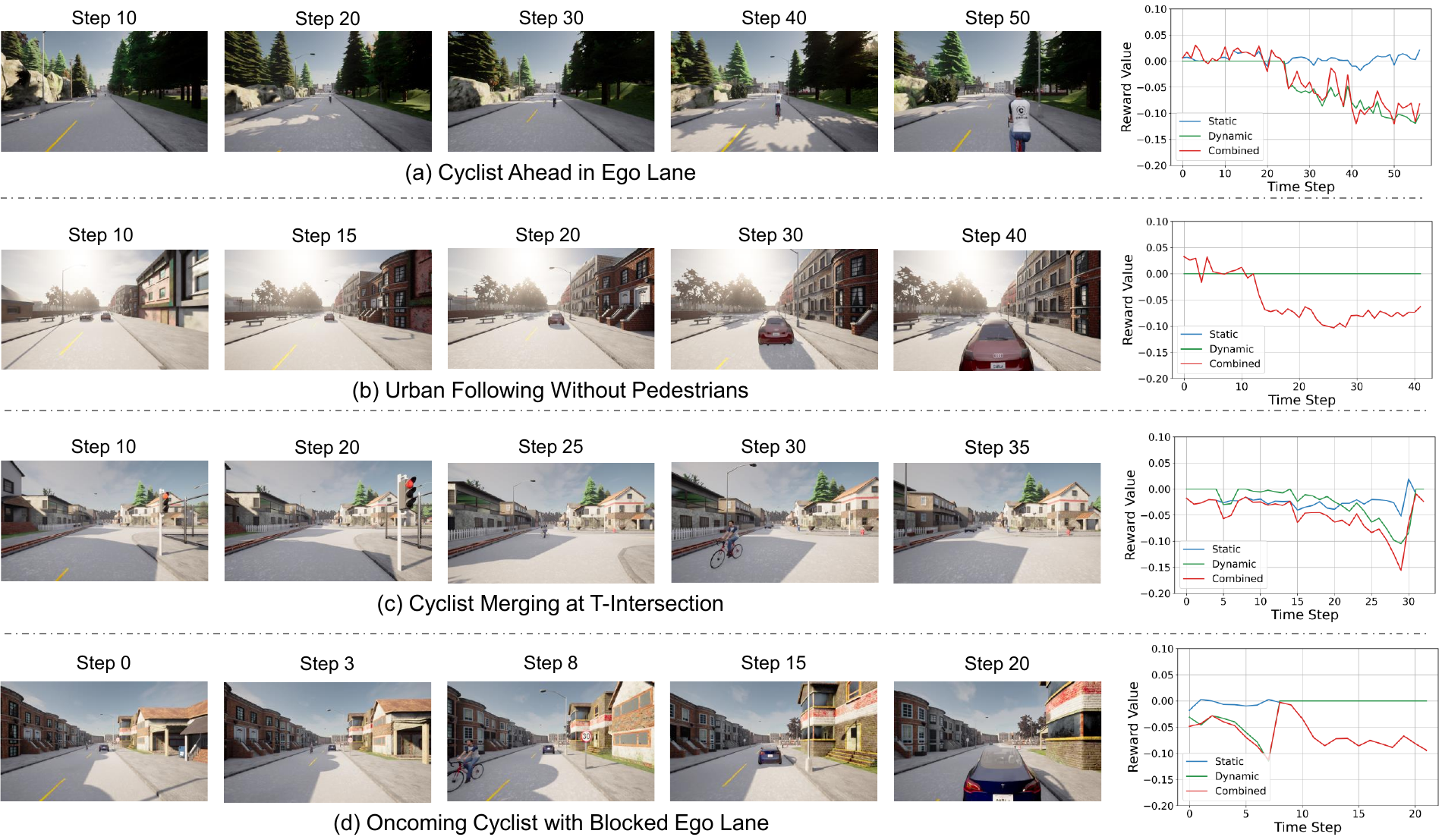}
\caption{Comparative reward curves and visual observations across diverse traffic scenarios. (a) A cyclist appears in the same lane ahead of the ego vehicle, requiring early deceleration and cautious following; (b) The ego vehicle follows a car in a structured urban environment with no pedestrians or dynamic obstacles; (c) At a T-junction, a cyclist enters the ego vehicle's path from the right, demanding quick perception and safe negotiation; (d) The ego lane is blocked by a parked or slow vehicle, while a cyclist is approaching from the opposite direction in the oncoming lane.}
\label{fig:reward_analysis}
\end{figure*}

Fig.~\ref{fig:reward_analysis} presents four representative 
driving scenarios that illustrate the complementary behavior 
of the static and dynamic reward components across varying 
levels of semantic complexity.

\textbf{(a) Cyclist Ahead in Ego Lane.} 
A cyclist gradually approaches the ego vehicle's lane over 
50 steps. The static reward ($R_{\text{static}}$, blue) 
remains relatively stable at a slightly negative level, 
reflecting persistent spatial proximity risk captured via 
BEV assessment. Once the cyclist enters the critical 
detection zone, the attentional gate activates ($g_t = 1$), 
triggering LVLM inference and producing a sharply negative 
dynamic reward ($R_{\text{dynamic}}$, green). The combined 
reward ($R_{\text{combined}}$, red) consequently drops, 
providing a strong penalty signal that encourages the agent 
to decelerate and yield.

\textbf{(b) Urban Following Without Pedestrians.} 
In a structured urban scenario with no vulnerable road users 
present, the attentional gate remains inactive throughout 
($g_t = 0$ for all $t$), and $R_{\text{dynamic}} = 0$ 
(green line flat at zero). The combined reward tracks 
the static reward exclusively, reflecting the designed 
fallback behavior: in the absence of safety-critical 
objects, the framework relies entirely on the Static 
Pathway for spatial safety assessment. This demonstrates
the computational efficiency of the gating mechanism, as no
LVLM inference is incurred in routine scenarios.

\textbf{(c) Cyclist Merging at T-Intersection.} 
A cyclist abruptly enters the ego vehicle's path from the 
right at a T-intersection around step 20--25. The static 
reward declines gradually as the cyclist approaches, while 
the dynamic reward exhibits a sharp negative spike upon
gate activation, capturing the semantic risk of the
crossing maneuver (e.g., ``A cyclist is crossing the
road ahead''). The combined reward reflects
both the spatial hazard and the semantic context, providing 
richer guidance than either signal alone.

\textbf{(d) Oncoming Cyclist with Blocked Ego Lane.} 
From step 0, the ego lane is partially blocked by a parked 
vehicle while a cyclist approaches from the opposite 
direction. The static reward immediately registers a 
negative value due to the spatial obstruction visible in 
BEV. As the oncoming cyclist is detected, the dynamic 
reward drops sharply, producing a strongly negative 
combined signal that discourages the agent from proceeding 
and encourages a cautious lane-change or stopping 
maneuver.

Across all four scenarios, the combined reward 
$R_{\text{combined}}$ consistently provides more 
discriminative and semantically grounded signals than 
either pathway alone, validating the design rationale 
of the Hierarchical Reward Synthesis module described 
in Section~\ref{Framework: DriveVLM-RL}.

\begin{figure*}[!t]
\centering
\includegraphics[width=\linewidth]{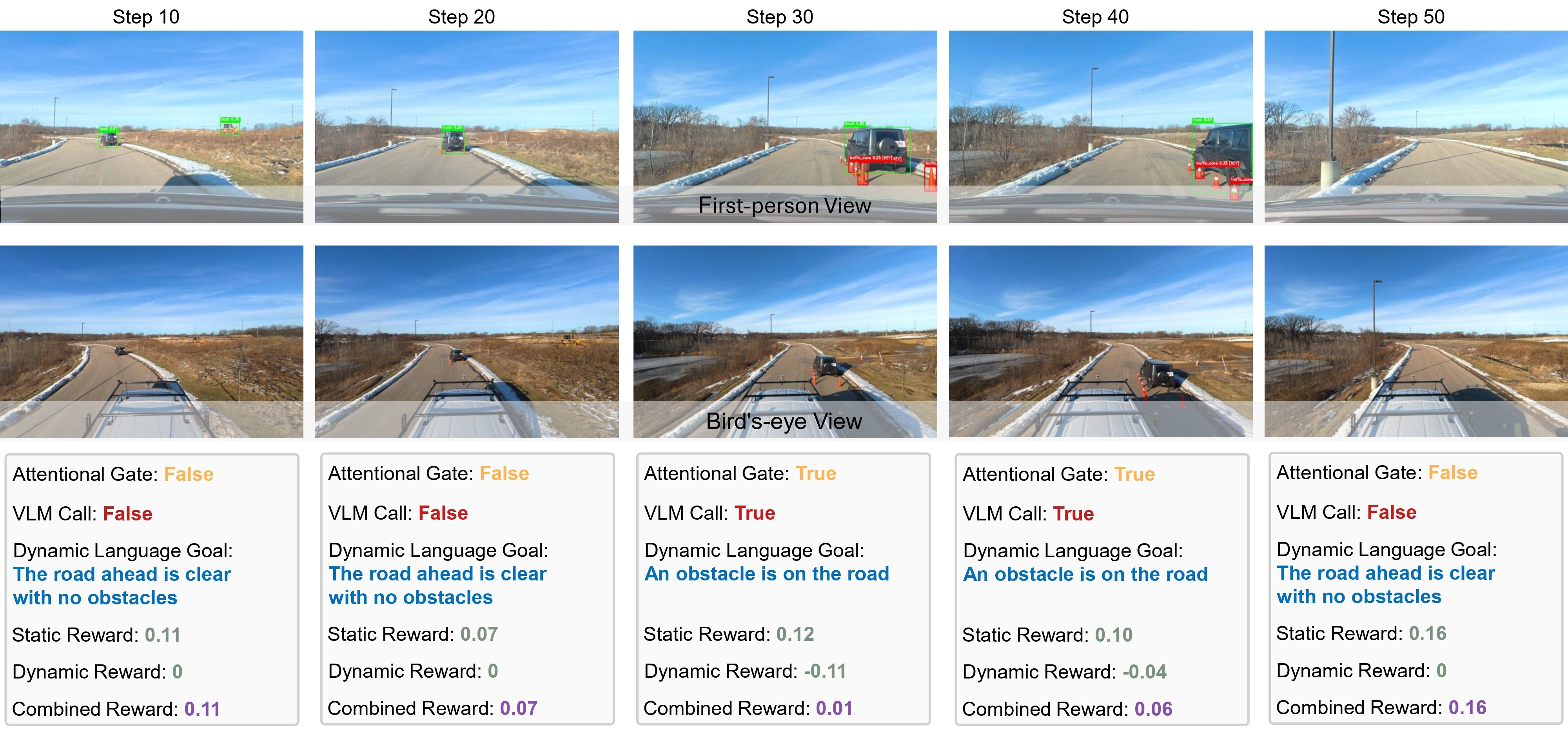}
\caption{Reward signal decomposition on a \emph{real-world} static-obstacle
bypass (DriveVLM-RL). Top: first-person view; middle: bird's-eye view; bottom:
per-step attentional gate, VLM call, dynamic language goal, and the static,
dynamic, and combined rewards. The attentional gate fires when the obstacle
enters the critical forward region, driving the dynamic reward negative and
flipping the combined-reward polarity. These signals are computed offline for
analysis and do not drive vehicle control.}
\label{fig:realworld_reward}
\end{figure*}

\textbf{Real-World Generalization.}
To verify that the reward design is not specific to the CARLA simulator,
Fig.~\ref{fig:realworld_reward} visualizes the same static/dynamic/combined
decomposition on a real-world recording from a full-scale vehicle during a
static-obstacle bypass. While the road ahead is clear (Steps~10--20), the
attentional gate stays inactive ($g_t = 0$, dynamic language goal ``The road
ahead is clear with no obstacles''), so the dynamic reward is zero and the
combined reward tracks the positive static reward. When the obstacle enters the
critical forward region (Steps~30--40), the gate fires ($g_t = 1$), the LVLM
updates the language goal to ``An obstacle is on the road,'' and the dynamic
reward turns sharply negative, flipping the combined reward and signaling the
hazard that motivates the bypass. After the obstacle is cleared (Step~50), the
gate deactivates and the combined reward recovers. The reward thus exhibits the
same gated, semantically grounded behavior on real perception inputs as in
simulation, indicating that the Hierarchical Reward Synthesis generalizes beyond
the training domain.

\subsection{Sensitivity Analysis}

\subsubsection{Attentional Gating Efficiency}
We evaluated three
YOLOv8~\citep{yolov8} variants of
increasing model capacity: \texttt{YOLOv8n},
\texttt{YOLOv8s}, and \texttt{YOLOv8x}.
As shown in Table~\ref{tab:yolo_comparison},
\texttt{YOLOv8s} attains the lowest steady-state
inference time (approximately 0.018--0.019~s per
frame) while substantially improving detection
coverage over \texttt{YOLOv8n}. For example, in
Episode~4 it detects 1744 objects versus 1388 for
\texttt{YOLOv8n} and triggers LVLM inference on
65.9\% of frames compared to 54.8\%, a meaningful
increase in safety-critical coverage at no
additional runtime cost. \texttt{YOLOv8x}
achieves marginally higher coverage (70.5\%) but
requires roughly 80\% more inference time
(0.033~s vs.\ 0.018~s). We therefore use
\texttt{YOLOv8s} as the detector in all main
experiments.

Fig.~\ref{fig:yolo_analysis} illustrates 
qualitative detection results across four 
representative cases. \texttt{YOLOv8n} shows 
notable limitations: in Case~(ii) it fails to 
detect a distant pedestrian (a false negative that 
suppresses LVLM inference), while in Cases~(iii) 
and~(iv) it misclassifies a fire hydrant as a 
pedestrian (a false positive that unnecessarily 
triggers LVLM reasoning). In Case~(i), it
mistakes a rock for a car, whereas \texttt{YOLOv8s}
suppresses this error and \texttt{YOLOv8x}
produces overly detailed detections that
introduce redundant KEY triggers. Based on
these observations, we adopt \texttt{YOLOv8s}
as the default detector, as it offers the best
balance between detection accuracy, robustness
to false positives, and runtime efficiency.

\begin{table}[!t]
\caption{YOLOv8 variant comparison for attentional 
gating. Infer: average inference time per frame~(s); 
Objects: total detections across all frames; 
Key: detections of safety-critical classes that 
trigger LVLM inference; 
VLM\%: percentage of frames invoking LVLM.}
\label{tab:yolo_comparison}
\centering
\renewcommand{\arraystretch}{1.2}
\setlength{\tabcolsep}{0pt}
\begin{tabular*}{\columnwidth}{@{\extracolsep{\fill}}
cccccc@{}}
\toprule
Episode (Frames) & Model & Infer~(s) 
& Objects & Key & VLM\% \\
\midrule
\multirow{3}{*}{Episode 1 (104)}
  & YOLOv8n & 0.020$^{*}$ & 107  & 13  & 10.6 \\
  & YOLOv8s & 0.027       & 181  & 13  & 10.6 \\
  & YOLOv8x & 0.044       & 280  & 23  & 18.3 \\
\midrule
\multirow{3}{*}{Episode 2 (142)}
  & YOLOv8n & 0.020 & 171  & 73  & 33.8 \\
  & YOLOv8s & 0.019 & 274  & 100 & 45.8 \\
  & YOLOv8x & 0.032 & 396  & 123 & 49.3 \\
\midrule
\multirow{3}{*}{Episode 3 (202)}
  & YOLOv8n & 0.020 & 181  & 24  & 8.4  \\
  & YOLOv8s & 0.019 & 289  & 26  & 12.4 \\
  & YOLOv8x & 0.034 & 501  & 27  & 11.9 \\
\midrule
\multirow{3}{*}{Episode 4 (522)}
  & YOLOv8n & 0.021 & 1388 & 634 & 54.8 \\
  & YOLOv8s & 0.018 & 1744 & 711 & 65.9 \\
  & YOLOv8x & 0.033 & 2099 & 798 & 70.5 \\
\bottomrule
\end{tabular*}
\vspace{3pt}
\begin{minipage}{\columnwidth}
\footnotesize\raggedright
$^{*}$Episode~1 inference time for YOLOv8n 
reflects cold-start model loading; steady-state 
inference time is 0.020~s, consistent with
subsequent episodes. Model load times:
YOLOv8n~=~0.047~s, YOLOv8s~=~0.040~s,
YOLOv8x~=~0.128~s.
\end{minipage}
\end{table}

\begin{figure*}[!t]
\centering
\includegraphics[width=0.993\textwidth]{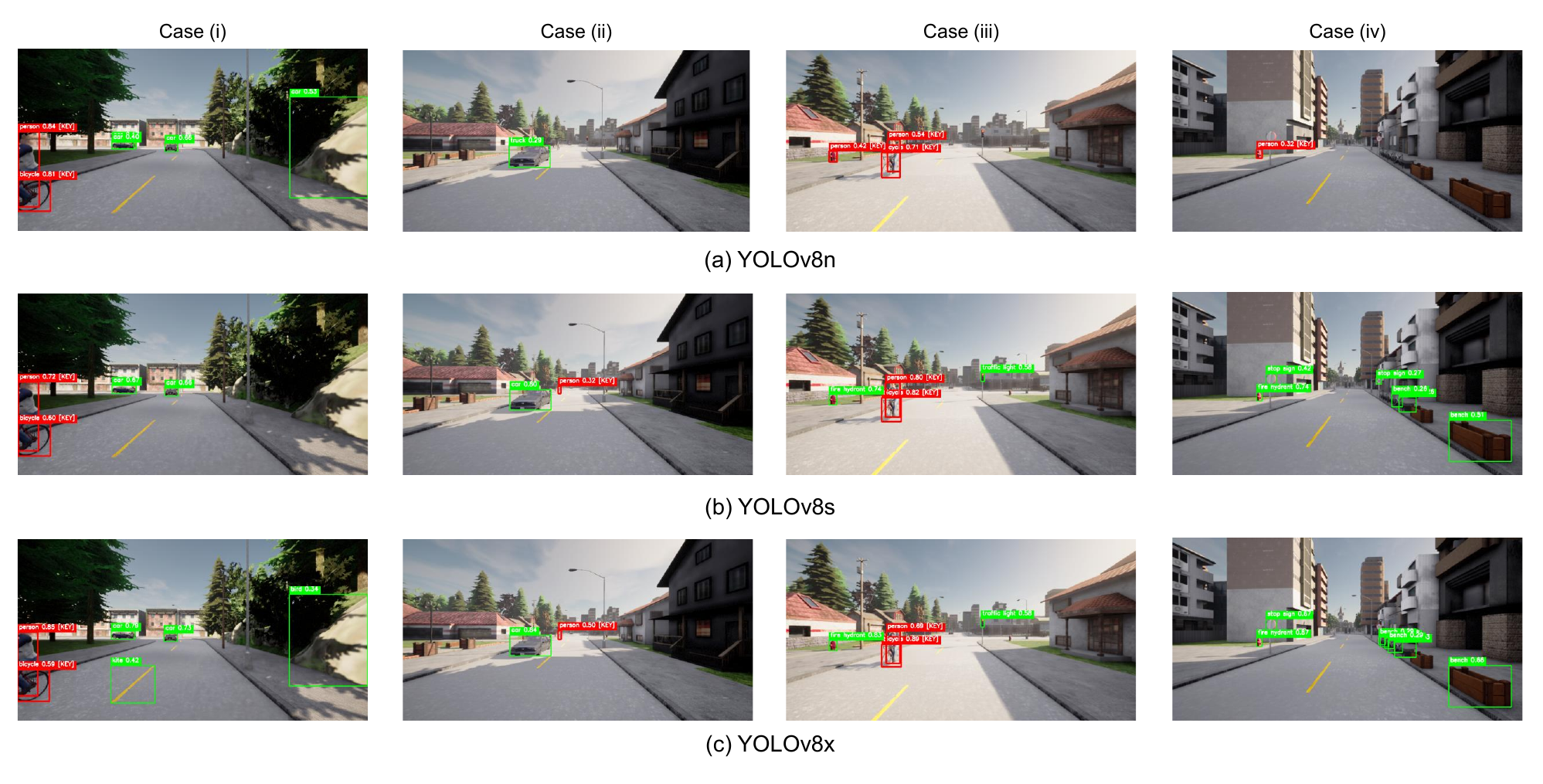}
\caption{Visualization of YOLOv8 Detection Results in Four Driving Cases. Red bounding boxes labeled with \texttt{[KEY]} denote critical objects  (e.g., pedestrians, bikes) that would activate VLM reasoning, while green boxes indicate non-critical detections. }
\label{fig:yolo_analysis}
\end{figure*}

\subsubsection{Impact of LVLM Backbone Capacity}
Table~\ref{tab:vlm} compares DriveVLM-RL
performance across three Qwen3-VL backbones of
increasing capacity. Owing to computational resource
constraints, the 2B and 8B variants are each
evaluated with a single training run, whereas the
default 4B configuration follows the three-seed
protocol used throughout the paper. The results
exhibit a consistent but rapidly saturating scaling
trend: enlarging the backbone from 4B to 8B yields
only marginal gains in success rate
(SR: $0.57$ to $0.60$), collision severity
(CS: $1.75$ to $1.42$~km/h), and average collisions
(AC: $0.20$ to $0.18$), despite doubling the
parameter count and inference cost.
This saturation stems from three properties of
our framework. First, the LVLM is tasked only
with scene description rather than complex
multi-step reasoning, a capability that is
well handled even by compact models. Second,
the reference vocabulary of canonical scene
descriptions (Section~\ref{sec:language_goal})
constrains the output space, reducing the task
to structured selection rather than open-ended
generation and minimizing sensitivity to model
capacity. Third, and most importantly, the LVLM
is invoked only during training to annotate
reward signals; at inference time, the deployed
policy operates entirely without LVLM calls.
As a result, the quality of the learned reward
signal matters more than the raw capacity of
the model generating it.
In contrast, the 2B model exhibits a pronounced
performance gap (SR: $0.43$, CS: $3.58$~km/h,
AC: $0.30$), indicating that a minimum level of
semantic understanding is required to produce
coherent and reliable scene descriptions; below
this threshold, noisier risk annotations translate
into less safe driving policies. We therefore adopt
Qwen3-VL-4B as the default backbone, as it surpasses
this semantic threshold while avoiding the
diminishing returns and elevated inference cost of
larger models, offering the most practical trade-off
for asynchronous reward annotation during training.
\begin{table}[!t]
\caption{Performance comparison using different
LVLM backbones. All values are reported as means;
the default 4B configuration is averaged over 3 seeds
(consistent with Table~\ref{tab:carla_test}), whereas
the 2B and 8B variants are each evaluated with a
single run due to computational resource constraints.
The baseline configuration is highlighted.}
\label{tab:vlm}
\centering
\renewcommand{\arraystretch}{1.3}
\setlength{\tabcolsep}{5pt}
\begin{tabular*}{\columnwidth}{@{\extracolsep{\fill}}cccccccc@{}}
\toprule
LVLM Model & Params & AS~$\uparrow$ & RC~$\uparrow$ 
& TD~$\uparrow$ & CS~$\downarrow$ 
& SR~$\uparrow$ & AC~$\downarrow$ \\
\midrule
Qwen3-VL-2B & 2B
  & 13.12
  & 0.41
  & 151.30
  & 3.58
  & 0.43
  & 0.30 \\
\rowcolor{green!10}
Qwen3-VL-4B (ours) & 4B
  & 14.54
  & 0.57
  & 186.59
  & 1.75
  & 0.57
  & 0.20 \\
Qwen3-VL-8B & 8B
  & 14.71
  & 0.58
  & 191.40
  & 1.42
  & 0.60
  & 0.18 \\
\bottomrule
\end{tabular*}
\end{table}

\section{Conclusion}
\label{Conclusion}
This paper presented DriveVLM-RL, a 
neuroscience-inspired framework that integrates 
VLM into RL 
for safe and deployable autonomous driving. 
Motivated by the brain's dual-pathway cognitive 
architecture, DriveVLM-RL decomposes semantic 
reward learning into a Static Pathway for 
continuous spatial safety assessment and a Dynamic 
Pathway for attention-gated, multi-frame semantic 
risk reasoning. A hierarchical reward synthesis 
mechanism fuses these signals with vehicle state 
information, while an asynchronous pipeline 
decouples expensive LVLM inference from 
environment interaction. Critically, all VLM 
components are used exclusively during training 
and completely removed at deployment, eliminating 
the latency constraints that plague existing 
VLM-as-Control approaches.

Extensive experiments in CARLA demonstrate that
DriveVLM-RL consistently outperforms 11 baseline
methods across expert-designed, LLM-based, and
VLM-based reward paradigms, attaining the highest
test success rate (57\%) and route completion
together with the longest travel distance, while
reducing collision severity from 10.09~km/h for the
strongest VLM-based baseline to 1.75~km/h. Under the
extreme no-reward-after-collision setting,
DriveVLM-RL maintains low collision rates
throughout training, demonstrating that the policy
internalizes predictive safety through semantic
reasoning rather than penalty-driven avoidance.
Its advantage is most pronounced in safety under
distribution shift: DriveVLM-RL attains the lowest
collision severity in every out-of-distribution
town (mean $4.93$~km/h vs.\ roughly $11$~km/h for
baselines) and leads the high-density regime. By
construction, the semantic reward synthesis is
decoupled from policy optimization, so the framework
integrates with standard RL algorithms without
modification.

Several directions remain open for future work. 
First, performance in structurally dissimilar 
environments such as highway-style and multi-level 
road layouts remains limited, highlighting the 
need for richer semantic vocabularies beyond 
pedestrian-rich urban scenarios. Second, the 
current framework is validated in simulation; 
bridging the sim-to-real gap for deployment on 
physical vehicles (e.g., Sky-Drive \citep{huang2025sky}) will require addressing 
sensor noise, domain shift in visual observations, 
and real-time safety constraints. Third, investigating online reward 
adaptation where the language goal vocabulary 
evolves during training could further improve 
long-tail robustness. Finally, extending the 
dual-pathway architecture to multi-agent settings 
represents a promising direction toward human-level 
generalization in autonomous driving.

\section*{Acknowledgment}
This work was supported by the University of Wisconsin-Madison's Center for Connected and Automated Transportation (CCAT), a part of the larger CCAT consortium, a USDOT Region 5 University Transportation Center funded by the U.S. Department of Transportation, Award \#69A3552348305. The contents of this paper reflect the views of the authors, who are responsible for the facts and the accuracy of the data presented herein, and do not necessarily reflect the official views or policies of the sponsoring organization.

\clearpage
\appendix
\setcounter{lemma}{0}
\setcounter{theorem}{0}
\setcounter{corollary}{0}
\setcounter{proposition}{0}
\setcounter{remark}{0}

\section{Static Pathway Proofs}
\label{appendices1}

The theoretical properties of the CLG-based static reward 
formulation follow from the structure of cosine similarity 
and real-valued arithmetic. Lemmas~1--2 extend the analysis
in~\citep{huang2025vlm} to the specific CLG formulation
defined in Definition~2.

\begin{lemma}[Boundedness]
For any observation $o_t$ and CLG pair 
$(l_{\text{pos}}, l_{\text{neg}})$, the static reward 
is bounded: $R_{\text{static}}(o_t) \in [-1, 1]$.
\end{lemma}

\begin{proof}
For any two unit-normalized vectors $v_1, v_2 \in 
\mathbb{R}^d$, the Cauchy--Schwarz inequality gives:
\begin{equation}
|v_1^\top v_2| \leq \|v_1\|\,\|v_2\|
\end{equation}
Since CLIP encoders produce $\ell_2$-normalized embeddings, 
$\|f_I(\cdot)\| = \|f_L(\cdot)\| = 1$, and thus:
\begin{equation}
\mathrm{sim}(f_I(o_t), f_L(l)) 
= \frac{f_I(o_t)^\top f_L(l)}
       {\|f_I(o_t)\|\,\|f_L(l)\|}
\in [-1, 1]
\end{equation}
Let $s^+ = \mathrm{sim}(f_I(o_t^{\text{BEV}}), 
f_L(l_{\text{pos}})) \in [-1,1]$ and 
$s^- = \mathrm{sim}(f_I(o_t^{\text{BEV}}), 
f_L(l_{\text{neg}})) \in [-1,1]$.
Then:
\begin{equation}
R_{\text{static}}(o_t) = \alpha \cdot s^+ - \beta \cdot s^-
\end{equation}
The upper bound is achieved when $s^+ = 1$ and $s^- = -1$:
\begin{equation}
R_{\text{static}}(o_t) 
\leq \alpha \cdot 1 - \beta \cdot (-1) 
= \alpha + \beta = 1
\end{equation}
The lower bound is achieved when $s^+ = -1$ and $s^- = 1$:
\begin{equation}
R_{\text{static}}(o_t) 
\geq \alpha \cdot (-1) - \beta \cdot 1 
= -(\alpha + \beta) = -1
\end{equation}
Therefore $R_{\text{static}}(o_t) \in [-1, 1]$.
\end{proof}

\begin{lemma}[Discriminability]
The CLG formulation provides strictly greater reward 
discrimination than single-goal formulations. Specifically, 
for observations $o_1, o_2$ where 
$\mathrm{sim}(f_I(o_1), f_L(l_{\text{pos}})) 
= \mathrm{sim}(f_I(o_2), f_L(l_{\text{pos}}))$
but $\mathrm{sim}(f_I(o_1), f_L(l_{\text{neg}})) 
\neq \mathrm{sim}(f_I(o_2), f_L(l_{\text{neg}}))$, 
we have $R_{\text{static}}(o_1) \neq R_{\text{static}}(o_2)$, 
even when single-goal similarity fails to distinguish 
the two states.
\end{lemma}

\begin{proof}
Let $s_i^+ = \mathrm{sim}(f_I(o_i), f_L(l_{\text{pos}}))$ 
and $s_i^- = \mathrm{sim}(f_I(o_i), f_L(l_{\text{neg}}))$ 
for $i \in \{1, 2\}$.

By hypothesis, $s_1^+ = s_2^+$ and $s_1^- \neq s_2^-$.

A single-goal reward $r_i = \mathrm{sim}(f_I(o_i), 
f_L(l_{\text{pos}})) = s_i^+$ satisfies $r_1 = r_2$, 
so it \emph{cannot} distinguish $o_1$ from $o_2$.

For the CLG reward:
\begin{align}
R_{\text{static}}(o_1) - R_{\text{static}}(o_2) 
&= (\alpha s_1^+ - \beta s_1^-) 
 - (\alpha s_2^+ - \beta s_2^-) \\
&= \alpha(s_1^+ - s_2^+) - \beta(s_1^- - s_2^-) \\
&= \alpha \cdot 0 - \beta(s_1^- - s_2^-) \\
&= -\beta(s_1^- - s_2^-)
\end{align}
Since $\beta > 0$ and $s_1^- \neq s_2^-$ by hypothesis, 
we conclude $R_{\text{static}}(o_1) \neq 
R_{\text{static}}(o_2)$.
\end{proof}

\begin{theorem}[Reward-Induced State Ordering]
Let $\mathcal{S}$ be the state space and define the 
binary relation $\succeq$ on $\mathcal{S}$ such that 
$s_1 \succeq s_2$ if and only if 
$R_{\text{static}}(s_1) \geq R_{\text{static}}(s_2)$. 
Then $\succeq$ is a total preorder (reflexive, transitive, 
and total), inducing a consistent preference ranking 
over states aligned with the semantic safety 
specification $(l_{\text{pos}}, l_{\text{neg}})$.
\end{theorem}

\begin{proof}
We verify the three defining properties of a total preorder.

\textbf{(1) Reflexivity.} For any $s \in \mathcal{S}$:
\begin{equation}
R_{\text{static}}(s) \geq R_{\text{static}}(s)
\end{equation}
holds trivially, so $s \succeq s$.

\textbf{(2) Transitivity.} Suppose $s_1 \succeq s_2$ 
and $s_2 \succeq s_3$ for some $s_1, s_2, s_3 \in 
\mathcal{S}$. Then:
\begin{equation}
R_{\text{static}}(s_1) \geq R_{\text{static}}(s_2) 
\geq R_{\text{static}}(s_3)
\end{equation}
By transitivity of $\geq$ on $\mathbb{R}$, 
$R_{\text{static}}(s_1) \geq R_{\text{static}}(s_3)$, 
thus $s_1 \succeq s_3$.

\textbf{(3) Totality.} For any $s_1, s_2 \in \mathcal{S}$, 
since $R_{\text{static}}(s_1), R_{\text{static}}(s_2) 
\in \mathbb{R}$ and $\geq$ is a total order on $\mathbb{R}$:
\begin{equation}
R_{\text{static}}(s_1) \geq R_{\text{static}}(s_2) 
\quad \text{or} \quad 
R_{\text{static}}(s_2) \geq R_{\text{static}}(s_1)
\end{equation}
Thus either $s_1 \succeq s_2$ or $s_2 \succeq s_1$ 
(or both when equality holds).

\textbf{Semantic Alignment.} 
The ordering $\succeq$ reflects the safety specification 
$(l_{\text{pos}}, l_{\text{neg}})$ because 
$R_{\text{static}}$ is monotonically increasing in 
$\mathrm{sim}(f_I(o), f_L(l_{\text{pos}}))$ and 
monotonically decreasing in 
$\mathrm{sim}(f_I(o), f_L(l_{\text{neg}}))$. 
Formally, for any $s_1 \succeq s_2$:
\begin{equation}
\alpha \cdot \mathrm{sim}(f_I(o(s_1)), 
f_L(l_{\text{pos}}))
- \beta \cdot \mathrm{sim}(f_I(o(s_1)), 
f_L(l_{\text{neg}}))
\geq
\alpha \cdot \mathrm{sim}(f_I(o(s_2)), 
f_L(l_{\text{pos}}))
- \beta \cdot \mathrm{sim}(f_I(o(s_2)), 
f_L(l_{\text{neg}}))
\end{equation}
This guarantees that $s_1$ is ranked no lower than $s_2$ 
precisely when $s_1$ is jointly more similar to the 
desired state $l_{\text{pos}}$ and less similar to 
the undesired state $l_{\text{neg}}$, consistent with 
the semantic safety specification. 

Therefore, $\succeq$ is a total preorder inducing a 
semantically grounded preference ranking over 
$\mathcal{S}$. 
\end{proof}

\section{Dynamic Pathway Proofs}
\label{appendices3}

\subsection{Proof of Lemma 3 (Computational Efficiency)}

\begin{lemma}[Computational Efficiency]
Let $p = P(g_t = 1)$ be the probability of gate 
activation, and let $T_{\text{LVLM}}$, $T_{\text{det}}$ 
denote the inference time of the LVLM and detection 
model respectively. The expected per-frame computation 
time of the Dynamic Pathway is 
$T_{\text{det}} + p \cdot T_{\text{LVLM}}$, compared 
to $T_{\text{LVLM}}$ for ungated approaches. When 
$p \ll 1$ and $T_{\text{det}} \ll T_{\text{LVLM}}$, 
this yields relative computational savings of 
approximately $(1-p)\times 100\%$ compared to 
ungated LVLM inference.
\end{lemma}

\begin{proof}
For each frame, the Dynamic Pathway executes:
(1) detection model $D(\cdot)$, always at cost 
$T_{\text{det}}$; and (2) LVLM $F_{\text{LVLM}}(\cdot)$, 
only when $g_t = 1$, at cost $T_{\text{LVLM}}$.

By linearity of expectation:
\begin{equation}
\mathbb{E}[T_{\text{gated}}]
= T_{\text{det}} + p \cdot T_{\text{LVLM}}
\label{eq:gated_time}
\end{equation}

The ungated baseline always runs the LVLM:
\begin{equation}
T_{\text{ungated}} = T_{\text{LVLM}}
\label{eq:ungated_time}
\end{equation}

The relative savings are:
\begin{equation}
\text{Savings}
= \frac{T_{\text{ungated}} - 
        \mathbb{E}[T_{\text{gated}}]}
       {T_{\text{ungated}}}
= 1 - p - \frac{T_{\text{det}}}{T_{\text{LVLM}}}
\;\approx\; 1 - p
\label{eq:savings}
\end{equation}
where the approximation holds when
$T_{\text{det}} \ll T_{\text{LVLM}}$.
In our experiments, $p$ is scene-dependent
(Table~\ref{tab:yolo_comparison}), ranging from
about $0.11$ in sparse scenes to $0.66$ in dense
traffic; since $T_{\text{det}} \ll T_{\text{LVLM}}$,
the resulting savings $1 - p -
T_{\text{det}}/T_{\text{LVLM}} \approx 1-p$ remain
substantial across this range.
\end{proof}

\subsection{Proof of Theorem 2 (Information Preservation 
under Gating)}

\begin{theorem}[Information Preservation under Gating]
Let $\mathcal{S}_{\text{critical}} \subseteq \mathcal{S}$ 
denote the set of safety-critical states, and let 
$\mu$ be a distribution over 
$\mathcal{S}_{\text{critical}}$. Assume:
\begin{enumerate}[label=(\roman*)]
  \item The detection model $D(\cdot)$ achieves recall 
        $\rho = P(g(s)=1 \mid s \in 
        \mathcal{S}_{\text{critical}})$ on 
        $\mathcal{S}_{\text{critical}}$;
  \item $R_{\text{LVLM}}(s) \geq 0$ for all 
        $s \in \mathcal{S}_{\text{critical}}$;
  \item Detection misses are not systematically 
        correlated with reward magnitude, i.e.,
        $\mathbb{E}_\mu[R_{\text{LVLM}} \mid g=1] 
        \geq 
        \mathbb{E}_\mu[R_{\text{LVLM}}]$.
\end{enumerate}
Let $g(s) \in \{0,1\}$ denote the gating indicator. 
Then:
\begin{equation}
\mathbb{E}_{s\sim\mu}
\!\left[g(s)\cdot R_{\text{LVLM}}(s)\right]
\;\geq\;
\rho \cdot
\mathbb{E}_{s\sim\mu}
\!\left[R_{\text{LVLM}}(s)\right]
\label{eq:info_preservation}
\end{equation}
\end{theorem}

\begin{proof}
By the law of total expectation, conditioning on 
$g(s)$:
\begin{align}
\mathbb{E}_{s\sim\mu}
\!\left[g(s)\cdot R_{\text{LVLM}}(s)\right]
&= P(g=1)\cdot
   \mathbb{E}_\mu[R_{\text{LVLM}}\mid g=1]
+  P(g=0)\cdot 0
\notag\\
&= \rho \cdot
   \mathbb{E}_\mu[R_{\text{LVLM}}\mid g=1]
\label{eq:lte}
\end{align}
where the second term vanishes because $g(s)=0$ 
implies $g(s)\cdot R_{\text{LVLM}}(s) = 0$.

By assumption (iii):
\begin{equation}
\mathbb{E}_\mu[R_{\text{LVLM}}\mid g=1]
\;\geq\;
\mathbb{E}_{s\sim\mu}[R_{\text{LVLM}}(s)]
\label{eq:assumption3}
\end{equation}

Substituting Eq.~(\ref{eq:assumption3}) into 
Eq.~(\ref{eq:lte}):
\begin{equation}
\mathbb{E}_{s\sim\mu}
\!\left[g(s)\cdot R_{\text{LVLM}}(s)\right]
= \rho \cdot
  \mathbb{E}_\mu[R_{\text{LVLM}}\mid g=1]
\;\geq\;
\rho \cdot
\mathbb{E}_{s\sim\mu}
\!\left[R_{\text{LVLM}}(s)\right]
\label{eq:final_bound}
\end{equation}
which establishes the claimed inequality. 
\end{proof}

\begin{remark}
Assumption (iii) is mild in practice: it states that 
the detection model does not systematically fail on 
the highest-risk frames. In our implementation,
YOLOv8~\citep{yolov8} attains high recall on the
safety-critical classes, so missed detections are
primarily low-confidence borderline cases rather than
high-severity scenarios, supporting the validity of
this assumption. Assumptions (i)--(ii) are standard; 
(ii) holds because $R_{\text{LVLM}}$ is defined via 
cosine similarity against a positive goal 
$l_{\text{pos}}$, which is non-negative in the 
normalized CLIP embedding space when safety-critical 
states are present.
\end{remark}

\section{Hierarchical Reward Synthesis Proofs}
\label{appendices4}

\begin{theorem}[Policy Improvement Guarantee]
\label{thm:policy_improvement_appendix}
Let $\pi_k$ denote the policy at iteration $k$, and 
$\pi_{k+1}$ the updated policy obtained under the 
hierarchical reward $R_{\text{final}}$. Under standard 
assumptions of soft actor--critic learning, including 
bounded rewards, sufficient exploration, and stable 
function approximation, the policy update satisfies:
\begin{equation}
J(\pi_{k+1}) \geq J(\pi_k) - \epsilon_k
\label{eq:C1}
\end{equation}
where $J(\pi) = \mathbb{E}_{\pi}\!\left[\sum_{t=0}^{T} 
\gamma^t R_{\text{final}}(o_t)\right]$, and $\epsilon_k$ 
denotes a bounded approximation error that diminishes 
as training progresses.
\end{theorem}

\begin{proof}
We verify that the hierarchical reward $R_{\text{final}}$ 
satisfies all conditions required for the SAC policy 
improvement theorem~\citep{haarnoja2018soft} to apply.

\paragraph{Step 1: Boundedness of $R_{\text{final}}$.}
We establish the reward bound by tracing through the 
hierarchical construction.

First, by Lemma~\ref{lemma:Boundedness}, the static 
reward satisfies $R_{\text{static}}(o_t) \in [-1, 1]$.

Second, for the dynamic reward (Definition~5), since 
$g_t \in \{0,1\}$ and the bracketed term 
$\alpha \cdot \mathrm{sim}(f_I(o_t^{\text{cam}}), 
f_L(l_{\text{pos}})) - \beta \cdot 
\mathrm{sim}(f_I(o_t^{\text{cam}}), f_L(l_t^{\text{dyn}}))
\in [-1, 1]$ by the same cosine similarity argument 
as Lemma~\ref{lemma:Boundedness}, we have:
\begin{equation}
R_{\text{dynamic}}(o_t) 
= g_t \cdot \bigl[\alpha\cdot\mathrm{sim}(
  f_I(o_t^{\text{cam}}), f_L(l_{\text{pos}}))
- \beta\cdot\mathrm{sim}(
  f_I(o_t^{\text{cam}}), f_L(l_t^{\text{dyn}}))
\bigr] \in [-1, 1]
\label{eq:C_dyn_bound}
\end{equation}
with $R_{\text{dynamic}} = 0$ when $g_t = 0$.

Third, since $R_{\text{combined}} = R_{\text{static}} + 
R_{\text{dynamic}} \in [-2, 2]$, the clipping and 
normalization in Eq.~(\ref{eq9}) map this to 
$R_{\text{norm}}(o_t) \in [0, 1]$ by construction of 
the $\mathrm{clip}$ operator.

Fourth, by Corollary~\ref{cor:bounded_final}, each 
factor $f_{\text{speed}}, f_{\text{center}}, 
f_{\text{angle}}, f_{\text{stability}} \in [0,1]$, 
so their product satisfies 
$R_{\text{shaping}}(o_t) \in [0, 1]$.

Therefore, the final reward is bounded:
\begin{equation}
R_{\text{final}}(o_t) \in [R_{\text{penalty}},\ 1],
\quad
|R_{\text{final}}(o_t)| \leq 
R_{\max} \triangleq \max(|R_{\text{penalty}}|,\ 1)
\label{eq:C2}
\end{equation}

\paragraph{Step 2: Compatibility with SAC Policy 
Improvement.}
The SAC algorithm~\citep{haarnoja2018soft} optimizes 
the maximum-entropy objective:
\begin{equation}
J(\pi_\phi) = \mathbb{E}_{\pi_\phi}\!\left[
\sum_{t=0}^{T} \gamma^t 
\left(R(o_t, a_t) + \lambda\,
\mathcal{H}(\pi_\phi(\cdot \mid o_t))\right)
\right]
\label{eq:C3}
\end{equation}
where $\lambda > 0$ is the entropy regularization 
coefficient. The SAC policy improvement theorem 
guarantees that each policy update satisfies 
$J(\pi_{k+1}) \geq J(\pi_k) - \epsilon_k$ provided 
that: (a) the reward function is bounded, and (b) the 
policy and Q-function lie within a sufficiently 
expressive function approximation class.

Condition (a) is satisfied by Eq.~(\ref{eq:C2}). 
Condition (b) is a standard assumption on the neural 
network architecture, which we adopt here.

\paragraph{Step 3: Approximation Error 
Characterization.}
In practice, neural function approximation introduces 
estimation error. Let $\mathcal{F}$ denote the 
function class of the critic network with 
pseudo-dimension $\mathrm{Pdim}(\mathcal{F})$. 
Following standard analyses in approximate dynamic 
programming~\citep{farahmand2010error}, the 
per-iteration approximation error can be bounded as:
\begin{equation}
\epsilon_k = 
\mathcal{O}\!\left(
\sqrt{\frac{\mathrm{Pdim}(\mathcal{F})}{N_k}}
\right) + \epsilon_{\text{approx}}
\label{eq:C4}
\end{equation}
where $N_k$ is the number of transitions sampled at 
iteration $k$ and $\epsilon_{\text{approx}}$ is the 
irreducible approximation error of the critic class. 
As training proceeds and $N_k \to \infty$, the first 
term vanishes, leaving only the approximation bias 
$\epsilon_{\text{approx}}$, which is bounded by the 
expressiveness of the chosen network architecture.

\paragraph{Step 4: Conclusion.}
Since $R_{\text{final}}$ is bounded 
(Eq.~\ref{eq:C2}), preserves the POMDP structure 
(rewards depend only on observations $o_t$), and 
the SAC conditions are satisfied, the policy 
improvement bound in Eq.~(\ref{eq:C1}) holds for 
all $k$. The sequence $\{\pi_k\}$ therefore converges 
to a stable fixed point with bounded suboptimality 
$\epsilon_k$ under the hierarchical reward 
$R_{\text{final}}$.
\end{proof}

\begin{remark}
The hierarchical structure of $R_{\text{final}}$ 
does not interfere with convergence for three reasons:
(i) all component rewards remain bounded 
(Corollary~\ref{cor:bounded_final}), satisfying 
the prerequisite of the SAC improvement theorem;
(ii) the shaping reward $R_{\text{shaping}}$ is 
observation-dependent only, preserving the underlying 
POMDP structure and ensuring that the Bellman 
operator remains a contraction; and
(iii) the asynchronous reward computation 
(Section~\ref{Framework: DriveVLM-RL}) introduces 
bounded reward staleness controlled by 
$N_{\text{warmup}}$, which affects convergence 
speed but not the validity of the improvement bound, 
since the Learner Thread preferentially samples 
reward-annotated transitions as described in 
Section~3.5.2.
\end{remark}

\clearpage
\section{Training Procedure with Asynchronous 
Batch-Processing}
\label{appendices5}

\begin{algorithm}[H]
\caption{DriveVLM-RL Training with Asynchronous 
Reward Synthesis}
\label{alg:drivevlm-rl}
\begin{algorithmic}[1]

\Require
Policy parameters $\phi$,
Q-function parameters $\theta$,
target parameters $\theta^-$,
replay buffer $\mathcal{D}$,
batch size $B$,
CLIP encoders $f_I,\, f_L$,
generative LVLM $F_{\text{LVLM}}$,
detection model $D$,
language goals $(l_{\text{pos}}, l_{\text{neg}})$,
safety-critical classes $\mathcal{C}_{\text{critical}}$,
CLG weighting factors $\alpha, \beta$ with 
$\alpha + \beta = 1$,
reward bounds $\theta_{\min}, \theta_{\max}$,
maximum speed $v_{\max}$,
update interval $\Delta$,
warmup threshold $N_{\text{warmup}}$,
entropy coefficient $\lambda$,
target network smoothing factor $\tau$

\State \textbf{Precompute language embeddings:}
\State $\mathbf{v}_{\text{pos}} \leftarrow 
f_L(l_{\text{pos}})$,\quad
$\mathbf{v}_{\text{neg}} \leftarrow 
f_L(l_{\text{neg}})$

\State $N_{\text{ready}} \leftarrow 0$
\Comment{Counter for reward-annotated transitions}

\For{$t = 1, 2, \ldots, T$}

    \State \textbf{// Interaction Thread}
    \State Observe $o_t = (o_t^{\text{BEV}},\, 
    o_t^{\text{cam}})$ from environment
    \State Select action $a_t \sim 
    \pi_\phi(\cdot \mid o_t)$, execute in environment,
    observe $o_{t+1}$
    \State Store transition 
    $(o_t,\, a_t,\, o_{t+1},\, 
    r_t \leftarrow \texttt{NaN},\, 
    \texttt{ready} \leftarrow 0)$ 
    in $\mathcal{D}$

    \If{$t \bmod \Delta = 0$}
        \State \textbf{// Reward Worker Thread}
        \State Sample mini-batch 
        $\{(o_i, a_i, o_{i+1})\}_{i=1}^{B}$ 
        from $\mathcal{D}$ where $\texttt{ready}=0$

        \For{each transition $i$ in mini-batch}

            \State \textbf{--- Static Pathway ---}
            \State $\mathbf{v}_i^{\text{BEV}} 
            \leftarrow f_I(o_i^{\text{BEV}})$
            \State $R_{\text{static}} \leftarrow
            \alpha \cdot \mathrm{sim}(
            \mathbf{v}_i^{\text{BEV}},\, 
            \mathbf{v}_{\text{pos}})
            - \beta \cdot \mathrm{sim}(
            \mathbf{v}_i^{\text{BEV}},\, 
            \mathbf{v}_{\text{neg}})$

            \State \textbf{--- Dynamic Pathway ---}
            \State $\mathcal{O}_i \leftarrow 
            D(o_i^{\text{cam}})$
            \If{$\exists\, o \in \mathcal{O}_i$ 
            s.t. $\mathrm{cls}(o) \in 
            \mathcal{C}_{\text{critical}}$}
                \State Construct temporal window 
                $\mathcal{W}_i = \{o_{i-K}^{\text{cam}},
                \ldots, o_i^{\text{cam}}\}$
                \State $l_i^{\text{dyn}} \leftarrow 
                F_{\text{LVLM}}(\mathcal{W}_i,\, 
                \mathcal{O}_i)$
                \Comment{Generate risk description}
                \State $\mathbf{v}_i^{\text{cam}} 
                \leftarrow f_I(o_i^{\text{cam}})$
                \State $R_{\text{dynamic}} \leftarrow
                \alpha \cdot \mathrm{sim}(
                \mathbf{v}_i^{\text{cam}},\, 
                \mathbf{v}_{\text{pos}})
                - \beta \cdot \mathrm{sim}(
                \mathbf{v}_i^{\text{cam}},\, 
                f_L(l_i^{\text{dyn}}))$
            \Else
                \State $R_{\text{dynamic}} \leftarrow 0$
            \EndIf

            \State \textbf{--- Hierarchical Reward 
            Synthesis ---}
            \State $R_{\text{combined}} \leftarrow 
            R_{\text{static}} + R_{\text{dynamic}}$
            \State $R_{\text{norm}} \leftarrow
            \dfrac{\mathrm{clip}(R_{\text{combined}},\, 
            \theta_{\min},\, \theta_{\max}) 
            - \theta_{\min}}
            {\theta_{\max} - \theta_{\min}}$

            \State $v_{\text{actual}} \leftarrow 
            \mathrm{speed}(o_i)$,\quad
            $v_{\text{desired}} \leftarrow 
            R_{\text{norm}} \cdot v_{\max}$
            \State $f_{\text{speed}} \leftarrow 
            \max\!\left(0,\; 1 - 
            \dfrac{|v_{\text{actual}} - 
            v_{\text{desired}}|}{v_{\max}}\right)$
            \Comment{Clipped to $[0,1]$}

            \State $R_{\text{shaping}} \leftarrow
            f_{\text{speed}} \cdot 
            f_{\text{center}}(o_i) \cdot
            f_{\text{angle}}(o_i) \cdot
            f_{\text{stability}}(o_i)$

            \State $R_{\text{final}} \leftarrow
            \begin{cases}
                R_{\text{penalty}}, 
                & \text{if collision at step } i \\
                R_{\text{shaping}}, 
                & \text{otherwise}
            \end{cases}$

            \State Update $r_i \leftarrow 
            R_{\text{final}}$,\;
            $\texttt{ready} \leftarrow 1$ in 
            $\mathcal{D}$
            \State $N_{\text{ready}} \leftarrow 
            N_{\text{ready}} + 1$

        \EndFor
    \EndIf

    \State \textbf{// Learner Thread}
    \If{$N_{\text{ready}} \geq N_{\text{warmup}}$}
        \Comment{Wait until sufficient annotated data}
        \State Sample mini-batch from $\mathcal{D}$ 
        with $\texttt{ready}=1$
        \State Update critic: minimize 
        $J'_Q(\theta)$ from Eq.~(\ref{eq15})
        \State Update actor: maximize SAC objective 
        $J(\pi_\phi)$ from Eq.~(\ref{eq13})
        \State Soft update target networks: 
        $\theta^- \leftarrow (1-\tau)\,\theta^- 
        + \tau\,\theta$
    \EndIf

\EndFor

\State \Return learned policy $\pi_\phi$
\end{algorithmic}
\end{algorithm}

\clearpage
\section{Reference Vocabulary for Dynamic Language Goal Generation}
\label{appendices6}

To constrain the output space of Qwen3-VL and ensure CLIP-compatible semantic embeddings, we provide the LVLM with the following 10 canonical scene descriptions, covering the primary safety-relevant actor types encountered in urban driving:

\begin{enumerate}[leftmargin=2.4em,itemsep=1.5pt,topsep=2pt,parsep=0pt]
    \item ``A cyclist is directly ahead on the road''
    \item ``A cyclist is crossing the road ahead''
    \item ``A pedestrian is directly ahead on the road''
    \item ``A pedestrian is crossing the road ahead''
    \item ``A motorcycle is directly ahead on the road''
    \item ``A motorcycle is crossing the road ahead''
    \item ``A construction zone is ahead''
    \item ``An obstacle is on the road''
    \item ``An animal is on the road''
    \item ``The road ahead is clear with no obstacles''
\end{enumerate}

\noindent\textbf{Matching Procedure.}
At each inference step, the LVLM receives the vocabulary embedded in a structured prompt and is instructed to output \texttt{REASONING: [brief explanation]} followed by \texttt{SELECTION: [copy exactly one description from the list]}. The system locates the \texttt{SELECTION:} field and performs exact substring matching against the vocabulary; if no match is found there, the full response is scanned for any vocabulary item. If the LVLM indicates that all detected objects are beyond 50~meters or that the scene is clear, description~10 is returned as a fallback. The matched description is then passed to the CLIP text encoder to compute the dynamic reward $r_t^{\text{dyn}}$ via cosine similarity with the corresponding image embedding. Fig.~\ref{fig:vocab_reasoning} illustrates this procedure on a real-world stop-sign interaction, where the LVLM reasons over the selected keyframes and outputs ``A pedestrian is crossing the road ahead'' as the dynamic language goal.

\begin{figure}[!ht]
\centering
\includegraphics[width=\linewidth]{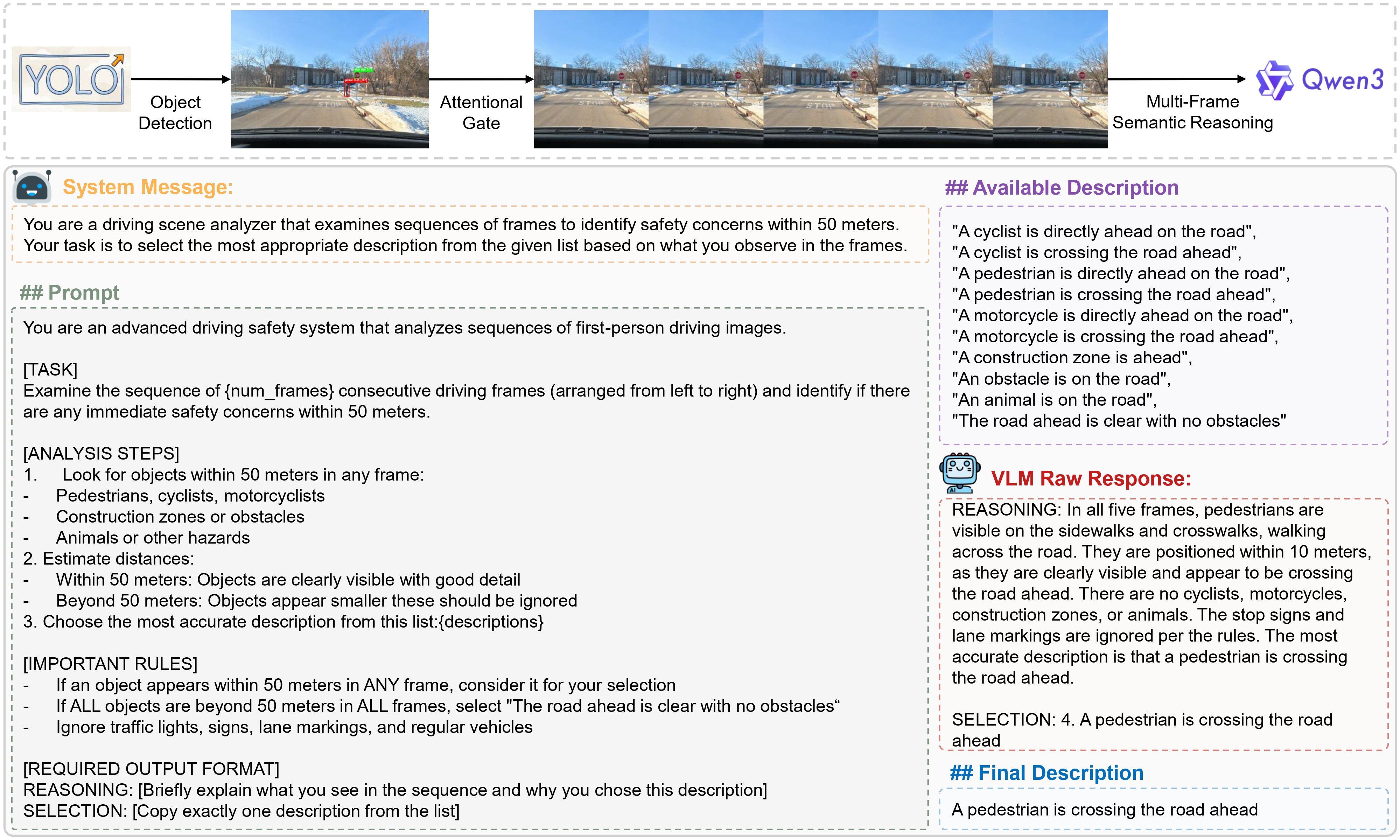}
\caption{LVLM semantic reasoning with the reference vocabulary on a \emph{real-world} stop-sign interaction (DriveVLM-RL). YOLO object detection and the attentional gate select keyframes that are passed to the Qwen3-VL LVLM for multi-frame semantic reasoning. The structured prompt embeds the ten canonical descriptions (``Available Description''); the LVLM returns step-by-step reasoning over the sequence and selects a single item (``A pedestrian is crossing the road ahead''), which becomes the dynamic language goal and is embedded by CLIP to compute the dynamic reward.}
\label{fig:vocab_reasoning}
\end{figure}

\clearpage
\bibliography{mybibfile}
	
\end{document}